\documentclass[twoside,11pt]{article}

 \usepackage[abbrvbib, preprint]{jmlr2e}

\usepackage{algorithm}
\usepackage{algorithmic}
\usepackage{xcolor}

\usepackage{amsmath,amsfonts,bm}
\allowdisplaybreaks

\newtheorem{assum}{Assumption} 

\usepackage{cleveref}
\hypersetup{linkcolor =blue, citecolor = blue}

\usepackage{multirow} 
\usepackage{colortbl}
\usepackage[toc,page,header]{appendix}
\usepackage{minitoc}
\definecolor{light-gray}{gray}{0.89}

\usepackage{mathtools}
\newcommand*{\colorboxed}{}
\def\colorboxed#1#{%
	\colorboxedAux{#1}%
}
\newcommand*{\colorboxedAux}[3]{%
	\begingroup
	\colorlet{cb@saved}{.}%
	\color#1{#2}%
	\boxed{%
		\color{cb@saved}%
		#3%
	}%
	\endgroup
}

\DeclareMathOperator*{\argmin}{arg\,min}

\jmlrheading{, submitted}{}{}{Aug/15}{not yet}{meila00a}{Kaiyi Ji  and Yingbin Liang}
\ShortHeadings{Lower Bounds and Accelerated Algorithms for Bilevel Optimization}{Ji and Liang}
\firstpageno{1}

\begin{document}

\title{Lower Bounds and Accelerated Algorithms for Bilevel Optimization}

\author{\name Kaiyi\ Ji \email ji.367@osu.edu \\
       \addr Department of Electrical and Computer Engineering\\
       The Ohio State University\\
      Columbus, OH 98195-4322, USA
       \AND
        \name Yingbin\ Liang \email liang.889@osu.edu \\
       \addr Department of Electrical and Computer Engineering\\
       The Ohio State University\\
      Columbus, OH 98195-4322, USA}

\editor{}

\maketitle

\begin{abstract}
Bilevel optimization has recently attracted growing interests due to its wide applications in modern machine learning problems. Although recent studies have characterized the convergence rate for several such popular algorithms, it is still unclear how much further these convergence rates can be improved. In this paper, we address this fundamental question from two perspectives. First, we provide the first-known lower complexity bounds of $\widetilde{\Omega}(\frac{1}{\sqrt{\mu_x}\mu_y})$ and $\widetilde \Omega\big(\frac{1}{\sqrt{\epsilon}}\min\{\frac{1}{\mu_y},\frac{1}{\sqrt{\epsilon^{3}}}\}\big)$ respectively for strongly-convex-strongly-convex and convex-strongly-convex bilevel optimizations. Second, we propose an accelerated bilevel optimizer named AccBiO, for which we provide the first-known complexity bounds without the gradient boundedness assumption (which was made in existing analyses) under the two aforementioned geometries. We also provide significantly tighter upper bounds than the existing complexity when the bounded gradient assumption does hold. We show that AccBiO achieves the optimal results (i.e., the upper and lower bounds match up to logarithmic factors)
when the inner-level problem takes a quadratic form with a constant-level condition number. Interestingly, our lower bounds under both geometries are larger than the corresponding optimal complexities of minimax optimization, establishing that bilevel optimization is provably more challenging than minimax optimization. 
\end{abstract}

\begin{keywords}
Bilevel optimization, lower bounds, accelerated algorithms, computational complexity, convergence rate, optimality. 
\end{keywords}

\section{Introduction}
Bilevel optimization was first introduced by~\cite{bracken1973mathematical}, and since then has been studied for decades~\citep{hansen1992new,shi2005extended,moore2010bilevel}. Recently, bilevel optimization  has attracted growing interests due to its important role in various machine learning applications including meta-learning~\citep{franceschi2018bilevel,rajeswaran2019meta}, hyperparameter optimization~\citep{franceschi2018bilevel,feurer2019hyperparameter}, imitation learning~\citep{arora2020provable}, and network architecture search~\citep{liu2018darts,he2020milenas}. A general formulation of unconstrained bilevel optimization can be written as follows. 
\begin{align}\label{obj}
&\min_{x\in\mathbb{R}^p} \Phi(x):=f(x, y^*(x)),\quad \mbox{s.t.}\;\;y^*(x)= \argmin_{y\in\mathbb{R}^{q}} g(x,y),
\end{align} 

\vspace{-0.2cm}
\noindent where $f$ and $g$ are continuously differentiable functions. The problem~\cref{obj} contains two optimization procedures: at the inner level we search  $y^*(x)$ as the minimizer of $g(x,y)$ with respect to (w.r.t.) $y$ given $x$, and 
at the outer level we minimize the objective function $\Phi(x)$ w.r.t. $x$, which includes the compositional dependence on $x$ via $y^*(x)$.

%A large body of bilevel optimization algorithms have been proposed, which include constraint-based approaches~\citep{shi2005extended,moore2010bilevel} and gradient-based approaches~\citep{domke2012generic,pedregosa2016hyperparameter,gould2016differentiating,maclaurin2015gradient,franceschi2018bilevel,ghadimi2018approximation,liao2018reviving,shaban2019truncated,hong2020two,liu2020generic,li2020improved,grazzi2020iteration,lorraine2020optimizing}.. 

Most theoretical studies of bilevel optimization algorithms have focused on {\em the asymptotic} 
 analysis without the convergence rate characterization. For example, \cite{franceschi2018bilevel,shaban2019truncated} established the asymptotic convergence for gradient-based approaches when there is one single solution for the inner-level problem, and \cite{liu2020generic,li2020improved} extended the analysis to the setting where the inner-level problem allows multiple solutions. 
 The {\em finite-time} analysis that characterizes the convergence rate  of bilevel optimization algorithms is rather limited except a few studies recently.  \cite{grazzi2020iteration} provided the iteration complexity of two dominant types of strategies, i.e., approximate implicit differentiation (AID) and iterative differentiation (ITD), 
  for approximating the hypergradient $\nabla \Phi(x)$,  but did not characterize the finite-time convergence for the entire execution of algorithms.  \cite{ghadimi2018approximation} proposed an AID-based bilevel approximation (BA) algorithm as well as an accelerated variant ABA, and analyzed their finite-time complexities under different loss geometries. In particular, the complexity upper bounds of BA  and ABA are given by $\mathcal{\widetilde O}(\frac{1}{\mu_y^{6}\mu_x^{2}})$ and $\mathcal{ \widetilde O}(\frac{1}{\mu_y^{3}\mu_x })$ for the strongly-convex-strongly-convex setting where $\Phi(\cdot)$ is $\mu_x$-strongly-convex and $g(x,\cdot)$ is $\mu_y$-strongly-convex,  $\mathcal{O}\big(\frac{1}{\mu_y^{11.25}\epsilon^{1.25}}\big)$ and $\mathcal{ O}\big(\frac{1}{\mu_y^{6.75}\epsilon^{0.75}} \big)$ for the convex-strongly-convex setting, and $\mathcal{ O}\big(\frac{1}{\mu_y^{6.25}\epsilon^{1.25}} \big)$ for the nonconvex-strongly-convex setting.  \cite{ji2021bilevel} further improved the bound for the nonconvex-strongly-convex setting to $\mathcal{O}\big(\frac{1}{\mu_y^{4}\epsilon }\big)$.

In this paper, we address several open and important questions about bilevel optimization. We first observe 
 that the existing complexity results on  {\bf bilevel} optimization are much worse than those on {\bf minimax} optimization, which is a special case of bilevel optimization with $f(x,y)=g(x,y)$. For example, for the convex-strongly-convex case, it was shown in~~\cite{lin2020near}  that the optimal complexity for {\bf minimax} optimization is given by $\mathcal{\widetilde O}\big(\frac{1}{\epsilon^{0.5}\mu_y^{0.5}}\big)$, which is much smaller than the best known $\mathcal{ \widetilde O}\big(\frac{1}{\mu_y^{6.75}\epsilon^{0.75} }\big)$ for {\bf bilevel} optimization. Similar observations hold for the strongly-convex-strongly-convex setting. Therefore, one fundamental question arises.
\begin{list}{$\bullet$}{\topsep=0.1in \leftmargin=0.2in \rightmargin=0.1in \itemsep =0.01in}
 \item[1.] \textit{What is the performance limit of bilevel optimization in terms of computational complexity? Whether bilevel optimization is provably more challenging (i.e., requires more computations) than minimax optimization?}
%\item  \textbf{Whether bilevel optimization is fundamentally harder than minimax optimization?}
%\item[2.] \textit{Can we design near-optimal bilevel optimization algorithms under certain conditions? }
 \end{list}
Furthermore, existing analyses reply on a strong assumption on the boundedness of the outer-level gradient $\nabla_y f(x,\cdot)$\footnote{\cite{grazzi2020iteration} assume that the inner-problem solution $y^*(x)$ is uniformly bounded for all $x$ so that $\nabla_y f(x,y^*(x))$ is bounded.} to guarantee that the smoothness parameter of $\Phi(\cdot)$ and the hyperparameter estimation error are bounded as the algorithm runs. Then the following question needs to be addressed. 
% In addition, their complexities involve pessimistic dependences on the condition numbers, e.g., $\mathcal{O}(\frac{1}{\mu_y^{6.75}})$ for the convex-strongly-convex case. 
 \begin{list}{$\bullet$}{\topsep=0.1in \leftmargin=0.2in \rightmargin=0.1in \itemsep =0.01in}
 \item[2.] \textit{Can we design a new bilevel optimization algorithm, which provably converges without the gradient boundedness? If so, whether such an algorithm achieves the optimal computational complexity?}
 \end{list} 
 In addition, even when the boundedness assumption holds, existing complexity bounds show pessimistic complexity dependences on the condition numbers, e.g., $\mathcal{O}(\frac{1}{\mu_y^{6.75}})$ for the convex-strongly-convex case. Then, the following question arises. 
 \begin{list}{$\bullet$}{\topsep=0.1in \leftmargin=0.2in \rightmargin=0.1in \itemsep =0.01in}
 \item[3.] \textit{Under the bounded gradient assumption, can we provide new upper bounds with tighter complexity dependences on the condition numbers for strongly-convex-strongly-convex and convex-strongly-convex
 bilevel optimizations?}
 \end{list} 
 In this paper, we provide affirmative answers to the above questions. 

\begin{table*}[!t]
\renewcommand{\arraystretch}{1.1}
\centering
\small
\caption{Comparison of computational complexities for finding an $\epsilon$-approximate point {\bf without the gradient boundedness assumption}. All listed results are from this paper, as all existing results were developed under the gradient boundedness condition, which we compare in \Cref{tab:results_withB}.  
%for the convex-strongly-convex case or $\epsilon$-stationary point for (non)convex-strongly-convex case. 
The complexity is measured by $\tau (n_J+n_H) + n_G$ (\Cref{complexity_measyre}), where $n_G, n_J,n_H$ are the numbers of gradients, Jacobian- and Hessian-vector products, and $\tau$ is a universal constant. 
In the `references' column, quadratic $g(x,y)$ means that $g$ takes a quadratic form as {\small $g(x,y)=y^T H y +  x^T J y + b^Ty+h(x)$} for the constant matrices $H,J$ and a constant vector $b$.
In the `computational complexity' column, $\widetilde L_y$ denotes the smoothness parameter of $g(x,\cdot)$, $\rho_{xy}$ and $\rho_{yy}$ are the Lipschitz parameters of {\small $\nabla_y^2g(\cdot,\cdot)$} and {\small$\nabla_x\nabla_yg(\cdot,\cdot)$} (see \cref{def:three}), {\small $\Delta^*_{\text{\normalfont\tiny SCSC}}=\|\nabla_y f( x^*,y^*(x^*))\|+\frac{\|x^*\|}{\mu_y}+\frac{\sqrt{\Phi(0)-\Phi(x^*)}}{\sqrt{\mu_x}\mu_y}$}, and $\Delta^*_{\text{\normalfont\tiny CSC}}$ takes the same form as $\Delta^*_{\text{\normalfont\tiny SCSC}}$ but with $\mu_x$ replaced by {\small $\frac{\epsilon}{(\|x^*\|+1)^2}$}. The lower bounds hold for both the general and quadratic $g(x,y)$ cases. 
% In the reference column, quadratic $g(x,\cdot)$ contains constant Hessian and Jacobian (see \cref{quadratic_case}).
}
\label{tab:results}
%\vspace{0.4cm}
%{ \em \normalsize Without Gradient Boundedness Assumption }

\vspace{0.3cm}
\begin{tabular}{|c|c|c|} \hline
 \textbf{Types} & \textbf{References} & \textbf{Computational Complexity} \\ \hline 
\multirow{3}{*}{\shortstack{Strongly-\\Convex-\\Strongly-\\Convex}} 
%&BA~\citep{ghadimi2018approximation} & {\small $\mathcal{\widetilde O}\Big(\max\Big\{\frac{1}{\mu_x^2\mu_y^4},\frac{\widetilde L^2_y}{\mu^2_y}\Big\}\Big)$}  \\ \cline{2-3}
&  \cellcolor{light-gray} \textbf{AccBiO} (\Cref{upper_srsr_withnoB}) &\cellcolor{light-gray} {\scriptsize $\mathcal{\widetilde O}\Big(\sqrt{\frac{\widetilde L_y}{\mu_x\mu_y^{3}}}+\Big(\sqrt{\frac{ \rho_{yy}\widetilde L_y}{\mu_x\mu_y^{4}} }+ \sqrt{ \frac{\rho_{xy}\widetilde L_y}{\mu_x\mu_y^{3}}}\Big)\sqrt{\Delta^*_{\text{\normalfont \tiny SCSC}}}\Big)$} \\ \cline{2-3}
& \cellcolor{light-gray} \textbf{AccBiO} (quadratic $g$, \Cref{coro:quadaticSr}) & \cellcolor{light-gray} {\small $ \mathcal{\widetilde O}\Big(\sqrt{\frac{\widetilde L_y}{\mu_x\mu_y^{3}}}\Big)$} \\ \cline{2-3} 
&  \cellcolor{blue!15} \textbf{Lower bound} (\Cref{thm:low1}) & \cellcolor{blue!15} {$\widetilde{\Omega}\big(\sqrt{\frac{1}{\mu_x\mu_y^2}}\big)$} \\ \hline \hline
\multirow{4}{*}{\shortstack{Convex-\\Strongly-\\Convex} }  %%% th:upper_csc1sc  coro:quadaticConv
%& BA~\citep{ghadimi2018approximation} & {\small $\mathcal{\widetilde O}\Big(\frac{1}{\epsilon^{2.5}}\max\Big\{\frac{1}{\mu_y^{5}},\frac{\widetilde L^{10}_y}{\mu^{10}_y}\Big\}\Big)$} \\ \cline{2-3}
%&ABA~\citep{ghadimi2018approximation} &{\small $\mathcal{\widetilde O}\Big(\frac{1}{\epsilon^{1.5}}\max\Big\{\frac{1}{\mu_y^{3}},\frac{\widetilde L^{6}_y}{\mu^{6}_y}\Big\}\Big)$}  \\ \cline{2-3}
& \cellcolor{light-gray} \textbf{AccBiO} ( \Cref{th:upper_csc1sc})& \cellcolor{light-gray} {\scriptsize
$\mathcal{\widetilde O}\Big( \sqrt{\frac{\widetilde L_y}{\epsilon\mu_y^3}}+\Big(\sqrt{\frac{\rho_{yy}\widetilde L_y}{\epsilon\mu_y^{4}}} +  \sqrt{\frac{\rho_{xy}\widetilde L_y}{\epsilon\mu_y^3}}\Big)\sqrt{\Delta^*_{\text{\normalfont\tiny CSC}}}\Big)$
}\\ \cline{2-3}
& \cellcolor{light-gray} \textbf{AccBiO} (quadratic $g$, \Cref{coro:quadaticConv})& \cellcolor{light-gray} {\small$\mathcal{\widetilde O}\Big(\sqrt{\frac{\widetilde L_y}{\epsilon\mu_y^3}}\Big)$}\\ \cline{2-3}
& \cellcolor{blue!15} \textbf{Lower bound} $($\Cref{co:co1}, {\scriptsize $\widetilde L_y\leq \mathcal{O}(\mu_y)$}$)$& \cellcolor{blue!15} {$\widetilde \Omega \Big(\sqrt{\frac{1}{\epsilon\mu_y^2}}\Big)$}\\ \cline{2-3}
& \cellcolor{blue!15} \textbf{Lower bound} $($\Cref{co:co2}, {\scriptsize $\widetilde L_y\leq \mathcal{O}(1)$}$)$& \cellcolor{blue!15} {\small $\widetilde \Omega\big(\frac{1}{\sqrt{\epsilon}}\min\{\frac{1}{\mu_y},\frac{1}{\epsilon^{1.5}}\}\big)$}\\ \hline 
%\multirow{3}{*}{\shortstack{Nonconvex-\\Strongly-Convex}} & BA~\citep{ghadimi2018approximation} & {\small$\mathcal{ \widetilde O}(\mu_y^{-6.25}\epsilon^{-1.25} )$} \\ 
%\hhline{|~--|}
%& AID-BiO \citep{ji2021bilevel} & {\small$\mathcal{\widetilde O}(\mu_y^{-5}\epsilon^{-1} )$} \\ \cline{2-3}
%& \cellcolor{light-gray} \textbf{AccBiO} (\Cref{th:nonconvex}) & \cellcolor{light-gray} {\small$\mathcal{\widetilde O}(\mu_y^{-4.5}\epsilon^{-1} )$} \\ \hline 
\end{tabular}
\end{table*}

\begin{table*}[h]
\renewcommand{\arraystretch}{1.1}
\centering
\small
\caption{Comparison of computational complexities for finding an $\epsilon$-approximate point {\bf with the gradient boundedness assumption}. }
\label{tab:results_withB}
%\vspace{0.4cm}
%{ \em \normalsize With Gradient Boundedness Assumption that $\|\nabla_y f(x,\cdot)\|\leq U$ }

\vspace{0.3cm}
\begin{tabular}{|c|c|c|} \hline
 \textbf{Types} & \textbf{References} & \textbf{Computational Complexity} \\ \hline 
\multirow{3}{*}{\shortstack{Strongly-Convex-\\Strongly-Conconvex}} 
&BA~\citep{ghadimi2018approximation} & {\small $\mathcal{\widetilde O}\Big(\max\Big\{\frac{1}{\mu_x^2\mu_y^6},\frac{\widetilde L^2_y}{\mu^2_y}\Big\}\Big)$}  \\ \cline{2-3}
&ABA~\citep{ghadimi2018approximation} & {\small $\mathcal{\widetilde O}\Big(\max\Big\{\frac{1}{\mu_x\mu_y^3},\frac{\widetilde L^2_y}{\mu^2_y}\Big\}\Big)$}  \\ \cline{2-3}& \cellcolor{light-gray} \textbf{AccBiO-BG} (this paper, \Cref{upper_srsr}) & \cellcolor{light-gray} {\small $ \mathcal{\widetilde O}\Big(\sqrt{\frac{\widetilde L_y}{\mu_x\mu_y^4}}\Big)$} \\ \hline \hline
%&  \cellcolor{blue!15} \textbf{Lower bound} (\Cref{thm:low1}) & \cellcolor{blue!15} {$\widetilde{\Omega}\big(\sqrt{\frac{1}{\mu_x\mu_y^2}}\big)$} \\ \hline \hline
\multirow{3}{*}{\shortstack{Convex-\\Strongly-Convex} }  %%% th:upper_csc1sc  coro:quadaticConv
& BA~\citep{ghadimi2018approximation} & {\small $\mathcal{\widetilde O}\Big(\frac{1}{\epsilon^{1.25}}\max\Big\{\frac{1}{\mu_y^{3.75}},\frac{\widetilde L^{10}_y}{\mu^{11.25}_y}\Big\}\Big)$} \\ \cline{2-3}
&ABA~\citep{ghadimi2018approximation} &{\small $\mathcal{\widetilde O}\Big(\frac{1}{\epsilon^{0.75}}\max\Big\{\frac{1}{\mu_y^{2.25}},\frac{\widetilde L^{6}_y}{\mu^{6.75}_y}\Big\}\Big)$}  \\ \cline{2-3}
& \cellcolor{light-gray} \textbf{AccBiO-BG} (this paper, \Cref{convex_upper_BG})& \cellcolor{light-gray} {\small 
$\mathcal{\widetilde O}\Big( \sqrt{\frac{\widetilde L_y}{\epsilon\mu_y^4}}\Big)$
}\\ 
\hline 
%\multirow{3}{*}{\shortstack{Nonconvex-\\Strongly-Convex}} & BA~\citep{ghadimi2018approximation} & {\small$\mathcal{ \widetilde O}(\mu_y^{-6.25}\epsilon^{-1.25} )$} \\ 
%\hhline{|~--|}
%& AID-BiO \citep{ji2021bilevel} & {\small$\mathcal{\widetilde O}(\mu_y^{-5}\epsilon^{-1} )$} \\ \cline{2-3}
%& \cellcolor{light-gray} \textbf{AccBiO} (\Cref{th:nonconvex}) & \cellcolor{light-gray} {\small$\mathcal{\widetilde O}(\mu_y^{-4.5}\epsilon^{-1} )$} \\ \hline 
\end{tabular}
\end{table*}

 \subsection{Summary of Contributions}
Our  main contributions  lie in developing several new results for bilevel optimization, including the first-known lower bounds on the computational complexity, a new convergence analysis without the gradient boundedness assumption, and significantly tighter upper bounds for bilevel optimization under different geometries. Our upper bounds meet the lower bounds in various cases, suggesting the tightness of the lower bounds and the optimality of the proposed algorithms. We summarize our results as follows.

\begin{list}{$\bullet$}{\topsep=0.03in \leftmargin=0.2in \rightmargin=0.in \itemsep =0.01in}

\item We provide the first-known lower bound of {\small $\widetilde{\Omega}(\frac{1}{\sqrt{\mu_x}\mu_y})$} for solving the strongly-convex-strongly-convex bilevel optimization.  
%Our constructed worst-case instance uses the bilinearly coupled inner-level function $g(x,y)$, which is %The  inner-level function $g(x,y)$ in our constructed worst-case instance take the quadratic form
We then propose a new accelerated bilevel optimizer named AccBiO. 
 In contrast to existing bilevel optimizers, we show that AccBiO converges to the $\epsilon$-accurate solution  without the requirement on the boundedness of the gradient {\small $\nabla_y f(x,\cdot)$} for any $x$. In particular, \Cref{tab:results} shows that AccBiO achieves an upper complexity bound of {\small $\mathcal{\widetilde O}\Big(\sqrt{\frac{\widetilde L_y}{\mu_x\mu_y^{3}}}+\Big(\sqrt{\frac{ \rho_{yy}\widetilde L_y}{\mu_x\mu_y^{4}} }+ \sqrt{ \frac{\rho_{xy}\widetilde L_y}{\mu_x\mu_y^{3}}}\Big)\sqrt{\Delta^*_{\text{\normalfont\tiny SCSC}}}\Big)$}.  When the inner-level function $g(x,y)$ takes the quadratic form as {\small $g(x,y)=y^T H y +  x^T J y + b^Ty+h(x)$}, we further improve the upper bounds to {\small $ \mathcal{\widetilde O}\Big(\sqrt{\frac{\widetilde L_y}{\mu_x\mu_y^{3}}}\Big)$}. For such a quadratic subclass of bilevel problems with {\small $\widetilde L_y\leq \mathcal{O}(\mu_y)$}, our upper bound matches the lower bound up to logarithmic factors, suggesting that AccBiO is near-optimal.   Technically, our analysis of the lower bound involves careful construction of quadratic $f$ and $g$ functions with a properly structured bilinear term, as well as novel characterization of subspaces of iterates for updating $x$ and $y$. For upper bounds, our analysis controls the finiteness of all iterates $x_k,k=0,....$ as the algorithm runs via an induction proof to ensure that the hypergradient estimation error will not explode after the acceleration steps. 

%and for the special 
%
%As shown in \Cref{tab:results}, AccBiO achieves an upper bound of {\small$\widetilde{O}\big(\frac{1}{\sqrt{\mu_x\mu_y^2}}\big(1+\frac{\sqrt{ \widetilde L_y}}{\sqrt{\mu_y}}\big)\big)$}, which improves existing upper bounds by an order of $\frac{1}{\sqrt{\mu_x\mu_y}}$. Furthermore, our upper bound matches the lower bound for a subclass of bilevel problems with $\widetilde L_y\leq \mathcal{O}(\mu_y)$, suggesting that AccBiO is near-optimal for such a case. Technically, our lower-bounding analysis involves a careful construction of quadratic $f$ and $g$ functions with a nicely structured bilinear term, as well as a characterization of subspaces of iterates for updating $x$ and $y$. 

\item We next provide lower and upper bounds for solving convex-strongly-convex bilevel optimization. 
%to find an $\epsilon$-stationary point $\widetilde x$ such that $\|\nabla\Phi(\widetilde x)\|\leq\epsilon$. 
As shown in \Cref{tab:results}, AccBiO achieves an upper bound of {\small 
$\mathcal{\widetilde O}\Big( \sqrt{\frac{\widetilde L_y}{\epsilon\mu_y^3}}+\Big(\sqrt{\frac{\rho_{yy}\widetilde L_y}{\epsilon\mu_y^{4}}} +  \sqrt{\frac{\rho_{xy}\widetilde L_y}{\epsilon\mu_y^3}}\Big)\sqrt{\Delta^*_{\text{\normalfont\tiny CSC}}}\Big)$
}, which is further improved to  {\small$\mathcal{\widetilde O}\Big(\sqrt{\frac{\widetilde L_y}{\epsilon\mu_y^3}}\Big)$} for the quadratic $g(x,y)$. For such a quadratic case with $\widetilde L_y\leq \mathcal{O}(\mu_y)$, our upper bound matches the lower bound up to logarithmic factors, suggesting the optimality of AccBiO. Technically, the analysis of the lower bound is different from that for the strongly-convex $\Phi(\cdot)$, and exploits the structures of different powers of an unnormalized graph Laplacian matrix $Z$. 
 
\item Furthermore, when the gradient $\nabla_y f(x,\cdot)$ is bounded, as assumed by existing studies, we provide new upper bounds with significantly tighter dependence on the condition numbers. In specific, as shown in \Cref{tab:results_withB}, our upper bounds outperform the best known results by a factor of $\frac{1}{\mu_x^{0.5}\mu_y}$ and $\frac{1}{\epsilon^{0.25}\mu_y^{4.75}}$ for the strongly-convex-strongly-convex and  convex-strongly-convex 
cases, respectively. 
%improves existing upper bounds by an order of $\min\big\{\frac{1}{\epsilon^{8.5}},\frac{1}{\epsilon\mu_y^5}\big\}$. Furthermore,  the upper bound matches the lower bound when $\widetilde L_y\leq \mathcal{O}(\mu_y)$, suggesting the optimality of AccBiO in this case. Technically, the lower-bounding analysis is different from that for the strongly-convex $\Phi(\cdot)$, which exploits the structures of different powers of an unnormalized graph Laplacian matrix $Z$. 

%For the general constant-level $\widetilde L_y$, our upper bound improves existing upper-bounds by an order of $\min\big\{\frac{1}{\epsilon^{8.5}},\frac{1}{\epsilon\mu_y^5}\big\}$. 

%For the nonconvex-strongly-convex bilevel optimization, 
%we show that AccBiO improves existing upper bounds by an order of $\frac{1}{\sqrt{\mu_y}}$.  

\item To compare between bilevel optimization and minimax optimization, for the strongly-convex-strongly-convex case, our lower bound is larger than the optimal complexity of {\small$\widetilde{\Omega}(\frac{1}{\sqrt{\mu_x\mu_y}})$} for the same type of minimax optimization by a factor of $\frac{1}{\sqrt{\mu_y}}$. Similar observation holds for the convex-strongly-convex case. This establishes that bilevel optimization is fundamentally more challenging than minimax optimization.   

%\item Finally, we discuss the application and extension of our results to more general bilevel problem class.  We also talk about the extension of our lower-bounding analysis to convex-concave or convex-strongly-concave minimax optimization. However, we note that for a general constant-level $\widetilde L_y$,  there exists a gap of $\frac{1}{\sqrt{\mu_y}}$ between our upper and lower bounds for the (strongly)-convex-strongly-convex settings, which we leave as an open problem for future study. 
\end{list}

%\vspace{-0.1cm}
\subsection{Related Works}
The studies of bilevel optimization problems and algorithms can be dated back to~\cite{bracken1973mathematical}, and since then, different types of approaches have been proposed. Earlier approaches in \cite{aiyoshi1984solution,edmunds1991algorithms,al1992global,hansen1992new,shi2005extended,lv2007penalty,moore2010bilevel} reduced the bilevel problem to a single-level optimization problem using the Karush-Kuhn-Tucker (KKT) conditions or penalty function methods.
%\cite{aiyoshi1984solution,lv2007penalty} proposed penalty function methods by replacing the upper-level problem with a penalized problem. 
%However, such approaches are often hard to implement and unscalable to the problems in modern applications involving neural networks. 
%Some works~\citep{aiyoshi1981hierarchical,aiyoshi1984solution,lv2007penalty} proposed penalty function methods by replacing the inner- or upper-level problem with a penalized problem.
% However, the reduced problem is still difficult to solve~\citep{sinha2017review}. 
In comparison, gradient-based approaches are more attractive due to their efficiency and effectiveness.  
%The basic idea of these approaches is to estimate the hypergradient $\nabla\Phi(x)$ for iterative updates.  %Based on different ways for hypergradient estimation~\citep{grazzi2020iteration}, gradient-based 
Such a type of approaches estimate the hypergradient $\nabla\Phi(x)$ for iterative updates, and are generally divided into AID- and ITD-based categories.  ITD-based approaches~\citep{maclaurin2015gradient,franceschi2017forward,finn2017model,grazzi2020iteration}  estimate the hypergradient $\nabla\Phi(x)$ in either a reverse (automatic differentiation) or forward manner.  
%In particular, the reverse mode $\frac{\partial f(x,y^N(x))}{\partial x}$, where $y^N(x)$ is output of an iterative method such as gradient decent (GD). 
AID-based approaches~\citep{domke2012generic,pedregosa2016hyperparameter,grazzi2020iteration,ji2021bilevel} estimate the hypergradient via implicit differentiation.%, which involves solving a linear system. 
 
Theoretically, bilevel optimization has been studied via both the asymptotic and finite-time (non-asymptotic) analysis. \cite{franceschi2018bilevel} characterized the asymptotic convergence of a backpropagation-based approach as one of ITD-based algorithms by assuming the inner-level problem is strongly convex.  
\cite{shaban2019truncated} provided a similar analysis for a truncated backpropagation scheme. \cite{liu2020generic,li2020improved} analyzed the asymptotic performance of ITD-based approaches when   the inner-level problem is convex. The finite-time complexity analysis for bilevel optimization has also been explored. In particular, \cite{ghadimi2018approximation} provided a finite-time convergence analysis for an AID-based algorithm under two different loss geometries, where $\Phi(\cdot)$ is strongly convex, convex or nonconvex, and $g(x,\cdot)$ is strongly convex. \cite{ji2021bilevel} provided an improved finite-time analysis for AID- and ITD-based algorithms under the nonconvex-strongly-convex geometry. In this paper, we provide the first-known lower  bounds on complexity as well as tighter upper bounds under these two geometries. 

When the objective functions can be expressed in an expected or finite-time form, 
\cite{ghadimi2018approximation,ji2021bilevel,hong2020two} developed stochastic bilevel algorithms and provided the finite-time analysis. In particular, \cite{ji2021bilevel} proposed a SGD type of bilevel optimization algorithm named stocBiO with a sample efficient hypergradient estimator. Since then, there have been a few subsequent studies on accelerating 
SGD-type bilevel optimization via momentum-based variance reduction~\cite{chen2021single,guo2021stochastic,khanduri2021near,yang2021provably,huang2021biadam}. For example, \citealt{guo2021stochastic} proposed a single-loop algorithm SEMA based on the momentum-based technique introduced by~\citealt{cutkosky2019momentum}. \citealt{chen2021single} proposed a single-loop method named STABLE by using a similar momentum scheme for the Hessian updates. \citealt{yang2021provably} improved the sample complexity of stocBiO via both single-loop and double-loop variance reduction. While the stochastic setting is not within the scope of this paper, the accelerating algorithms and lower bounds developed here can be extended to the stochastic setting.  
%SEMA and STABLE achieve the same complexity as our stocBiO w.r.t.~$\epsilon$  but at a cost of a worse dependence on the conditional number. \citealt{khanduri2021near,guo2021randomized,yang2021provably} improved the dependence on $\epsilon$ of our stocBiO from $\mathcal{O}(\epsilon^{2})$ to $\mathcal{O}(\epsilon^{1.5})$ via recursive momentum and variance reduction but the resulting complexity still has a worse dependence on the condition number. 

Bilevel optimization has been applied to meta-learning and led to various algorithms such as model-agnostic meta-learning (MAML)~\citep{finn2017model}, implicit MAML (iMAML)~\citep{rajeswaran2019meta}, and almost no inner loop (ANIL)~\citep{raghu2019rapid}.  Theoretically, \cite{rajeswaran2019meta} analyzed the complexity of iMAML via implicit differentiation under the strongly-convex setting. \cite{ji2020multi,fallah2020convergence} characterized the convergence of MAML under the nonconvex function geometry. \cite{ji2020convergence} analyzed the convergence and complexity of ANIL with either strongly-convex or nonconvex inner-level geometries.

Bilevel optimization has been applied to study various machine learning problems. For example, bilevel optimization has exhibited great effectiveness in hyperparameter optimization, and received tremendous attention recently in automatic machine learning (autoML)~\citep{okuno2018hyperparameter,yu2020hyper}. A variety of bilevel optimization algorithms have been proposed for this area, which include but not limited to AID-based~\citep{pedregosa2016hyperparameter,franceschi2018bilevel}, ITD-based~\citep{franceschi2018bilevel,shaban2019truncated,grazzi2020iteration}, self-tuning network based~\citep{mackay2018self,bae2020delta}, penalty-based~\citep{mehra2019penalty,sinha2020gradient,liu2021value}, and proximal approximation based~\citep{jenni2018deep} approaches.  Bilevel optimization has also been exploited to improve the search efficiency for neural architecture search (NAS)~\citep{liu2018darts,xie2018snas,he2020milenas}.  For example, \cite{liu2018darts} proposed a continuous relaxation of the discrete architecture representation, and tremendously accelerated the architecture search via a gradient-based bilevel optimization method named DARTS. \cite{xie2018snas} further proposed a new stochastic reformulation of NAS coupled with a sampling process to address the bias issue of DARTS. \cite{he2020milenas}
reformulated the bilevel objective function of NAS into a mixed-level optimization procedure, and proposed an efficient MiLeNAS method with a lower validation error. We anticipate that the proposed acceleration schemes will be useful for the aforementioned applications. 
%Bilevel optimization has been widely used in 

\vspace{-0.2cm}
\section{Preliminaries on Bilevel Optimization}
\subsection{Bilevel Problem Class}
%In \Cref{sec:he}, we know that existing bilevel optimization algorithms involve both the first-order information (i.e., gradient) as well as second-order information (such as Jacobians $\nabla_x\nabla_y g(\cdot)$ and Hessians $\nabla_y^2 g(\cdot)$). 

%We first introduce the following problem class.  Let 
%$x_\epsilon^* =\argmin_{x\in\mathbb{R}^d} \Phi(x)+\frac{\epsilon}{4B} \|x\|^2$ for any given target accuracy $\epsilon$.
%Let $z:=(x,y)$ for notational simplification. 
%where the inner problem takes the quadratic form w.r.t. $y$ as given by 
%\begin{align}\label{quadratic_case}
%g(x,y) = \frac{\alpha}{2} y^T A(x) y + \beta x^T B(x) y + h(x),
%\end{align}
%where $A(x)$ and $B(x)$ are matrices related to $x$ but independent from $y$. 
%we impose some conditions on the total objective function $\Phi(x)$, the outer- and inner-level functions $f(x,y)$ and $g(x,y)$.  We first introduce the following 

In this section, we introduce the problem class we are interested in. First, we suppose functions $f(x,y)$ and $g(x,y)$ satisfy the following smoothness property. 
\begin{assum}\label{fg:smooth}
The outer-level function $f$ satisfies, for $\forall x_1,x_2,x\in\mathbb{R}^p$ and $y_1,y_2,y\in\mathbb{R}^q$, there exist constants $L_x,L_{xy},L_y\geq 0$ such that
 \begin{align}\label{def:first}
\| \nabla_x  f(x_1,y)-\nabla_x f(x_2,y)\| \leq &L_x \|x_1-x_2\|,\,\| \nabla_x  f(x,y_1)-\nabla_x f(x,y_2)\| \leq L_{xy} \|y_1-y_2\| \nonumber
\\\| \nabla_y  f(x_1,y)-\nabla_y f(x_2,y)\| \leq &L_{xy} \|x_1-x_2\|,\,\| \nabla_y  f(x,y_1)-\nabla_y f(x,y_2)\| \leq L_{y} \|y_1-y_2\|.
\end{align}
The inner-level function $g$ satisfies that, there exist $\widetilde L_{xy},\widetilde L_{y}\geq 0$ such that 
\begin{align}\label{df:sec} 
\| \nabla_y  g(x_1,y)-\nabla_y g(x_2,y)\| \leq \widetilde L_{xy} \|x_1-x_2\|,\,\| \nabla_y  g(x,y_1)-\nabla_y g(x,y_2)\| \leq \widetilde L_{y} \|y_1-y_2\|.
\end{align}
\end{assum}
The hypergradient $\nabla \Phi(x)$ plays an important role for designing bilevel optimization algorithms. The computation of $\nabla\Phi(x)$ involves Jacobians $\nabla_x\nabla_y g(x,y)$ and Hessians $\nabla_y^2 g(x,y)$. In this paper, we are interested in the following inner-level problem with general Lipschitz continuous  Jacobians and Hessians, as adopted by~\cite{ghadimi2018approximation,ji2021bilevel,hong2020two}. For notational convenience, let $z:=(x,y)$ denote both variables.  
\begin{assum}\label{g:hessiansJaco}
There exist constants $\rho_{xy},\rho_{yy}\geq 0$ such that for any $(z_1,z_2)\in\mathbb{R}^p \times \mathbb{R}^q$, 
\begin{align}\label{def:three}
% \nabla_y^2 g (x,y) 	\equiv H, \quad \nabla_x\nabla_y g(x,y) \equiv J,\quad \forall  x\in\mathcal{X}, y\in \mathbb{R}^d.
\|\nabla_x\nabla_y g(z_1)-\nabla_x\nabla_y g(z_2)\| \leq \rho_{xy}\|z_1-z_2\|,\;\; \|\nabla^2_y g(z_1)-\nabla_y^2 g(z_2)\| \leq \rho_{yy}\|z_1-z_2\|. 
\end{align} 
\end{assum} %has also been considered in existing theoretical studies~\citep{ghadimi2018approximation,ji2021bilevel}, 
%Note that in Assumption~\ref{g:hessiansJaco}, Hessian $\nabla_y^2g(x,y)$ is constant w.r.t.~all $y\in\mathbb{R}^q$. At the end of this paper, we discuss how our results are extended to the case where $\nabla_y^2g(x,y)$ is Lipschitz continuous w.r.t.~$y$. 
In this paper, we study the following two classes of bilevel optimization problems. 
\begin{definition}[Bilevel Problem Classes]\label{de:pc}
Suppose $f$ and $g$ satisfy Assumptions~\ref{fg:smooth},~\ref{g:hessiansJaco} and there exists a constant $B>0$ such that  $\|x^*\|=B$, where $x^*\in\argmin_{x\in\mathbb{R}^p}\Phi(x)$.
We define the following two classes of bilevel problems under different geometries.     
\begin{list}{$\bullet$}{\topsep=0.1in \leftmargin=0.1in \rightmargin=0.1in \itemsep =0.01in}
%\begin{itemize}
\item {\bf Strongly-convex-strongly-convex class $\mathcal{F}_{scsc}:$}  $\Phi(\cdot)$ is $\mu_x$-strongly-convex and $g(x,\cdot)$ is $\mu_y$-strongly-convex. 
\item {\bf Convex-strongly-convex class $\mathcal{F}_{csc}:$}  $\Phi(\cdot)$ is convex and $g(x,\cdot)$ is $\mu_y$-strongly-convex. % In addition, $x^*_\epsilon\in\mathcal{X}$, where $x_\epsilon^* =\argmin_{x\in\mathbb{R}^d} \Phi(x)+\frac{\epsilon}{4B} \|x\|^2$ for a given target accuracy $\epsilon$. 
%\item {\bf Nonconvex-strongly-convex class $\mathcal{F}_{nsc}:$}  $\Phi(\cdot)$ is nonconvex and $g(x,\cdot)$ is $\mu_y$-strongly-convex. 
%\end{itemize}
\end{list}
\end{definition}
A simple but important subclass of the bilevel problem class in \Cref{de:pc} includes the following quadratic inner-level functions $g(x,y)$.
%Assumption~\ref{g:hessiansJaco} includes the quadratic inner-level problem, where the inner-level function $g(x,y) $  takes the quadratic form w.r.t.~$y$ as 
%whose second-order derivatives are constant, i.e., there exist constant matrices $H$ and $J$ such that  $\nabla_x\nabla_y g(x,y)=J$ and $\nabla_y^2 g(x,y)=H$. In this case, 
%
%has constant second-order derivatives, i.e., there exist constant matrices $H$ and $J$ such that  $\nabla_x\nabla_y g(x,y)=J$ and $\nabla_y^2 g(x,y)=H$. 
%Note that when the Lipschitz parameters  $\rho_{xx}=\rho_{xy}=0$, the inner-level problem reduces to a quadratic programming, where  Jacobians $\nabla_x\nabla_y g(x,y)$ and Hessians $\nabla_y^2 g(x,y)$ are constant matrices, and the inner-level problem takes the quadratic form w.r.t.~$y$ as 
%In particular, it can be verified that the following simple quadratic function 
\begin{align}\label{quadratic_case}
(\text{Quadratic $g$ subclass:}) \quad g(x,y) = \frac{1}{2} y^T H y +  x^T J y +b^Ty + h(x),   
\end{align}
where the Hessian $H$ and the Jacobian $J$ satisfy $ H \preceq \widetilde L_{y} I$ and $J \preceq \widetilde L_{xy} I$ for $\forall x\in\mathbb{R}^p$ and $\forall y\in\mathbb{R}^q$. 
Note that the above quadratic subclass also covers a large collection of applications such as few-shot meta-learning with shared embedding model~\citep{bertinetto2018meta} and biased regularization in hyperparameter optimization~\citep{grazzi2020iteration}.

\subsection{Algorithm Class for Bilevel Optimization}
Compared to  {\bf minimization} and {\bf minimax} problems, the most different and challenging component of  {\bf bilevel optimization} lies in the computation  of the {\em hypergradient} $\nabla\Phi(\cdot)$. In specific, when functions $f$ and $g$ are continuously twice differentiable, it has been shown in~\cite{foo2008efficient} that  $\nabla\Phi(\cdot)$ takes the form of 
\begin{align}\label{hyperG}
\nabla \Phi(x) =&  \nabla_x f(x,y^*(x)) -\nabla_x \nabla_y g(x,y^*(x)) [\nabla_y^2 g(x,y^*(x)) ]^{-1}\nabla_y f(x,y^*(x)). %\nonumber
%\\&+ \nabla_x f(x_k,y^*(x_k)).
\end{align}
In practice, exactly calculating the Hessian inverse  $(\nabla_y^2 g(\cdot) )^{-1}$ in \cref{hyperG} is computationally infeasible, and hence two types of hypergradient  estimation approaches named AID and ITD have been proposed, where only efficient  {\bf Hessian- and Jacobian-vector products} need to be computed.
We present ITD-based bilevel optimization algorithms as follows, and the introduction of AID-based methods can be found in~\Cref{example:appen}.

\begin{example}[ITD-based Bilevel Algorithms]\label{exam:itd}\citep{maclaurin2015gradient,franceschi2017forward,ji2021bilevel,grazzi2020iteration} Such type of algorithms use ITD-based approaches for hypergradient computation, and take the following updates.  

\vspace{0.2cm}
\noindent For each outer iteration $m=0,....,Q-1$,
\begin{list}{$\bullet$}{\topsep=0.3ex \leftmargin=0.09in \rightmargin=0.in \itemsep =0.01in}
\item  Update variable $y$ for $N$ times via iterative algorithms (e.g., gradient descent, accelerated gradient methods).%using gradient descent by 
%\begin{align}\label{inner_up}
%y_{m}^t =  T(y_m^{t-1}, x_m), t=1,...,N.
%\end{align}
\begin{align}\label{inner_up}
(\text{\normalfont Gradient descent:}) \quad y_{m}^t =  y_m^{t-1} - \eta \nabla_y g(x_m, y_m^{t-1}), t=1,...,N.
\end{align}
\item Compute the hypergradient estimate $G_m=\frac{\partial f(x_m,y_m^N(x_m))}{\partial x_m}$ via backpropagation. Under the gradient updates in~\cref{inner_up}, $G_m$ takes the form of % and given by %takes the form of  
\begin{align}\label{g:omamascsa}
G_m = \nabla_x f(x_m,y^N_m) -\eta\sum_{t=0}^{N-1}\nabla_x\nabla_y g(x_m,y_m^t)\prod_{j=t+1}^{N-1}(I-\eta  \nabla^2_y g(x_m,y_m^{j}))\nabla_y f(x_m,y_m^N).
\end{align}
A similar form holds for case when updating $y$ with accelerated gradient methods. 
\item Update $x$ based on $G_m$ via gradient-based iterative methods.
\end{list}
\end{example}
It can be seen from~\cref{g:omamascsa}  that only Hessian-vector products $ \nabla^2_y g(x_m,y_m^{j})v_j, j=1,...,N$ and Jacobian-vector products $\nabla_x\nabla_y g(x_m,y_m^j)v_{j},j=1,...,N$ are computed, where each $v_j$ is  obtained recursively via   
\begin{align*}
v_{j-1} = \underbrace{(I-\alpha  \nabla^2_y g(x_m,y_m^{j}))v_j}_{\text{Hessian-vector product}} \text{ with } v_N = \nabla_y f(x_m,y_m^N).
\end{align*} 
The same observation applies to AID-based methods as shown in~\Cref{example:appen}. 
% It can be verified that the ITD-based bilevel optimization algorithms in \Cref{exam:itd} belong to the hypergradient-based algorithm class defined in \Cref{alg_class} with $Q=T=N$. Next, we introduce another important class of  AID-based approaches.

We next introduce a general hypergradient-based algorithm class, which includes popular ITD-based (given above) and AID-based (in~\Cref{example:appen}) bilevel optimization algorithms. 
%, which in existing representative algorithms belong to this class. 
%To simplify notations, let $\prod_{n_1}^{n_2}= I$ for any $n_1>n_2$. 
\begin{definition}[Hypergradient-Based Algorithm Class]\label{alg_class}
Suppose there are totally $K$ iterations and $x$ is updated for $Q$ times at iterations indexed by $s_i,i=1,...,Q-1$ with $s_0<...<s_{Q-1}\leq K$. Note that $Q$ is an arbitrary  positive integer in $0,...,K$ and $s_i,i=1,...,Q-1$ are $Q$ arbitrary distinct integers in $0,...,K$. The iterates $\{(x_k,y_k)\}_{k=0,...,K}$ are generated according to $(x_k,y_k)\in\mathcal{H}_x^k, \mathcal{H}_y^k$, where the linear subspaces $\mathcal{H}_x^k, \mathcal{H}_y^k, k=0,...,K$ with $\mathcal{H}_x^{0} = \mathcal{H}_y^0=\{{\bf 0 }\}$ are given as follows. 
\begin{align}\label{hxk}
H_y^{k+1} = \text{\normalfont Span}\left\{y_i,\nabla_y g(\widetilde x_i,\widetilde y_i), \forall \widetilde x_i\in\mathcal{H}_x^i,\forall y_i,\widetilde y_i\in\mathcal{H}_y^i, 1\leq i \leq k \right\}.  
\end{align}
For $x$, we have, for all $m=0,..., Q-1$,  
\begin{align}\label{x_span}
&\mathcal{H}_x^{s_m} =  \text{\normalfont Span} \Big\{x_i,\nabla_xf(\widetilde x_i,\widetilde y_i),\nabla_x\nabla_y g( x_{i}^t, y_{i}^t)\prod_{j=1}^t(I-\alpha \nabla_y^2g(x_{i,j}^t,y_{i,j}^t))\nabla_y f(\hat x_i,\hat y_i), \nonumber
\\&\hspace{1.3cm}t=0,...,T,\forall x_i, \hat x_i, x_{i}^t,x_{i,j}^t \in \mathcal{H}_x^i, \forall \hat y_i, y_{i}^t,y_{i,j}^t \in \mathcal{H}_y^i, 1\leq i \leq s_m-1, \forall \alpha\in\mathbb{R}, T\in \mathbb{N}\Big\}\nonumber
\\ &\mathcal{H}_x^n = \mathcal{H}_x^{s_m}, \forall s_m\leq n\leq s_{m+1}-1 \text{ with } s_Q=K+1. 
\end{align}
%Note that $Q, s_i, i=0,...,Q-1$ are {\bf arbitrary} positive integers. 
\end{definition}
Note that in this algorithm class, $x$ can be updated at any iteration due to the arbitrary choices of $Q,s_i,i=1,...,Q-1$ and the hypergradient estimate can be constructed using any combination of points in the historical search space (similarly for $y$). Moreover, this algorithm class allows to update $x$ and $y$ at the same time or alternatively, and hence include both single- and double-loop bilevel optimization algorithms. Note that the above hypergradient-based algorithm class include popular examples such as HOAG~\citep{pedregosa2016hyperparameter},  AID-FP~\citep{grazzi2020iteration}, reverse~\citep{franceschi2017forward}, $K$-RMD~\citep{shaban2019truncated},  AID-BiO and ITD-BiO~\citep{ji2021bilevel}.

%We introduce the integer $Q$ and the indices $s_i,i=1,...,Q-1$  in \Cref{alg_class}  to simplify the characterization of the subspaces {\small$\{\mathcal{H}_x^k,\mathcal{H}_y^k\}_{k=1}^K$} in the derivations of the lower bounds (see Step 3 in \Cref{appen:thm1}), and we can remove $Q$
%Due to the arbitrary choices of $Q,s_i,i=1,...,Q-1$, it can be checked that 

%In \Cref{example:appen}, we show existing AID- and ITD-based bilevel optimizers belong to the algorithm class in \Cref{alg_class}. 
%It can be seen from~\cref{x_span}  that only Hessian-vector products $ \nabla^2_y g(x^t_{i,j},y^t_{i,j})v_j, j=1,...,t$ and Jacobian-vector products $\nabla_x\nabla_y g(x^t_i,y^t_i)v_{0}$ are computed, where each $v_j$ is  obtained recursively by
%\begin{align*}
%v_{j-1} = \underbrace{(I-\alpha  \nabla^2_y g(x_{i,j}^t,y_{i,j}^t))v_j}_{\text{Hessian-vector product}} {\text{ with }} v_t = \nabla_y f(\hat x_i,\hat y_i).
%\end{align*}
%In \Cref{example:appen}, we show that most existing AID- and ITD-based bilevel optimization algorithms belong to this
%Note that the above hypergradient-based algorithm class also include 
%
%HOAG~\citep{pedregosa2016hyperparameter},  AID-FP~\citep{grazzi2020iteration}, reverse~\citep{franceschi2017forward}, $K$-RMD~\citep{shaban2019truncated},  AID-BiO and ITD-BiO~\citep{ji2021bilevel}.

\subsection{Complexity Measures}
We introduce the criterion for measuring the computational complexity of bilevel optimization algorithms. Note that the updates of $x$ and $y$ of bilevel algorithms involve computing gradients, Jacobian- and Hessian-vector products. In practice, it has been shown in~\cite{griewank1993some,rajeswaran2019meta} that the time and memory cost for computing a  Hessian-vector product $\nabla^2f(\cdot) v $ (similarly for a Jacobian-vector product) via automatic differentiation (e.g., the widely-used reverse mode in PyTorch or TensorFlow)  is no more than a (universal) constant order (e.g., usually $2$-$5$ times) over the cost for computing gradient $\nabla f(\cdot)$. For this reason,  we take the following complexity measures.    
\begin{definition}[Complexity Measure]\label{complexity_measyre}
The total complexity $\mathcal{C}_{\text{\normalfont fun}}(\mathcal{A},\epsilon)$ of a bilevel optimization algorithm $\mathcal{A}$ to find a point $\bar x$ such that the suboptimality gap $f(\bar x)-\min_xf(x)\leq \epsilon$  is given by 
%\begin{align}
$\mathcal{C}_{\text{\normalfont fun}}(\mathcal{A},\epsilon) = \tau (n_J+n_H) + n_G$, 
%\end{align}
where $n_J,n_ H$ and $n_G$ are the total numbers of Jacobian- and Hessian-vector product, and gradient evaluations, and $\tau>0$ is a universal constant.  Similarly, we define $\mathcal{C}_{\text{\normalfont grad}}(\mathcal{A},\epsilon)= \tau (n_J+n_H) + n_G$ as the complexity to find a point $\bar x$ such that the gradient norm $\|\nabla f(\bar x)\|\leq \epsilon$.
\end{definition}

\vspace{-0.2cm}
\section{Lower Bounds for Bilevel Optimization}
%In this section, we provide the lower complexity bounds for bilevel optimization under two different settings: strongly-convex-strongly-convex and  convex-strongly-convex geometries.    
%characterize the fundamental limits of bilevel optimization by providing the first-known  lower complexity bounds. 
\subsection{Strongly-Convex-Strongly-Convex Bilevel Optimization}\label{caoananasca:s}
We first study the case when $\Phi(\cdot)$ is $\mu_x$-strongly-convex and the inner-level function $g(x,\cdot)$ is $\mu_y$-strongly-convex. We present our lower bound result for this case as below.
\begin{theorem}\label{thm:low1} 
Let $M = K+QT+Q + 2$ with $K, T,Q $ given by~\Cref{alg_class}. 
There exists a problem instance in $\mathcal{F}_{scsc}$ defined in~\Cref{de:pc} with dimensions $p=q=d> \max\big\{2M,M+1+\log_{r}\big(\mbox{\normalfont poly}\big(\mu_x\mu_y^2\big)\big)\big\}$ such that for this problem, any output $x^K$ belonging to the subspace $\mathcal{H}_x^K$, i.e., generated by any algorithm in the hypergradient-based algorithm class
defined in~\Cref{alg_class}, satisfies 
\begin{align}\label{result:first}
\Phi(x^K)-\Phi(x^*) \geq \Omega\Big(\mu_x\mu_y^2(\Phi(x_0)-\Phi(x^*)) r^{2M}\Big),
\end{align}
where $x^*=\argmin_{x\in\mathbb{R}^d}\Phi(x)$ and the parameter $r$ satisfies $1-\Big(\frac{1}{2}+\sqrt{\xi+\frac{1}{4}}\Big)^{-1} < r< 1$ with $\xi$ given by 
%\begin{align*}
$\xi\geq \frac{\widetilde L_y}{4\mu_y} +\frac{L_x}{8\mu_x} +\frac{L_y\widetilde L_{xy}^2}{8\mu_x\mu_y^2} -\frac{3}{8}\geq \Omega \big(\frac{1}{\mu_x\mu_y^2}\big)$. 
%\end{align*}
To achieve $\Phi(x^K)-\Phi(x^*)\leq \epsilon$,  the total complexity $\mathcal{C}_{\text{\normalfont fun}}(\mathcal{A},\epsilon)$ satisfies 
\begin{align*}
\mathcal{C}_{\text{\normalfont fun}}(\mathcal{A},\epsilon) \geq \Omega\bigg(\sqrt{\frac{L_y\widetilde L_{xy}^2}{\mu_x\mu_y^2}}\log  \frac{\mu_x\mu_y^2(\Phi(x_0)-\Phi(x^*))}{\epsilon}  \bigg).
\end{align*}
\end{theorem}
Note that the inner-level function $g(x,y)$ in our constructed worst-case instance takes the same quadratic form as in \cref{quadratic_case} so that the lower bound in \Cref{thm:low1} also applies to the quadratic $g$ subclass. 
We provide a proof sketch of \Cref{thm:low1} as follows, and present the complete proof in~\Cref{appen:thm1}. 
 
\subsection*{Proof Sketch of \Cref{thm:low1}}
The proof of  \Cref{thm:low1} is divided into four main steps: 1) constructing a worst-case instance $(f,g)\in\mathcal{F}_{scsc}$; 2) characterizing the optimal point $x^*=\argmin_{x\in\mathbb{R}^d}\Phi(x)$; 3) characterizing the subspaces $\mathcal{H}_x^K,\mathcal{H}_y^K$; and 4) lower-bounding the convergence rate and complexity. 

\vspace{0.1cm}
\noindent {\bf Step 1 (construct a worse-case instance):} We first construct the following instance functions $f$ and $g$.
 \begin{align}\label{str_fg}
f(x,y) &= \frac{1}{2} x^T(\alpha Z^2 +\mu_x I) x -\frac{\alpha\beta}{\widetilde L_{xy}}x^TZ^3y +\frac{\bar L_{xy}}{2}x^TZy + \frac{L_y}{2}\|y\|^2 + \frac{\bar L_{xy}}{\widetilde L_{xy}} b^T y, \nonumber
\\g(x,y) &= \frac{1}{2} y^T (\beta Z^2 +\mu_y I)y -\frac{\widetilde L_{xy}}{2} x^TZy + b^Ty, 
\end{align} %the radius $B$ is chosen such that $\|x^*\|<B$ (hence $x^*=\argmin_{x\in\mathbb{R}^d}\Phi(x)$), 
where $\alpha=\frac{L_x-\mu_x}{4}$, $\beta=\frac{\widetilde L_y-\mu_y}{4}$, and the coupling matrices  $Z,Z^2,Z^4$ take the forms of 
{\footnotesize
\begin{align}\label{matrices_coupling}
Z= \begin{bmatrix}
 &   &  & 1\\
% &   & & 1 & -1\\
 &  &  1& -1  \\
  &\text{\reflectbox{$\ddots$}} &\text{\reflectbox{$\ddots
  $}} &  \\
  1& -1  &  & \\ 
\end{bmatrix},
Z^2 =
\begin{bmatrix}
 1& -1 &  &  & \\
 -1& 2  &-1 &  & \\
 &   \ddots & \ddots & \ddots  & \\
  && -1& 2 & -1 \\
  &  & & -1 & 2\\ 
\end{bmatrix},
Z^4 =  
\begin{bmatrix}
 2& -3 & 1 &  & &\\
 -3& 6  &-4 &1  & &\\
 1& -4 & 6 & -4 & 1&\\
  &  \ddots&  \ddots&  \ddots &  \ddots & \ddots\\
   &  & 1& -4 &  6& -4 \\
  &  & & 1 & -4 & 5\\ 
\end{bmatrix}.
\end{align}
}
\hspace{-0.14cm}The above matrices play an important role in developing lower bounds due to their following zero-chain properties~\citep{nesterov2003introductory,zhang2019lower}. Let $\mathbb{R}^{k,d}= \{x\in\mathbb{R}^d | x_i=0 \text{ for } k+1\leq i \leq d\}$, where $x_i$ denotes the $i^{th}$ coordinate of the vector $x$. 
%Note that our construction in \cref{str_fg} also falls into the quadratic subclass in \cref{quadratic_case}.   

\begin{lemma}[Zero-Chain Property]\label{zero_chain}
For any given vector $v\in\mathbb{R}^{k,d}$, we have $Z^2v\in\mathbb{R}^{k+1,d}$.
\end{lemma}
\Cref{zero_chain} indicates  that if a vector $v$ has nonzero entries only at the first $k$ coordinates, then multiplying it with a matrix $Z^2$ has at most one more nonzero entry at position $k+1$. We demonstrate the validity of the constructed instance by 
showing that  $f$ and $g$ in \cref{str_fg} satisfy Assumptions~\ref{fg:smooth} and~\ref{g:hessiansJaco}, and $\Phi(x)$ is $\mu_x$-strongly-convex.  

\vspace{0.2cm}
 \noindent {\bf Step 2 (characterize the minimizer $x^*$):} We show that the unique minimizer $x^*$ satisfies the following equation   
 \begin{align}\label{refere:eqs}
 Z^4 x^* +\lambda Z^2x^* +\tau x^* = \gamma Zb,
\end{align}
where $\lambda=\Theta(1)$ and $\gamma=\Theta(1)$, $\tau = \Theta(\mu_x\mu_y^2)$. We  choose $b$ in \cref{refere:eqs} such that $(Zb)_t=0$  for all $t\geq 3$, which is feasible because we show that $Z$ is invertible. Using the structure of $Z$ in \cref{matrices_coupling}, we show that there exists a vector $\hat x$ with its $i^{th}$ coordinate $\hat x_i = r^i$ such that 
\begin{align}\label{sketch_p}
\|x^*-\hat x\| \leq \mathcal{O}(r^d),
\end{align}
where $0<r<1$ satisfies  $1-r=\Theta(\mu_x\mu_y^2)$. Then, based on the above  \cref{sketch_p}, we are able to characterize $x^*$, e.g., its norm $\|x^*\|$, using its approximate (exponentially close) $\hat x$. 

\vspace{0.2cm}
 \noindent {\bf Step 3 (characterize the iterate subspaces):}  By exploiting the forms of  the subspaces {\small$\{\mathcal{H}_x^k,\mathcal{H}_y^k\}_{k=1}^K$}
defined in \Cref{alg_class}, we use the induction to show that 
%\begin{align}
$$H_{x}^{K }\subseteq \mbox{Span}\{Z^{2(K+QT+Q)}(Zb),....,Z^2(Zb),(Zb) \}.$$
% \end{align}
 Then, noting that $(Zb)_t=0$  for all $t\geq 3$ and using the zero-chain property of $Z^2$, we have the $t^{th}$ coordinate of the output $x^K$ to be zero, i.e., $(x^K)_{t}=0$,  for all $t\geq M+1$.
% , where $M=K+QT+Q + 2$. 
 
 \vspace{0.2cm}
 \noindent {\bf Step 4 (combine Steps $1,2,3$ and characterize the complexity):} By choosing $d> \max\big\{2M,M+1+\log_{r}\big(\frac{\tau}{4(7+\lambda)}\big)\big\}$, and based on {\bf Steps} 2 and 3, we have  
% \begin{align*}
$ \|x^K-x^*\| \geq \frac{\|x^*-x_0\|}{3\sqrt{2}} r^M$
% \end{align*}
which, in conjunction with the form of $\Phi(x)$, yields the result in~\cref{result:first}. The complexity result then follows because $1-r=\Theta(\mu_x\mu_y^2)$ and from the definition of the complexity measure in~\Cref{complexity_measyre}.

\vspace{0.2cm}
{\noindent \bf Remark.} We note that the introduction of the term $\frac{\alpha\beta}{\widetilde L_{xy}}x^TZ^3y$ in $f$ is necessary to obtain the lower bound $\widetilde \Omega (\mu_x\mu_y^{2})$. Without such a term, there will be an additional high-order term $\Omega(A^6x)$ at the left hand side of \cref{refere:eqs}. Then, following the same steps as in Step 2, we would obtain a result similar to \cref{sketch_p}, but with a parameter $r$ satisfying 
%\begin{align*}
$ 0<\frac{1}{1-r}<\mathcal{O}\big( {\frac{1}{\mu_x\mu_y}} \big).$
%\end{align*}
Then, following the same steps as in Steps 3 and 4, the final overall complexity $\mathcal{C}_{\text{\normalfont fun}}(\mathcal{A},\epsilon)\geq \Omega\big( {\frac{1}{\mu_x\mu_y}} \big)$, which is not as tight as 
 $\Omega\big( {\frac{1}{\mu_x\mu^2_y}} \big)$ obtained under the selection in \cref{str_fg}.

\subsection{Convex-Strongly-Convex Bilevel Optimization}
We next characterize the lower complexity bound for the convex-strongly-convex setting, where $\Phi(\cdot)$ is {\bf convex} and the inner-level function $g(x,\cdot)$ is $\mu_y$-strongly-convex.
% For this case, we study the complexity for achieving an $\epsilon$-stationary point, i.e., $\|\nabla \Phi(x)\|\leq \epsilon$. For the 
%, which is also a standard criterion for convex optimization as adopted in~\cite{carmon2019lower}. 
We state our main result for this case in the following theorem.
% Let $x^*$  be a minimizer of $\Phi(\cdot)$, i.e., $x^*\in\argmin_{x\in\mathbb{R}^p}\Phi(x)$
\begin{theorem}\label{main:convex}
Let $M = K+QT-Q+3$ with $K, T,Q $ given by~\Cref{alg_class}, and let $x^K$ be an output belonging to the  subspace $\mathcal{H}_x^K$, i.e., generated by any algorithm in the hypergradient-based algorithm class
defined in~\Cref{alg_class}.  
There exists an instance in $\mathcal{F}_{csc}$ defined in~\Cref{de:pc} with dimensions $p=q=d$ such that in order to achieve $\|\nabla\Phi(x^K)\|\leq \epsilon$, it requires 
%\begin{align}
$M \geq \lfloor r^*\rfloor -3$, 
% \end{align}
 where $r^*$ is the solution of the equation 
 \begin{align}\label{eq:rsoltion}
 r^4 +r \Big(\frac{2\beta^4}{\mu_y^4}+\frac{4\beta^3}{\mu_y^3}+\frac{4\beta^2}{\mu_y^2}\Big) = \frac{B^2(\widetilde L^2_{xy}L_y+L_x\mu_y^2)^2}{128\mu_y^4\epsilon^2},
 \end{align}
 where $\beta=\frac{\widetilde L_y-\mu_y}{4}$ and $B$ is given in \Cref{de:pc}. The complexity satisfies $\mathcal{C}_{\text{\normalfont grad}}(\mathcal{A},\epsilon)\geq \Omega(r^*)$.
\end{theorem}
Note that \Cref{main:convex} uses the gradient norm $\|\nabla \Phi(x)\|\leq \epsilon$ rather than the suboptimality gap $\Phi(x^K)-\Phi(x^*)$ as the convergence criteria. This is because for the convex-strongly-convex case,  lower-bounding the suboptimality gap requires  the Hessian matrix $A$ in the worst-case construction of the total objective function $\Phi(x)$ to have a nice structure, e.g., the solution of $A^\prime x = e_1$ ($e_1$ has a single non-zero value $1$ at the first coordinate) is explicit, where $A^\prime$ is derived by removing last $k$ columns and rows of $A$.  However, in bilevel optimization, $A$ often contains different powers of the zero-chain matrix $Z$, and does not have such a structure. We will leave the lower bound under the suboptimality criteria for the future study. Note that $r^*$ in \Cref{main:convex} has a complicated form. 
The following two corollaries simplify the complexity results by considering specific parameter regimes. 
\begin{corollary}\label{co:co1}
Under the same setting of \Cref{main:convex}, consider the case when $\beta\leq \mathcal{O}(\mu_y)$. Then, we have 
%\begin{small}
%\begin{align*}
$\mathcal{C}_{\text{\normalfont grad}}(\mathcal{A},\epsilon)\geq \Omega \big(\frac{B^{\frac{1}{2}}(\widetilde L^2_{xy}L_y+L_x\mu_y^2)^{\frac{1}{2}}}{\mu_y\epsilon^{\frac{1}{2}}} \big)$. 
%\end{align*}
%\end{small}
\end{corollary}
\begin{corollary}\label{co:co2}
Under the same setting of \Cref{main:convex}, consider the case when $\beta \leq\mathcal{O}(1)$, i.e., at a constant level. Then, we have {\small $\mathcal{C}_{\text{\normalfont grad}}(\mathcal{A},\epsilon)\geq \widetilde \Omega(\frac{1}{\sqrt{\epsilon}}\min\{\frac{1}{\mu_y},\frac{1}{\sqrt{\epsilon^{3}}}\})$}.
\end{corollary}
%From \Cref{co:co1}, it can be still seen that  the complexity results for bilevel optimization depend on  $\frac{1}{\mu_y}$, which is order-wisely larger than $\frac{1}{\sqrt{\mu_y}}$ for conventional minimization or minimax optimization problems. The similar phenomenon can be observed for the case in \Cref{co:co2}.
%where bilevel optimization requires a worse $\epsilon^{-2}$ dependence than the conventional $\epsilon^{-1/2}$. 
%The proof of \Cref{main:convex} uses a coupling matrix $Z$  different from that for \Cref{thm:low1}, and exploits the relation between the gradient norm $\|\nabla \Phi(x)\|$ and $Z^2,Z^4,Z^6$. 
%is essentially different from \Cref{thm:low1} with strongly-convex $\Phi$, 
The proof sketch of \Cref{main:convex} is provided as follows. The complete proof is provided in \Cref{apep:mainconvexs}. 

%\vspace{-0.2cm}
\subsection*{Proof Sketch of \Cref{main:convex}}
{\bf Step 1(construct the worst-case instance):} We construct the instance functions $f$ and $g$ as follows.
\begin{align}\label{con_fg}
f(x,y) &= \frac{L_x}{8} x^T Z^2 x+ \frac{L_y}{2}\|y\|^2, \nonumber
\\g(x,y) &= \frac{1}{2} y^T (\beta Z^2 +\mu_y I)y -\frac{\widetilde L_{xy}}{2} x^TZy + b^Ty, 
\end{align}
where $\beta=\frac{\widetilde L_y-\mu_y}{4}$.  Here, the coupling matrix $Z$ is different from that~\cref{matrices_coupling} for the strongly-convex-strongly-convex case, which takes the form of 
\begin{align}\label{matrices_coupling_convex}
Z:=\begin{bmatrix}
   & &  1& -1\\
   &1&  -1&   \\
 \text{\reflectbox{$\ddots $}}  &\text{\reflectbox{$\ddots$}}  &  \\
% 1 & -1  &  & \\
  -1&   &  & \\ 
\end{bmatrix},\quad
&Z^2 := 
\begin{bmatrix}
 2& -1 &  &  & \\
 -1& 2  &-1 &  & \\
 &   \ddots & \ddots & \ddots  & \\
  && -1& 2 & -1 \\
  &  & & -1 & 1\\ 
\end{bmatrix}.
%Z^4 : 
%\begin{bmatrix}
% 5& -4 & 1 &  & &\\
% -4& 6  &-4 &1  & &\\
% 1& -4 & 6 & -4 & 1&\\
%  &  \ddots&  \ddots&  \ddots &  \ddots & \ddots\\
%   &  & 1& -4 &  6& -3 \\
%  &  & & 1 & -3 & 2\\ 
%\end{bmatrix}
\end{align}
It can be verified that $Z$ is invertible and $Z^2$ in  \cref{matrices_coupling_convex} also satisfies the zero-chain property, i.e., \Cref{zero_chain}. 
We can further verify that  $\Phi(x)$ is convex and functions $f,g$ satisfy Assumptions~\ref{fg:smooth} and~\ref{g:hessiansJaco}.  

 \vspace{0.2cm}
\noindent {\bf Step 2 (characterize the minimizer $x^*$):} Recall that $x^*\in \argmin_{x\in\mathbb{R}^d}\Phi(x)$. We then show that $x^*$ satisfies the equation 
\begin{align*}
\Big (\frac{L_x\beta^2}{4} Z^6 +  \frac{L_x\beta^2\beta\mu_y}{2} Z^4 +\Big (\frac{L_y\widetilde L_{xy}^2}{4}+\frac{L_x\mu_y^2}{4}\Big) Z^2 \Big) x^* =  \frac{L_y\widetilde L_{xy}}{2} Zb. 
\end{align*}
Let $\widetilde b =  \frac{L_y\widetilde L_{xy}}{2} Zb$ and choose $b$  such that $\widetilde b_t=0$ for all $t\geq 4$. Then, by choosing $\widetilde b_1,\widetilde b_2, \widetilde b_3$ properly, we derive that $x^* = \frac{B}{\sqrt{d}} {\bf 1}$, where ${\bf 1}$ is the all-one vector, and hence $\|x^*\|= B$. % the proof as follows.  
%become $\epsilon^{-2}$, which is still much worse than $\epsilon^{-1/2}$ for conventional minimization  or minimization 

 \vspace{0.2cm}
\noindent  {\bf Step 3 (characterize the gradient norm):} In this step, we show that for any $x$ whose last three coordinates are zeros, the gradient norm of $\nabla\Phi(x)$ is lower-bounded. Namely, we prove that 
\begin{align}\label{eq:sehcsa}
\min_{x\in\mathbb{R}^d: \;x_{d-2}=x_{d-1}=x_d=0} \|\nabla\Phi(x)\|^2 \geq \frac{B^2\Big(\frac{\widetilde L^2_{xy}L_y}{4}+\frac{L_x\mu_y^2}{4}\Big)^2}{8\mu_y^4d^4 +16d\beta^4+32d\beta^3\mu_y+32d\beta^2\mu_y^2}.
\end{align}

 \vspace{0.2cm}
\noindent{\bf Step 4 (characterize the iterate subspaces):} By exploiting the forms of  the subspaces $\{\mathcal{H}_x^k,\mathcal{H}_y^k\}_{k=1}^K$
defined in \Cref{alg_class} and by induction, we show that 
%\begin{align}
$$H_{x}^{K }\subseteq \mbox{Span}\{Z^{2(K+QT-Q)}(Zb),....,Z^2(Zb),(Zb)\}.$$
% \end{align}
Since  $(Zb)_t=0$  for all $t\geq 4$ and using the zero-chain property of $Z^2$, we have the $t^{th}$ coordinate of the output $x^K$ is zero, i.e., $(x^K)_{t}=0$,  for all $t\geq M+1$, where $M=K+QT-Q + 3$. 

 \vspace{0.2cm}
\noindent  {\bf Step 5 (combine Steps $1,2,3,4$ and characterize the complexity):} Choose $d$ such that the right hand side of \cref{eq:sehcsa} equals $\epsilon$ by  solving \cref{eq:rsoltion}. Then, using the results in Steps 3 and 4, it follows that for any $M\leq d-3$, $\|\nabla \Phi(x^K)\|\geq \epsilon$. Thus, to achieve $\|\nabla \Phi(x)\|\leq \epsilon$ , it requires $M>d-3$ and the complexity result follows as  $\mathcal{C}_{\text{\normalfont grad}}(\mathcal{A},\epsilon) \geq  \Omega(M)$.

\section{Accelerated Gradient Method and Upper Bounds for Bilevel Optimization} \label{upper_withoutB}
In this section, we propose a new bilevel optimization algorithm, and characterize its computational complexity, which serves as new upper bounds for bilevel optimization.
% under Assumptions~\ref{fg:smooth} and~\ref{g:hessiansJaco}, which impose standard Lipschitz conditions on the first- and seecond-order derivatives of $f,g$. 
 %without making any assumption on the boundedness of gradient.  
%However, existing results in \cite{ghadimi2018approximation,ji2021bilevel} for bilevel optimization further assume that the gradient norm $\|\nabla_y f(x,y)\|$ is bounded for all $(x,y)\in\mathbb{R}^p\times\mathbb{R}^q$, and hence their comparison to our result in \Cref{upper_srsr_withnoB} is unfair. In the next section, we address this concern by developing upper bounds under this additional bounded gradient assumption. 
\subsection{Accelerated Bilevel Optimization Algorithm: AccBiO}
As shown in \Cref{alg:bioNoBG}, we propose a new accelerated algorithm named AccBiO for bilevel optimization.  
At the beginning of each outer iteration, we run $N$ steps of accelerated gradient descent (AGD) to get $y_k^N$ as an approximate of $y_k^*=\argmin_y g(x_k,y)$. Then, based on the inner-level output $y_k^N$, we construct a hypergradient estimate via $G_k:= \nabla_x f(x_k,y_k^N) -\nabla_x \nabla_y g(x_k,y_k^N)v_k^M$, where $v_k^M$ is 
  the output of an $M$-step heavy ball method with stepsizes $\eta$ and $\theta$ for solving a quadratic problem as shown in line $7$. Finally, as shown in lines $8$-$9$, we update the variables $z_{k}$ and $x_k$  using  Nesterov's momentum acceleration scheme~\citep{nesterov2018lectures} over the estimated hypergradient $G_k$.
  Next, we analyze  the convergence and complexity performance of AccBiO for the two bilevel optimization classes $\mathcal{F}_{scsc}$ and $\mathcal{F}_{csc}$ described in \Cref{de:pc}.

\begin{algorithm}[t]
	\caption{Accelerated Bilevel Optimization (AccBiO) Algorithm} 
	\small
	\label{alg:bioNoBG}
	\begin{algorithmic}[1]
		\STATE {\bfseries Input:}  Initialization $ z_0=x_0=y_0=0$, parameters  $\lambda$ and $\theta$ 
		\FOR{$k=0,1,...,K$}
		\STATE{Set $y_k^0 = 0 $ as initialization}
		\FOR{$t=1,....,N$}
		\STATE{
%		Run $N$ steps of AGD to $g(\widetilde x_k,\cdot)$ with condition number $\kappa_y$:
		
		\vspace{-0.45cm} 
		\begin{align*}
		 \quad y_k^{t} &= s_k^{t-1} - \frac{1}{\widetilde L_y} \nabla_y g(x_k,s_k^{t-1}), \; s_k^{t} = \frac{2\sqrt{\kappa_y}}{\sqrt{\kappa_y}+1} y_k^{t} - \frac{\sqrt{\kappa_y}-1}{\sqrt{\kappa_y}+1} y_k^{t-1}.
		\end{align*}
		\vspace{-0.7cm} }
		\ENDFOR
                 \STATE{ \textit {Hypergradient computation}: \\
%\begin{list}{$\bullet$}{\topsep=0.3ex \leftmargin=0.3in \rightmargin=0.in \itemsep =0.01in}
               \hspace{0.4cm}1) Get $v_k^M$ after running $M$ steps of heavy-ball method $v_k^{t+1} = v_k^t-\lambda\nabla Q(v_k^t) +\theta(v_k^t-v_k^{t-1})$ \\ \hspace{0.7cm} with initialization $v_k^{0}=v_k^{1}=0$ over \vspace{-0.2cm} 
               \begin{align*}
            \min_v Q(v):=\frac{1}{2}v^T\nabla_y^2 g(x_k,y_k^N) v - v^T
\nabla_y f( x_k,y^N_k)
\end{align*}\vspace{-0.4cm}\\
    \hspace{0.4cm}2) Compute $\nabla_x \nabla_y g(x_k,y_k^N)v_k^M $ via automatic differentiation; \\
    \vspace{0.1cm}
    \hspace{0.4cm}3) compute $G_k:= \nabla_x f(x_k,y_k^N) -\nabla_x \nabla_y g( x_k,y_k^N)v_k^M.$ 
              }
               \vspace{0.1cm}
               \STATE{Update $z_{k+1}=x_k -\frac{1}{L_\Phi} G_k$}
               \vspace{0.1cm}
               \STATE{Update $x_{k+1}=\Big(1+\frac{\sqrt{\kappa_x}-1}{\sqrt{\kappa_x}+1}\Big)z_{k+1} - \frac{\sqrt{\kappa_x}-1}{\sqrt{\kappa_x}+1} z_k$}                          
		\ENDFOR
	\end{algorithmic}
	\end{algorithm}

%\vspace{-0.1cm}	
\subsection{Strongly-Convex-Strongly-Convex Bilevel Optimization}	
 In this setting, $\Phi(x)$ is $\mu_x$-strongly-convex and $g(x,\cdot)$ is $\mu_y$-strongly-convex. 
 The following theorem provides a performance guarantee for AccBiO. Recall $x^*=\argmin_{x}\Phi(x)$.
\begin{theorem}\label{upper_srsr_withnoB}  
Suppose that $(f,g)$ belong to the strongly-convex-strongly-convex class $\mathcal{F}_{scsc}$ in \Cref{de:pc}. Choose stepsizes {\small $\lambda=\frac{4}{(\sqrt{\widetilde L_y}+\sqrt{\mu_y})^2}$} and {\small $\theta=\max\big\{\big(1-\sqrt{\lambda\mu_y}\big)^2,\big(1-\sqrt{\lambda\widetilde L_y}\big)^2\big\}$} for the heavy-ball method.  Let $\kappa_y=\frac{\widetilde L_y}{\mu_y}$ be the condition number for the inner-level function $g(x,\cdot)$ and $L_\Phi =\Theta\big(\frac{1}{\mu_y^2}+\big(\frac{ \rho_{yy}}{\mu_y^3} +  \frac{\rho_{xy}}{\mu_y^2}\big)\big(\Delta^*_{\text{\normalfont\tiny SCSC}}+\frac{\sqrt{\epsilon}}{\sqrt{\mu_x}\mu_y}\big) \big) $ be the smoothness parameter of the objective  $\Phi(\cdot)$, where $\Delta^*_{\text{\normalfont\tiny SCSC}}=\|\nabla_y f( x^*,y^*(x^*))\|+\frac{\|x^*\|}{\mu_y}+\frac{\sqrt{\Phi(0)-\Phi(x^*)}}{\sqrt{\mu_x}\mu_y}$. Then, we have 
\begin{align*}
\Phi(z_K)- \Phi(x^*)\leq  \Big(1 -\frac{1}{\sqrt{\kappa_x}} \Big)^{K}(\Phi(0) -\Phi(x^*)+\frac{\mu_x}{2} \|x^*\|^2) +\frac{\epsilon}{2},
\end{align*}
where $\kappa_x=\frac{L_\Phi}{\mu_x}$ is the condition number for $\Phi(\cdot)$. To achieve $\Phi(z_K)- \Phi(x^*)<\epsilon$, the complexity satisfies 
\begin{align}\label{mybabasohotssca1}
\mathcal{C}_{\text{\normalfont fun}}(\mathcal{A},\epsilon)\leq \mathcal{\widetilde O}\bigg(\sqrt{\frac{\widetilde L_y}{\mu_x\mu_y^{3}}}+\Big(\sqrt{\frac{ \rho_{yy}\widetilde L_y}{\mu_x\mu_y^{4}} }+ \sqrt{ \frac{\rho_{xy}\widetilde L_y}{\mu_x\mu_y^{3}}}\Big)\sqrt{\Delta^*_{\text{\normalfont\tiny SCSC}}}\bigg).
\end{align}
\end{theorem}	
To the best of our knowledge, our result in \Cref{upper_srsr_withnoB} is the first-known upper bound on the computational complexity 
for strongly-convex bilevel optimization under only mild assumptions on the Lipschitz continuity of the first- and second-order derivatives of the outer- and inner-level functions $f,g$. As a comparison, existing results in \cite{ghadimi2018approximation,ji2021bilevel} for bilevel optimization further make a strong assumption that the gradient norm $\|\nabla_y f(x,y)\|$ is bounded for all $(x,y)\in\mathbb{R}^p\times\mathbb{R}^q$ to upper-bound the smoothness parameter $L_{\Phi_k}$ of $\Phi(x_k)$ and the hypergradient estimation error $\|G_k-\nabla \Phi(x_k)\|$ at the $k^{th}$ iteration. This is because  $L_{\Phi_k}$ and $\|G_k-\nabla \Phi(x_k)\|$ turn out to be increasing with the gradient norm $\|\nabla_y f(x_k,y^*(x_k))\|$, for which it is challenging to prove the boundedness given the theoretical frameworks in \cite{ghadimi2018approximation,ji2021bilevel} where no results on bounded iterates are established.  Our analysis does not require such a restrictive assumption because we show by induction that the optimality gap $\|x_k-x^*\|$ is well bounded as the algorithm runs. As a result, we can guarantee the boundedness of the smoothness parameter $L_{\Phi_k}$ and the error $\|G_k-\nabla \Phi(x_k)\|$ during the entire optimization process.  
In \Cref{sec:upper_gBscss}, we further develop tighter upper bounds than existing results under this additional bounded gradient assumption. 

%Due to this additional bounded gradient assumption, existing upper bounds cannot fairly compare to our result in \Cref{upper_srsr_withnoB}.  In \Cref{sec:upper_gBscss}, we further provide tighter upper bounds under this additional bounded gradient assumption. In \Cref{sec:upper_gBscss}, we address this concern by developing upper bounds under this additional bounded gradient assumption. 
%Note that the second component of the complexity result in~\cref{mybabasohotssca1} is caused by 

%The second term of the complexity in~\cref{mybabasohotssca1} is caused by the Jacobian and Hessian inverse components in the hypergradient  
%upper-bounding  the smoothness parameter and the hypergradient estimation error for the 
%originates from upper-bounding the smoothness parameter of $\Phi(x_k)$ at the $k^{th}$ iteration and the hypergradient estimation error $\|G_k-\nabla \Phi(x)_k)\|$ for 
Based on \Cref{upper_srsr_withnoB}, we next study the quadratic $g$ subclass, where the inner-level function $g(x,y)$ takes a quadratic form as in \cref{quadratic_case}. 
%a special class of bilevel problem in \Cref{de:pc}, where the Jacobians $\nabla_x\nabla_y g(x,y)$ and Hessians $\nabla_y^2 g(x,y)$ are {\bf constant} matrices. This includes the quadratic case where the inner-level function $g(x,y)$ is quadratic w.r.t.~$y$, as shown in  \cref{quadratic_case}. 
The following corollary provides upper bounds on the convergence rate and complexity of AccBiO under this case.
\begin{corollary}[\bf Quadratic $g$ subclass]\label{coro:quadaticSr}
Under the same setting of \Cref{upper_srsr_withnoB}, consider the quadratic inner-level function $g(x,y)$ in \cref{quadratic_case}, where $\nabla_x\nabla_y g(\cdot,\cdot)$ and $\nabla_y^2 g(\cdot,\cdot)$ are constant. To achieve $\Phi(z_K)- \Phi(x^*)<\epsilon$, 
the complexity satisfies 
%\begin{align}
$\mathcal{C}_{\text{\normalfont fun}}(\mathcal{A},\epsilon)\leq \mathcal{\widetilde O}\Big(\sqrt{\frac{\widetilde L_y}{\mu_x\mu_y^{3}}}\Big).$
%\end{align}
\end{corollary}
\Cref{coro:quadaticSr} shows that for the quadratic $g$ subclass, the complexity upper bound in \Cref{upper_srsr_withnoB} specializes to $\mathcal{\widetilde O}\Big(\sqrt{\frac{\widetilde L_y}{\mu_x\mu_y^{3}}}\Big)$. This improvement over the complexity for the general case in~\cref{mybabasohotssca1} comes from tighter upper bounds on the smoothness parameter $L_\Phi$ of the objective function $\Phi(x)$ and a smaller hypergradient estimation error $\|G_k-\nabla\Phi(x_k)\|$.  
In addition, it can be seen that when the inner-level problem  is easy to solve, i.e., $\widetilde L_y\leq\mathcal{O}(\mu_y)$, the complexity becomes $\mathcal{O}\big(\frac{1}{\sqrt{\mu_x\mu_y^2}}\big)$, which matches the lower bound established by \Cref{thm:low1} up to logarithmic factors. 

\subsection{Convex-Strongly-Convex Bilevel Optimization}	
We next provide an upper bound for convex-strongly-convex bilevel optimization, where the  function $\Phi(x)$ is convex. Recall from \Cref{de:pc} that $\|x^*\|=B$ for some constant $B>0$, where $x^*$ is one minimizer of $\Phi(\cdot)$.
For this case, we construct a strongly-convex-strongly-convex function $\widetilde \Phi(\cdot)=\widetilde f(x,y^*(x))$ by adding a small quadratic regularization to the outer-level function $f(x,y)$, i.e., 
\begin{align}\label{regularized_fxy}
\widetilde f(x,y) = f(x,y) +\frac{\epsilon}{2R} \|x\|^2. 
\end{align}
%where $D$ is the upper bound of $\|x^*\|$. 	
%By \Cref{de:pc}, we have $x_\epsilon^*\in\mathcal{X}$, where $x_\epsilon^*=\argmin_x \Phi_\epsilon(x)$. 
Then, we can apply the results in \Cref{upper_srsr_withnoB} to $\widetilde\Phi(x)$, and obtain the following theorem. %Let $x^*$  be a minimizer of $\Phi(\cdot)$, i.e., $x^*\in\argmin_{x\in\mathbb{R}^p}\Phi(x)$. 
\begin{theorem}\label{th:upper_csc1sc}
Suppose that $(f,g)$ belong to the convex-strongly-convex class $\mathcal{F}_{csc}$ in \Cref{de:pc}. Let $L_{\widetilde \Phi}$ be the smoothness parameter of function $\widetilde \Phi(\cdot)$, which takes the same form as $L_\Phi$ in \Cref{upper_srsr_withnoB} except that $L_x,f,x^*$ and $\Phi$ become $L_x+\frac{\epsilon}{R},\widetilde f,\widetilde x^*$ and $\widetilde \Phi$, respectively. Let  $\Delta^*_{\text{\normalfont\tiny CSC}} = \|\nabla_y f( x^*,y^*(x^*))\|+\frac{\|x^*\|}{\mu_y}+\frac{(\|x^*\|+1)\sqrt{(\Phi(0)-\Phi(x^*))}}{\sqrt{\epsilon}\mu_y}$.  We consider two widely-used convergence criterions as follows. 
\begin{list}{$\bullet$}{\topsep=0.1in \leftmargin=0.16in \rightmargin=0.in \itemsep =0.01in}
\item {\bf (Suboptimality gap)} Choose $R=B^2$ in \cref{regularized_fxy}, and choose the same parameters as in \Cref{upper_srsr_withnoB} with $\epsilon$ and $\mu_x$ being replaced by $\epsilon/2$ and $\frac{\epsilon}{R}$, respectively.
%$\mu_x=\frac{\epsilon^\prime}{B}$ and $\epsilon=\epsilon^\prime/2$, where $B$ is an upper bound of $\|x^*\|^2$, i.e., $\|x^*\|^2\leq B$. 
To achieve $\Phi(z_K) - \Phi(x^*)\leq \epsilon$, the required complexity is at most
 \begin{align*}
\mathcal{C}_{\text{\normalfont fun}}(\mathcal{A},\epsilon) \leq  \mathcal{O}\Big( B\Big( \sqrt{\frac{\widetilde L_y}{\epsilon\mu_y^3}}+\Big(\sqrt{\frac{\rho_{yy}\widetilde L_y}{\epsilon\mu_y^{4}}} +  \sqrt{\frac{\rho_{xy}\widetilde L_y}{\epsilon\mu_y^3}}\Big)\sqrt{\Delta^*_{\text{\normalfont\tiny CSC}}}\Big)\log\, {\text{\normalfont poly}(\epsilon,\mu_x,\mu_y,\Delta^*_{\text{\normalfont\tiny CSC}})}\Big).
\end{align*}
\item {\bf (Gradient norm)} Choose $R=B$ in \cref{regularized_fxy}, and choose the same parameters as in \Cref{upper_srsr_withnoB} with 
 $\epsilon$ and $\mu_x$ being replaced by $\epsilon^2/(4L_{\widetilde \Phi}+ \frac{8\epsilon}{R})$ and $\frac{\epsilon}{R}$, respectively. 
%$\mu_x=\frac{\epsilon^\prime}{B}$ and $\epsilon=(\epsilon^\prime)^2/(4L_{\widetilde \Phi}+ \frac{8\epsilon^\prime}{B})$, where $B$ is an upper bound of $\|x^*\|$, i.e., $\|x^*\|\leq B$.
 To achieve $\|\nabla \Phi (z_k)\|\leq 5\epsilon$, the required complexity is at most 
  \begin{align*}
\mathcal{C}_{\text{\normalfont grad}}(\mathcal{A},\epsilon) \leq \mathcal{O}\Big( \Big( \sqrt{\frac{B\widetilde L_y}{\epsilon\mu_y^3}}+\Big(\sqrt{\frac{B\rho_{yy}\widetilde L_y}{\epsilon\mu_y^{4}}} +  \sqrt{\frac{B\rho_{xy}\widetilde L_y}{\epsilon\mu_y^3}}\Big)\sqrt{\Delta^*_{\text{\normalfont\tiny CSC}}}\Big)\log\, {\small \text{\normalfont poly}(\epsilon,\mu_x,\mu_y,\Delta^*_{\text{\normalfont\tiny CSC}})}\Big).
\end{align*}
\end{list}
\end{theorem}  
As far as we know, \Cref{th:upper_csc1sc} is the first convergence result for convex-strongly-convex bilevel optimization without the bounded gradient assumption. Then, similarly to \Cref{coro:quadaticSr}, we also study the quadratic $g(x,y)$ case where the inner-level function$g(x,y)$ takes the quadratic form as given in \cref{quadratic_case}.  
\begin{corollary}[\bf Quadratic $g$ subclass]\label{coro:quadaticConv}
Under the same setting of \Cref{th:upper_csc1sc}, consider the quadratic $g(x,y)$ 
where  $\nabla_x\nabla_y g(\cdot,\cdot)$ and $\nabla_y^2 g(\cdot,\cdot)$ are constant. Then, we have
\begin{list}{$\bullet$}{\topsep=0.1in \leftmargin=0.17in \rightmargin=0.in \itemsep =0.01in}
\item {\bf (Suboptimality  gap)} To achieve $\Phi(z_K) - \Phi(x^*)\leq \epsilon$, we have $\mathcal{C}_{\text{\normalfont fun}}(\mathcal{A},\epsilon) \leq  \mathcal{\widetilde O}\Big(B \sqrt{\frac{\widetilde L_y}{\epsilon\mu_y^3}}\Big)$.
%\begin{align*}
%\mathcal{C}(\mathcal{A},\epsilon) \leq  \mathcal{O}\Big( \sqrt{\frac{B\widetilde L_y}{\epsilon^\prime\mu_y^3}}\log\, {\text{\normalfont poly}(\epsilon^\prime,\mu_x,\mu_y,\Delta_*)}\Big).
%\end{align*}
\item {\bf (Gradient norm)} To achieve $\|\nabla \Phi (z_k)\|\leq \epsilon$, we have $\mathcal{C}_{\text{\normalfont grad}}(\mathcal{A},\epsilon) \leq \mathcal{\widetilde O}\Big(\sqrt{\frac{B\widetilde L_y}{\epsilon\mu_y^3}}\Big).$
%\begin{align*}
%\mathcal{C}(\mathcal{A},\epsilon) \leq \mathcal{O}\Big(\sqrt{\frac{B\widetilde L_y}{\epsilon^\prime\mu_y^3}}\log\, {\small \text{\normalfont poly}(\epsilon^\prime,\mu_x,\mu_y,\Delta_*)}\Big).
%\end{align*}
\end{list}
\end{corollary}
It can be seen from \Cref{coro:quadaticConv} that for the quadratic $g$ subclass, AccBiO achieves a computational complexity of $ \mathcal{\widetilde O}\Big(\sqrt{\frac{B\widetilde L_y}{\epsilon\mu_y^3}}\Big)$ in term of the gradient norm. For the case where $\widetilde L_y\leq\mathcal{O}(\mu_y)$, the complexity becomes $ \mathcal{\widetilde O}\Big(\sqrt{\frac{B}{\epsilon\mu_y^2}}\Big)$, which matches the lower bound in \Cref{co:co1} up to logarithmic factors.

\subsection{Optimality of Bilevel Optimization and Discussion}
We compare the lower and upper bounds and make the following remarks on the optimality of bilevel optimization and its comparison to minimax optimization. 

\vspace{0.3cm}
\noindent {\bf Optimality of results for quadratic $g$ subclass.} We compare the developed lower and upper bounds and make a few remarks on the optimality of the proposed AccBiO algorithms. Let us first focus on the quadratic $g$ subclass where $g(x,y)$ takes the quadratic form as in \cref{quadratic_case}. For the strongly-convex-strongly-convex setting, comparison of \Cref{thm:low1} and \Cref{coro:quadaticSr} implies that AccBiO achieves the optimal complexity for $\widetilde L_y\leq \mathcal{O}(\mu_y)$, i.e., the inner-level problem is easy to solve.  For the general case, there is still a gap of $\frac{1}{\sqrt{\mu_y}}$ between lower and upper bounds.  For the convex-strongly-convex setting, comparison of \Cref{main:convex} and \Cref{coro:quadaticConv} shows that AccBiO is optimal for $\widetilde L_y\leq \mathcal{O}(\mu_y)$, and there is a gap for the general case. Such a gap is mainly due to the large smoothness parameter $L_\Phi$ of $\Phi(\cdot)$. We note that a similar issue also occurs for minimax optimization, which has been addressed by \cite{lin2020near} using an accelerated proximal point method for the inner-level problem and exploiting Sion's minimax theorem $\min_x\max_y f(x,y) =\max_y\min_x f(x,y)$. However, such an approach is not applicable for bilevel optimization due to the asymmetry of $x$ and $y$, e.g., $\min_xf(x,y^*(x)) \neq \min_y g(x^*(y),y)$. This gap between lower and upper bounds deserves future efforts.

\vspace{0.3cm}
\noindent {\bf Optimality of results for general $g$.} We now discuss the optimality of our results for a more general $g$ whose second-order derivatives are Lipschitz continuous. For the strongly-convex-strongly-convex setting, it can be seen from the comparison of \Cref{thm:low1} and \Cref{upper_srsr_withnoB} that  there is a gap between the lower and upper bounds. This gap is because the lower bounds construct the bilinearly coupled worst-case $g(x,y)$ whose Hessians and Jacobians are constant, rather than generally $\rho_{yy}$- and $\rho_{xy}$-Lipschitz continuous as considered in the upper bounds. Hence, tighter lower bounds need to be provided for this setting, which requires more sophisticated  worst-case instances with Lipschitz continuous Hessians $\nabla_y^2 g(x,y)$ and Jacobians $\nabla_x\nabla_yg(x,y)$.  For example, it is possible to construct $g(x,y)$ as $g(x,y)=\sigma(y)y^T Zy - x^TZy +b^Ty$, where $\sigma(\cdot):\mathbb{R}^d\rightarrow \mathbb{R}$ satisfies a certain Lipschitz property. For example, if $\sigma$ is Lipchitz continuous, simple calculation shows that $L_\Phi$ scales at an order of $\kappa_y^3$. However, it still requires significant efforts to determine the form of $\sigma$ such that the optimal point of $\Phi(\cdot)$  and the subspaces $\mathcal{H}_x,\mathcal{H}_y$ are easy to characterize and satisfy the properties outlined in the proof of \Cref{upper_srsr_withnoB}.  

\vspace{0.3cm}
\noindent {\bf Comparison to minimax optimization.} We compare the optimality between minimax optimization and bilevel optimization. For the strongly-convex-strongly-convex minimax optimization, \cite{zhang2019lower} developed a lower bound of  {\small$\widetilde{\Omega}(\frac{1}{\sqrt{\mu_x\mu_y}})$} for minimax optimization, which is achieved by the accelerated proximal point method proposed by \cite{lin2020near} up to logarithmic factors. 
For the same type of bilevel optimization,
we provide a lower bound of {\small$\widetilde{\Omega}\big(\sqrt{\frac{1}{\mu_x\mu_y^2}}\big)$} in~\Cref{thm:low1}, which is larger than that of minimax optimization by a factor of $\frac{1}{\sqrt{\mu_y}}$.  Similarly for the convex-strongly-convex bilevel optimization, we provide a lower bound of {\small $\widetilde \Omega\big(\frac{1}{\sqrt{\epsilon}}\min\{\frac{1}{\mu_y},\frac{1}{\epsilon^{1.5}}\}\big)$}, which is larger than the optimal complexity of {\small$\widetilde{\Omega}(\frac{1}{\sqrt{\epsilon\mu_y}})$} for the same type of minimax optimization~\citep{lin2020near} in a large regime of $\mu_y\geq\Omega(\epsilon^3)$. 
This establishes that bilevel optimization is fundamentally more challenging than minimax optimization. This is because bilevel optimization needs to handle the different structures of the outer- and inner-level functions $f$ and $g$ (e.g., second-order derivatives in the hypergradient), whereas for minimax optimization, the fact of $f=g$ simplifies the problem (e.g., no second-order derivatives) and allows more efficient algorithm designs.

%The larger lower bounds for bilevel optimization are due to the c
%by a factor of $\min\big\{\frac{1}{\sqrt{\mu_y}},\sqrt{ \frac{\mu_y}{\epsilon^3}} \big\}$. 

%is larger than the optimal complexity of {\small$\widetilde{\Omega}(\frac{1}{\sqrt{\mu_x\mu_y}})$} for the same type of minimax optimization by a factor of $\frac{1}{\sqrt{\mu_y}}$. Similar observation holds for the convex-strongly-convex case. This establishes that bilevel optimization is fundamentally more challenging than minimax optimization.   

\section{Upper Bounds with Gradient Boundedness Assumption} \label{sec:upper_gBscss}
%We focus on developing order-level tighter upper bounds  for bilevel optimization problem defined in \Cref{de:pc} by designing faster hypergradient-based bilevel algorithms as defined in~\Cref{alg_class}. % Acknowledgments---Will not appear in anonymized version

Our study in \Cref{upper_withoutB} does not make the bounded gradient assumption, which has been commonly taken in the existing studies~\citep{ghadimi2018approximation,ji2021bilevel,hong2020two,ji2020convergence}. 
In this section, we establish tighter upper bounds than those in existing works~\citep{ghadimi2018approximation,ji2021bilevel} under such an additional assumption.

\begin{assum}[{\bf Bounded gradient}]\label{bounded_f_assump}
There exists a constant $U$ such that for any $(x^\prime,y^\prime)\in\mathbb{R}^p\times \mathbb{R}^q$, $\|\nabla_y f(x^\prime,y^\prime)\|\leq U$.
\end{assum}
%The bounded gradient assumption in Assumption~\ref{bounded_f_assump} is to guarantee the {\bf boundedness} of the hypergradient estimation error, as shown in~\cite{ghadimi2018approximation,ji2021bilevel}.  

\subsection{Accelerated Bilevel Optimization Algorithm: AccBiO-BG}
We propose an accelerated algorithm named AccBiO-BG in \Cref{alg:bio} for bilevel optimization under the additional bounded gradient assumption. Similarly to  AccBiO, AccBiO-BG first runs $N$ steps of accelerated gradient descent (AGD) at each outer iteration. Note that AccBiO-BG here adopts a warm start strategy with $y_k^0=y_{k-1}^N$ so that our analysis does not require the boundedness of $y^*(x_k),k=0...,K$ and reduces the total computational complexity. %, which is needed in \cite{ghadimi2018approximation}. 
Then, 
AccBiO-BG constructs the hypergradient estimate $G_k:= \nabla_x f(\widetilde x_k,y_k^N) -\nabla_x \nabla_y g(\widetilde x_k,y_k^N)v_k^M$ following the same steps as in AccBiO. 
Finally, we update variables $x_{k},z_{k}$ via two  accelerated gradient steps, where we incorporate a variant~\citep{ghadimi2016accelerated} of Nesterov's momentum. We use this variant instead of vanilla Nesterov's momentum~\citep{nesterov2018lectures} in \Cref{alg:bioNoBG}, because the resulting analysis is easier to handle the warm start strategy, which backpropagates the tracking error $\|y^N_k-y^*(x_k)\|$ to previous loops. 

%proof for this variant is straightforward without induction steps, and hence is easier to handle the warm start strategy, i.e., backpropagate the tracking error $\|y^N_k-y^*(x_k)\|$ to previous loops. 

%by choosing different parameters $\alpha_k,\eta_k,\beta_k,\tau_k$. 

\begin{algorithm}[t]
	\caption{Accelerated Bilevel Optimization Method under Bounded Gradient Assumption (AccBiO-BG) } 
	\small
	\label{alg:bio}
	\begin{algorithmic}[1]
		\STATE {\bfseries Input:}  Initialization $ z_0=x_0=y_0=0$, parameters  $\eta_k,\tau_k.\alpha_k,\beta_k,\lambda$ and $\theta$ 
		\FOR{$k=0,...,K$}
		\STATE{Set $\widetilde x_k = \eta_k x_k + (1-\eta_k)z_k$}
		\STATE{Set $y_k^0 = y_{k-1}^{N} \mbox{ if }\; k> 0$ and $y_0$ otherwise (warm start)}
		\FOR{$t=1,....,N$}
		\STATE{
%		Run $N$ steps of AGD to $g(\widetilde x_k,\cdot)$ with condition number $\kappa_y$:
		
		\vspace{-0.45cm} 
		\begin{align*}
		\mbox{(AGD:)} \quad y_k^{t} &= s_k^{t-1} - \frac{1}{\widetilde L_y} \nabla_y g(\widetilde x_k,s_k^{t-1}), \quad s_k^{t} = \frac{2\sqrt{\kappa_y}}{\sqrt{\kappa_y}+1} y_k^{t} - \frac{\sqrt{\kappa_y}-1}{\sqrt{\kappa_y}+1} y_k^{t-1}.
		\end{align*}
		\vspace{-0.7cm} }
		\ENDFOR
                 \STATE{ \textit {Hypergradient computation}: \\
%\begin{list}{$\bullet$}{\topsep=0.3ex \leftmargin=0.3in \rightmargin=0.in \itemsep =0.01in}
               \hspace{0.4cm}1) Get $v_k^M$ after running $M$ steps of heavy-ball method $v_k^{t+1} = v_k^t-\lambda\nabla Q(v_k^t) +\theta(v_k^t-v_k^{t-1})$ \\ \hspace{0.7cm} with initialization $v_k^{0}=v_k^{1}=0$ over \vspace{-0.3cm} 
               \begin{align*}
             \mbox{(Quadratic programming:)}\;\;  \min_v Q(v):=\frac{1}{2}v^T\nabla_y^2 g(\widetilde x_k,y_k^N) v - v^T
\nabla_y f(\widetilde x_k,y^N_k);
\end{align*}\vspace{-0.4cm}\\
    \hspace{0.4cm}2) Compute Jacobian-vector product $\nabla_x \nabla_y g(\widetilde x_k,y_k^N)v_k^M $ via automatic differentiation; \\
    \vspace{0.1cm}
    \hspace{0.4cm}3) compute {\bf hypergradient estimate} $G_k:= \nabla_x f(\widetilde x_k,y_k^N) -\nabla_x \nabla_y g(\widetilde x_k,y_k^N)v_k^M.$ 
              }
               \vspace{0.1cm}
               \STATE{Update $x_{k+1}=\tau_k \widetilde x_{k}+(1-\tau_k)x_k-\beta_k G_k$}
               \vspace{0.1cm}
               \STATE{Update $z_{k+1}=\widetilde x_{k}-\alpha_k G_k$}\vspace{0.1cm}
                                 % \STATE{Update $k \leftarrow k+1$}
                              
		\ENDFOR
	\end{algorithmic}
	\end{algorithm}

\vspace{-0.1cm}	
\subsection{Strongly-Convex-Strongly-Convex Bilevel Optimization}	
%In this setting, the goal is to find an $\epsilon$-accurate solution $\widehat x\in\mathbb{R}^p$ such that $\Phi(\widehat x) - \Phi(x^*) \leq \epsilon$, where $x^*=\argmin_{x\in\mathbb{R}^p}\Phi(x)$.  
The following theorem provides a theoretical performance guarantee for AccBiO-BG.
\begin{theorem}\label{upper_srsr} 
%Suppose that $f,g,\mathcal{X}$ satisfies \Cref{de:pc} and $\Phi(\cdot)$ is $\mu_x$-strongly-convex with  $x^*\in\mathcal{X}$ and $g(x,\cdot)$ is $\mu_y$-strongly-convex. 
Suppose that $(f,g)$ belong to the strongly-convex-strongly-convex class $\mathcal{F}_{scsc}$ in \Cref{de:pc} and further suppose Assumption~\ref{bounded_f_assump} is satisfied. 
Choose $\alpha_k=\alpha\leq \frac{1}{2L_\Phi}$, $\eta_k=\frac{\sqrt{\alpha\mu_x}}{\sqrt{\alpha\mu_x}+2}$, $\tau_k=\frac{\sqrt{\alpha\mu_x}}{2}$ and $\beta_k=\sqrt{\frac{\alpha}{\mu_x}}$, where $L_\Phi$ is the smoothness parameter of $\Phi(x)$. Choose stepsizes {\small $\lambda=\frac{4}{(\sqrt{\widetilde L_y}+\sqrt{\mu_y})^2}$} and {\small $\theta=\max\big\{\big(1-\sqrt{\lambda\mu_y}\big)^2,\big(1-\sqrt{\lambda\widetilde L_y}\big)^2\big\}$} for the heavy-ball method.  
Then, to achieve $\Phi(z_K) - \Phi(x^*) \leq \epsilon$, the required complexity $\mathcal{C}_{\text{\normalfont fun}}(\mathcal{A},\epsilon)$ is at most 
%the output $\widetilde x$ of AccBiO satisfies 
\begin{align*}
\mathcal{C}_{\text{\normalfont fun}}(\mathcal{A},\epsilon)\leq  \mathcal{O}\Big(\sqrt{\frac{\widetilde L_y}{\mu_x\mu_y^4}}\log\frac{\mbox{\normalfont\small poly}(\mu_x,\mu_y,U,\Phi(x_{0})-\Phi(x^*))}{\epsilon}  \log \frac{\mbox{\normalfont \small poly}(\mu_x,\mu_y,U)}{\epsilon} \Big).
\end{align*}
%where $x_0$ is the initialization point and $x^*\in\mathcal{X}$ is the unique minimizer of $\Phi(x)$.  
\end{theorem}	
The proof of \Cref{upper_srsr} is provided in \Cref{proof:upss_wb}.  \Cref{upper_srsr} shows that  the upper bound achieved by our proposed AccBiO-BG algorithm is {\small$\mathcal{\widetilde O}(\sqrt{\frac{1}{\mu_x\mu_y^4}})$}. This bound improves the best known $\mathcal{\widetilde O}\big(\max\big\{\frac{1}{\mu_x\mu_y^3},\frac{\widetilde L^2_y}{\mu_y^2}\big\}\big)$ (see eq. (2.60) therein)  achieved by the accelerated bilevel approximation algorithm (ABA) in  \cite{ghadimi2018approximation} by a factor of {\small$\mathcal{O}(\mu_x^{-1/2}\mu_y^{-1})$}. 

%There exists a gap of $\mathcal{O}(\mu_y^{-1/2})$ between our upper bound and  the lower bound {\small $\widetilde \Omega(\sqrt{\frac{1}{\mu_x\mu_y^2}})$} established by \Cref{thm:low1}, which requires future efforts to address.

%For the case where $\widetilde L_{y} \leq \mathcal{O}(\mu_y)$, i.e., the inner level problem $g(x,\cdot)$ is easy to solve, the upper bound by \Cref{upper_srsr} becomes $\mathcal{\widetilde O}(\sqrt{\frac{1}{\mu_x\mu_y^2}})$, which matches the lower bound up to logarithmic factors, and hence implies the optimality  of AccBiO in this case. %This implies that the proposed AccBiO is nearly optimal 

\subsection{Convex-Strongly-Convex Bilevel Optimization}	
Similarly to \Cref{th:upper_csc1sc}, we consider a strongly-convex-strongly-convex function $\widetilde \Phi(\cdot)=\widetilde f(x,y^*(x))$  with $\widetilde f(x,y) = f(x,y) +\frac{\epsilon}{2B^2} \|x\|^2$, where $B=\|x^*\|$ as defined in~\Cref{de:pc}. %\end{align}
%where $D$ is the upper bound of $\|x^*\|$. 	
%By \Cref{de:pc}, we have $x_\epsilon^*\in\mathcal{X}$, where $x_\epsilon^*=\argmin_x \Phi_\epsilon(x)$. 
Then, we have the following theorem. 
%we can apply the results in \Cref{upper_srsr} to $\widetilde\Phi(x)$, and obtain the following theorem. 
% Let $x^*$  be a minimizer of $\Phi(\cdot)$, i.e., $x^*\in\argmin_{x\in\mathbb{R}^p}\Phi(x)$.   
\begin{theorem}\label{convex_upper_BG}
Suppose that $(f,g)$ belong to the convex-strongly-convex class $\mathcal{F}_{csc}$ in \Cref{de:pc} and further suppose Assumption~\ref{bounded_f_assump} is satisfied. Let $L_{\widetilde\Phi}$ be the smoothness parameter of $\widetilde\Phi(\cdot)$, which takes the same form as $L_\Phi$ in \Cref{upper_srsr} but with $L_x$ being  replaced by  $L_x+ \frac{\epsilon}{B^2}$. 
Choose the same parameter as in \Cref{upper_srsr} with $\alpha=\frac{1}{2L_{\widetilde \Phi}}$ and $\mu_x=\frac{\epsilon}{B^2}$. Then, to achieve $\Phi(z_K) - \Phi(x^*) \leq \epsilon$, the required complexity $\mathcal{C}_{\text{\normalfont fun}}(\mathcal{A},\epsilon)$ is at most 
\begin{align}
\mathcal{C}_{\text{\normalfont fun}}(\mathcal{A},\epsilon)\leq \mathcal{O}\Big(B\sqrt{\frac{\widetilde L_y}{\epsilon\mu_y^4}}\log\frac{\mbox{\small \normalfont poly}(\epsilon,\mu_y,B,U,\Phi(x_{0})-\Phi(x^*))}{\epsilon}  \log \frac{\mbox{\small \normalfont poly}(B,\epsilon,\mu_y,U)}{\epsilon} \Big).
\end{align}
\end{theorem}
As shown in \Cref{convex_upper_BG}, our proposed AccBiO-BG algorithm achieves a complexity  of $\mathcal{\widetilde O}\big(\frac{1}{\epsilon^{0.5}\mu_y^2}\big)$, which improves the best known result $\mathcal{O}\big(\frac{1}{\epsilon^{0.75}\mu_y^{6.75}}\big)$ achieved by the ABA algorithm in \cite{ghadimi2018approximation} (see eq. (2.61) therein) by an order of $\mathcal{\widetilde O}\big(\frac{1}{\epsilon^{0.25}\mu_y^{4.75}}\big)$.

%\vspace{-0.2cm}
\section{Conclusion and Discussion} 
In this paper, we provide the first-known lower bounds and new upper bounds with relaxed assumptions and tighter characterizations for bilevel optimization under various function geometries. We here discuss the extensions and applications of our results as follows.  

\vspace{0.3cm}
\noindent{\bf Other loss geometries. }In this paper, we study two typical loss geometries, i.e., the strongly-convex-strongly-convex and convex-strongly-convex geometries. It will be interesting to investigate other types of loss landscapes. For example, when the total objective function $\Phi(x)$ involves neural networks and is generally nonconvex, new efforts are needed to address the boundedness of iterates $x_k$ as the algorithm runs, e.g., by adding a projection onto a bounded domain or a regularizer to force such a boundedness. Moreover, existing convergence rate analysis relies on the strong convexity of the inner problem to better capture the inner-level convergence behavior. It is interesting to extend to more general geometries that allows more than one unique solution, e.g., convexity or star-convexity, which, however, requires us to revise the hypergradient form in \cref{hyperG} or explore the convergence under other criterions  such as  stationarity based on the Moreau envelope~\citep{davis2019stochastic} due to the nonsmoothness of the inner-level solution $y^*(x)$ and the objective function $\Phi(x)$.

\vspace{0.3cm}
\noindent{\bf Applications of results. }We note that some of our analysis can be applied to other problem domains such as minimax optimization. For example, our lower-bounding technique for \Cref{main:convex} can be extended to {\bf convex-concave} or {\bf convex-strongly-concave minimax} optimization, where the objective function $f(x,y)$ satisfies the general smoothness property as in \cref{def:first} with the general smoothness parameters $L_x,L_{xy}, L_y\geq 0$. The resulting lower bound will be different from that in \cite{ouyang2019lower}, which considered a special case with $L_y=0$ and the convergence is measured in terms of the suboptimality gap $\mathcal{O}(\Phi(x)-\Phi(x^*))$ rather than the gradient norm $\|\nabla\Phi(x)\|$ considered in this paper. Thus, such an extension will serve as a new contribution to lower complexity bounds for minimax optimization.

\newpage
\appendix
\doparttoc % Tell to minitoc to generate a toc for the parts
\faketableofcontents % Run a fake tableofcontents command for the partocs

%################################################
%used to make content
\addcontentsline{toc}{section}{Appendix} % Add the appendix text to the document TOC
\part{Appendix} % Start the appendix part
\parttoc % Insert the appendix TOC
%################################################
\allowdisplaybreaks
 
\section{AID-Based Bilevel Algorithms}\label{example:appen}

In this section, we present existing AID-based bilevel optimization algorithms, and show that they belong to the hypergradient-based algorithm class we consider in  \Cref{alg_class}.

\begin{example}[AID-based Bilevel Algorithms]\label{exam:aids}\citep{domke2012generic,pedregosa2016hyperparameter,grazzi2020iteration,ji2021bilevel} Such a class of algorithms use AID-based approaches for hypergradient computation, and take the following updates.

\vspace{0.2cm}
\noindent For each outer iteration $m=0,....,Q-1$,
\begin{list}{$\bullet$}{\topsep=0.3ex \leftmargin=0.08in \rightmargin=0.in \itemsep =0.01in}
\item  Update variable $y$ using gradient decent (GD) or accelerated gradient descent (AGD)
\begin{align} 
(\mbox{\normalfont GD:})\quad y_{m}^t &=  y_m^{t-1} - \eta \nabla_y g(x_m, y_m^{t-1}), t=1,...,N \nonumber
\\ (\mbox{\normalfont AGD:}) \quad y_m^{t} &= z_m^{t-1} - \eta \nabla_y g(x_m,z_m^{t-1}),  \nonumber
\\z_m^{t} &= \Big(1+\frac{\sqrt{\kappa_y}-1}{\sqrt{\kappa_y}+1}\Big) y_m^{t} - \frac{\sqrt{\kappa_y}-1}{\sqrt{\kappa_y}+1} y_m^{t-1}, t= 1,...,N 
\end{align}
where $\kappa_y= \widetilde L_y/\mu_y$ denotes the condition number of the inner-level function $g(x,\cdot)$. 
\item Update $x$ via $x_{m+1} = x_{m}-\beta G_m$, where $G_m$ is constructed via AID and takes the form of 
\begin{align}\label{aid_hgd}
G_m = \nabla_x f(x_m,y_m^N)- \nabla_x \nabla_y g(x_m,y_m^N)v_m^S, %_{\mbox{\normalfont Jacobian-vector product}},
\end{align}
where vector $v_m^S$ is obtained by running $S$ steps of GD (with initialization $v_m^0=0$) or accelerated gradient methods (e.g., heavy-ball method with $v_m^0 = v_m^1=0$) to solve a quadratic programming
\begin{align}\label{exam:quadratic}
\min_{v} Q(v):= \frac{1}{2}v^T\nabla_y^2 g(x_m,y_m^N)v-v^T\nabla_y f(x_m,y_m^N).
\end{align}
\end{list}
\end{example}
We next verify that \Cref{exam:aids} belongs  to the algorithm class defined in \Cref{alg_class}.  
%To see that the hypergradient estimator $G_m$ in~\cref{aid_hgd} falls into the linear span subspaces in~\cref{x_span}. 
For the case when $S$-steps GD with initialization $\bf{0}$ is applied to solve the quadratic program in~\cref{exam:quadratic}, simple telescoping yields 
\begin{align*}
%v_m^{k+1} = v_m^{k} - \alpha \nabla_y^2g(x_m,y_m^N) v_m^{k} +\alpha \nabla_yf(x_m,y_m^N)
v_m^{S} = \alpha \sum_{t=0}^{S-1}(I-\alpha\nabla_y^2g(x_m,y_m^N))^{t} \nabla_yf(x_m,y_m^N),
\end{align*}
which, incorporated into \cref{aid_hgd}, implies that $G_m$ falls into the span subspaces in~\cref{x_span}, and hence all updates fall into the subspaces $\mathcal{H}_x^k, \mathcal{H}_y^k, k=0,...,K$ defined in \Cref{alg_class}. For the case when heavy-ball method, i.e., $v_m^{t+1} = v_m^t-\eta_t \nabla Q(v_m^t) +\theta_t(v_m^t-v_m^{t-1})$, with initialization $v_m^0=v_m^1=\bf{0}$ is applied to~\cref{exam:quadratic}, expressing the updates via a dynamic system perspective yields 
\begin{align}\label{vmt_form} 
\begin{bmatrix}
v_m^{S}\\v_m^{S-1} 
\end{bmatrix}
= 
\sum_{s=2}^{S}\prod_{t=s}^{S-1}
\begin{bmatrix}
(1+\theta_t) I- \eta_t \nabla_y^2 g(x_m,y_m^N) & -\theta_t I\\I & \bf{0}
\end{bmatrix}
 \begin{bmatrix}
\eta_t \nabla_y f(x_m,y_m^N)\\\bf{0} 
\end{bmatrix}.
\end{align}
Combining $v_m^S$ in \cref{vmt_form} with~\cref{aid_hgd}, we can see that the resulting $G_m$ falls into the span subspaces in~\cref{x_span}, and hence this case still belongs to the algorithm class in  \Cref{alg_class}. 

Note that the algorithm class considered in \Cref{alg_class} also includes single-loop bilevel optimization algorithms, e.g., by setting $N=1$ in \Cref{exam:itd} and \Cref{exam:aids}.

\section{Proof of \Cref{thm:low1}}\label{appen:thm1}
In this section, we provide a complete proof of \Cref{thm:low1} under the strongly-convex-strongly-convex geometry. Note that our construction sets the dimensions of variables $x$ and $y$ to be the same, i.e.,  $p=q=d$. The main proofs are divided into four steps: 1) constructing the worst-case instance that belongs to the problem class $\mathcal{F}_{scsc}$ defined in \Cref{de:pc}; 2) characterizing the optimal point $x^*=\argmin_{x\in\mathbb{R}^d}\Phi(x)$; 3) characterizing the subspaces $\mathcal{H}_x^k,\mathcal{H}_y^k$; and 4) developing lower bounds on the convergence and complexity.  

\vspace{0.2cm}
\noindent {\bf Step 1: Constructing the worst-case instance that satisfies \Cref{de:pc}.}
\vspace{0.2cm}

In this step, we show that the constructed $f,g$ in~\cref{str_fg} satisfy Assumptions~\ref{fg:smooth} and~\ref{g:hessiansJaco}, and $\Phi(x)$ is $\mu_x$-strongly-convex. It can be seen from~\cref{str_fg} that $f,g$ satisfy  \cref{def:first} \eqref{df:sec} and \eqref{def:three} in Assumptions~\ref{fg:smooth} and~\ref{g:hessiansJaco} with arbitrary constants $L_x,L_y,\widetilde L_y,\widetilde L_{xy}$ and $\rho_{xy}=\rho_{yy}=0$ but requires $L_{xy}\geq \frac{(L_x-\mu_x)(\widetilde L_y-\mu_y)}{2\widetilde L_{xy}}$ (which is still at a constant level) due to the introduction of the term $\frac{\alpha\beta}{\widetilde L_{xy}}x^TZ^3y$ in $f$. We note that such a term introduces necessary connection between $f$ and $g$, and yields a tighter lower bound, as pointed out in the remark at the end of \Cref{caoananasca:s}. 

We next show that the overall objective function $\Phi(x)=f(x,y^*(x)) $ is $\mu_x$-strongly-convex. 
According to~\cref{str_fg}, we have $g(x,\cdot)$ to be $\mu_y$-strongly-convex with a single minimizer 
%\begin{align}
$y^*(x) = (\beta Z^2+\mu_y I)^{-1} \big(\frac{\widetilde L_{xy}}{2}Zx-b\big)$,
%\end{align}
and hence we obtain from \cref{obj} that $\Phi(x)$ is given by 
\begin{align}\label{phi_x}
\Phi(x) = &\frac{1}{2} x^T(\alpha Z^2 +\mu_x I) x -\frac{\alpha\beta}{\widetilde L_{xy}}x^TZ^3(\beta Z^2+\mu_y I)^{-1} \Big(\frac{\widetilde L_{xy}}{2}Zx-b\Big) \nonumber
\\ &+\frac{\bar L_{xy}}{2}x^TZ(\beta Z^2+\mu_y I)^{-1} \Big(\frac{\widetilde L_{xy}}{2}Zx-b\Big) + \frac{\bar L_{xy}}{\widetilde L_{xy}} b^T (\beta Z^2+\mu_y I)^{-1} \Big(\frac{\widetilde L_{xy}}{2}Zx-b\Big) \nonumber
\\&+ \frac{L_y}{2} \Big(\frac{\widetilde L_{xy}}{2}Zx-b\Big)^T (\beta Z^2+\mu_y I)^{-1}(\beta Z^2+\mu_y I)^{-1} \Big(\frac{\widetilde L_{xy}}{2}Zx-b\Big).  
\end{align}
Note that $Z$ is symmetric and invertible with $Z^{-1}$ given by 
$$Z^{-1}= \begin{bmatrix}
 1&   1&  1& 1\\
  \vdots & \text{\reflectbox{$\ddots
  $}} &  1&   \\
  1&\text{\reflectbox{$\ddots$}} & & \\
  1&   &  & \\ 
\end{bmatrix},$$
and hence the eigenvalue decomposition of $Z$ can be written as $Z= U \,\text{Diag}\{\lambda_1,...,\lambda_d\}U^T$, where $\lambda_i\neq 0, i=1,...,d$ and $U$ is an orthogonal matrix. 
%Note that $Z$ is invertible with $Z^{-1}$ given by 
%$$Z^{-1}= \begin{bmatrix}
% 1&   1&  1& 1\\
%  \vdots & \text{\reflectbox{$\ddots
%  $}} &  1&   \\
%  1&\text{\reflectbox{$\ddots$}} & & \\
%  1&   &  & \\ 
%\end{bmatrix}.$$
%Then, the singular value decomposition of $Z$ can be written as $Z= U \,\text{Diag}\{\sigma_1,...,\sigma_d\}U^T$, where $\sigma_i\neq 0, i=1,...,d$ and $U$ is an orthogonal matrix. 
Then, for any integers $i,j>0$, simple calculation yields
\begin{align}\label{iterchange}
Z^i (\beta Z^2+\mu_y I)^{-j} = U \text{Diag}\bigg\{\frac{\lambda^i_1}{(\beta \lambda_1^2+\mu_y)^j},...,\frac{\lambda^i_d}{(\beta \lambda_d^2+\mu_y)^j}\bigg\}U^T = (\beta Z^2+\mu_y I)^{-j} Z^i. 
\end{align}
Using the relationship in \cref{iterchange}, we have  %to \cref{phi_x} yields 
\begin{align*}
\frac{1}{2} x^T\alpha Z^2 x =& \frac{\alpha\beta}{2} x^T Z^4 (\beta Z^2+\mu_y I)^{-1} x + \frac{\alpha \mu_y}{2}x^TZ^2(\beta Z^2+\mu_y I)^{-1}x,
\end{align*}
which, in conjunction with  \cref{phi_x} and \cref{iterchange}, yields
\begin{align}
\Phi(x) =  \frac{1}{2}& \mu_x \|x\|^2 + \frac{2\alpha\mu_y+\bar L_{xy}\widetilde L_{xy}}{4} x^T Z^2 (\beta Z^2+\mu_y I)^{-1} x - \frac{\bar L_{xy}}{\widetilde L_{xy}} b^T(\beta Z^2+\mu_y I)^{-1} b
 \nonumber
\\+ \frac{L_y}{2} &\Big(\frac{\widetilde L_{xy}}{2}Zx-b\Big)^T (\beta Z^2+\mu_y I)^{-2}\Big(\frac{\widetilde L_{xy}}{2}Zx-b\Big) + \frac{2\alpha\beta}{\widetilde L_{xy}^2}b^TZ^2 (\beta Z^2+\mu_y I)^{-1} b,  
\end{align}
which is $\mu_x$-strongly-convex. 
%, and hence the constructed $f,g,\mathcal{X}$ and $\Phi$ fall into the strongly-convex-strongly-convex case considered in \Cref{de:pc}.
% Next, we show $\nabla_y f(x,y^*(x))$ is bounded for any $x\in\mathcal{X}$. Note that 
%\begin{align*}
%%\nabla_y f(x,y^*(x)=   
%\nabla_y f(x,y^*(x)) &= -\frac{\alpha \beta}{\widetilde L_{xy}} Z^3x + \frac{\bar L_{xy}}{2} Zx +L_y (\beta Z^2+\mu_y I)^{-1} \big(\frac{\widetilde L_{xy}}{2}Zx-b\big) +\frac{\bar L_{xy}}{\widetilde L_{xy}} b - \frac{2\alpha\beta}{\widetilde L_{xy}^2} Z^2 b,
%\end{align*}
%which, in conjunction with that $\|x\|\leq B,\forall x
%\in \mathcal{X}$, yields the boundedness of $\|\nabla_y f(x,y^*(x))\|,x\in\mathcal{X}$ with the bound {\small $U=\mbox{\normalfont poly}(B,L_x,L_{xy},L_y,\widetilde L_{xy},\widetilde L_{y},\mu_x,\mu_y)$}. Then, Assumption~\ref{boundness_x_g} is satisfied. 

\vspace{0.2cm}
\noindent {\bf Step 2: Characterizing $x^*=\argmin_{x\in\mathbb{R}^d}\Phi(\cdot).$}
\vspace{0.2cm}

%firstcompute the global solution $x^*$ of the strongly-convex function $\Phi(x)$. 
Based on the form of $\Phi(\cdot)$,  we have 
\begin{align}\label{eq:aboveg}
\nabla \Phi(x) =& (\beta Z^2 +\mu_y I)^2\mu_x x + \Big(\alpha \mu_y +\frac{\bar L_{xy}\widetilde L_{xy}}{2}\Big) (\beta Z^2 +\mu_y I) Z^2 x + \frac{L_y\widetilde L_{xy}}{2}\Big(\frac{\widetilde L_{xy}}{2}Z^2x-Zb\Big) \nonumber
\\=& \Big( \beta^2\mu_x +\alpha \beta\mu_y+\frac{ \beta \bar L_{xy}\widetilde L_{xy}}{2}\Big) Z^4x + (2\beta\mu_x\mu_y+\alpha\mu_y^2+\frac{\mu_y\bar L_{xy}\widetilde L_{xy}}{2} +\frac{L_y\widetilde L_{xy}^2}{4}) Z^2x  \nonumber
\\&+ \mu_x\mu_y^2 x  - \frac{L_y\widetilde L_{xy}}{2} Zb.
\end{align}
%Let $\bar x = \argmin_{x\in\mathbb{R}^d} \Phi(x)$ be the unconstrained optimal point of $\Phi(x)$. We choose the ball radius $B\geq \|\bar x\|$ such that $\bar x\in\mathcal{X}$ and hence $x^*=\bar x$. 
By setting $\nabla\Phi(x^*) = 0$, we have 
\begin{align}\label{z_equations}
Z^4x^* + &\underbrace{\frac{2\beta\mu_x\mu_y+\alpha\mu_y^2+\frac{\mu_y\bar L_{xy}\widetilde L_{xy}}{2} +\frac{L_y\widetilde L_{xy}^2}{4}}{\beta^2\mu_x +\alpha \beta\mu_y+\frac{ \beta \bar L_{xy}\widetilde L_{xy}}{2}}}_{\lambda} Z^2x^*  \nonumber
\\&\hspace{1cm}+ \underbrace{\frac{ \mu_x\mu_y^2}{\beta^2\mu_x +\alpha \beta\mu_y+\frac{ \beta \bar L_{xy}\widetilde L_{xy}}{2}}}_{\tau} x^* = \underbrace{\frac{L_y\widetilde L_{xy}Zb}{2(\beta^2\mu_x +\alpha \beta\mu_y+\frac{ \beta \bar L_{xy}\widetilde L_{xy}}{2})}}_{\widetilde b},
\end{align}
%\vspace{-0.5cm}
where we define $\lambda,\tau,\widetilde b$ for notational convenience. The following lemma establishes  useful properties of $x^*$ under a specific selection of $\widetilde b$. 
\begin{lemma}\label{le:x_star}
Let $b$ be chosen such that $\widetilde b$ as defined in \cref{z_equations} satisfies $\widetilde  b_1= (2+\lambda +\tau) r -(3+\lambda)r^2 + r^3, \widetilde b_2 = r-1$ and $\widetilde b_t=0, t=3,...,d$, where $0<r<1$ is a solution of equation 
\begin{align}\label{eq:soulc}
1 - (4+\lambda)r + (6+2\lambda +\tau) r^2 -(4+\lambda) r^3 +r^4=0. 
\end{align}
Let $\hat x$ be a vector with each coordinate $\hat x_i = r^i$. Then, we have
\begin{align}
\|\hat x - x^*\| \leq \frac{(7+\lambda)
}{\tau}r^d.
\end{align}
% and a vector $\hat b$ with $\hat b$.
\end{lemma}
\begin{proof}
Note that the choice of $b$ is achievable because $Z$ is invertible with $Z^{-1}$ given by 
\begin{align*}
Z^{-1}= \begin{bmatrix}
 1&   1&  1& 1\\
  \vdots & \text{\reflectbox{$\ddots
  $}} &  1&   \\
  1&\text{\reflectbox{$\ddots$}} & & \\
  1&   &  & \\ 
\end{bmatrix}.
\end{align*}
Then, define a vector $\hat b$ with $\hat b_t=\widetilde b_t$ for $t=1,...,d-2$ and 
\begin{align}\label{eq:helpss}
\hat b_{d-1} =& r^{d-3} - (4+\lambda) r^{d-2} +(6+2\lambda+\tau)r^{d-1} - (4+\lambda) r^d \overset{\eqref{eq:soulc}}= -r^{d+1}\nonumber
\\ \hat b_d =& r^{d-2} - (4+\lambda) r^{d-1} +(5+2\lambda +\tau) r^d \overset{\eqref{eq:soulc}}= -r^d + (4+\lambda) r^{d+1}-r^{d+2}.
\end{align}
Then, it can be verified that $\hat x$ satisfies the following equations %Based on the forms of  
%$Z^2$ and $Z^4$ in \cref{matrices_coupling}, \cref{z_equations} can be written into the following equations 
\begin{align*}
(2+\lambda +\tau)\hat x_1 - \hat (3+\lambda) x_2 +\hat x_3 &= \hat b_1 \nonumber
\\ -(3+\lambda) \hat x_1 +(6+2\lambda +\tau) \hat x_2 -(4+\lambda) \hat x_3 + \hat x_4 &= \hat b_2 \nonumber
\\ \hat x_t - (4+\lambda)\hat x_{t+1} + (6+2\lambda +\tau) \hat x_{t+2} -(4+\lambda) \hat x_{t+3}+\hat x_{t+4} &= \hat  b_{t+2}, \mbox{ for } 1\leq t\leq d-4  \nonumber
\\ \hat x_{d-3} - (4+\lambda)\hat x_{d-2} +(6+2\lambda +\tau) \hat x_{d-1} -(4+\lambda)\hat x_d &= \hat b_{d-1}\nonumber
\\\hat x_{d-2} -(4+\lambda) \hat x_{d-1} + (5+2\lambda+\tau) \hat x_d &= \hat b_d, 
\end{align*}
which, in conjunction with the forms of  $Z^2$ and $Z^4$ in \cref{matrices_coupling}, yields 
\begin{align*}
Z^4\hat x + \lambda Z^2 \hat x + \tau \hat x  = \hat b.
\end{align*}
Noting that $Z^4 x^* + \lambda Z^2 x^* + \tau x^*  = \widetilde b$, we have
\begin{align*}
\tau \|x^*-\hat x \|\leq \|(Z^4+\lambda Z^2 + \tau I)(x^*-\hat x)\| = \|\widetilde b -\hat b\| \overset{(i)}\leq (7+\lambda) r^d
\end{align*}
where $(i)$ follows from the definition of $\hat b $ in \cref{eq:helpss}. 
%Then, let $p$ be an solution of the equation $1 - (4+\lambda)p + (6+2\lambda +\tau) p^2 -(4+\lambda) p^3 +p^4=0$ and 
%define a vector $\hat x^*$ with each coordinate $\hat x_i^* = p^i$ and a vector $\hat b$ with $\hat b_1= (2+\lambda +\tau) p -(3+\lambda)p^2 + p^3, \hat b_2 = p-1$ and $\hat b_t=0, t=3,...,d$. Then, 
\end{proof}
%\vspace{0.2cm}
\noindent {\bf Step 3: Characterizing subspaces $\mathcal{H}_x^K$ and $\mathcal{H}_y^K$.}
\vspace{0.2cm}

In this step, we characterize the forms of  the subspaces $\mathcal{H}_x^K$ and $\mathcal{H}_y^K$
for bilevel optimization algorithms considered in \Cref{alg_class}. Based on the constructions of $f,g$ in~\cref{str_fg}, we have
\begin{align*}
\nabla_x f(x,y) &= (\alpha Z^2+\mu_x I)x -\frac{\alpha\beta}{\widetilde L_{xy}} Z^3y +\frac{\bar L_{xy}}{2}Zy\nonumber
\\ \nabla_y f(x,y) &= -\frac{\alpha \beta}{\widetilde L_{xy}} Z^3x + \frac{\bar L_{xy}}{2} Zx +L_y y +\frac{\bar L_{xy}}{\widetilde L_{xy}} b - \frac{2\alpha\beta}{\widetilde L_{xy}^2} Z^2 b
\\\nabla_x\nabla_y g(x,y) &= -\frac{\widetilde L_{xy}}{2} Z, \; \nabla_y^2g(x,y) = \beta Z^2 +\mu_y I, \; \nabla_y g(x,y) = (\beta Z^2 + \mu_y I) y - \frac{\widetilde L_{xy}}{2} Zx + b, 
\end{align*}
which, in conjunction with \cref{hxk} and \cref{x_span}, yields 
\begin{align}\label{hxy_steps}
\mathcal{H}_y^0 &= \mbox{Span}\{0\}, ...., \mathcal{H}_y^{s_0} = \mbox{Span}\{Z^{2(s_0-1)}b,...,Z^2b,b\} \nonumber
%\\\mathcal{H}_y^{s_0+1} &= \mbox{Span}\{Z^{2s_0}b,...,b,Z^{2(T+s_0+1)}b,...,b\}, ...., \mathcal{H}_y^{s_1} = \mbox{Span}\{Z^{2(s_1-1)}b,...,b,Z^{2(T+s_0+1)}b,...,b\} \nonumber
\\\mathcal{H}_x^0 &= .... \mathcal{H}_x^{s_0-1} =\mbox{Span}\{0\}, \mathcal{H}_x^{s_0} \subseteq\mbox{Span}\{Z^{2(T+s_0)}(Zb),....,Z^2(Zb),(Zb)\}. 
%\\\mathcal{H}_x^{s_0} &= .... \mathcal{H}_x^{s_1-1}  \subseteq\mbox{Span}\{Z^{2(T+s_0)}(Zb),....,Z(Zb),(Zb)\}, \mathcal{H}_x^{s_1}  =  \mbox{Span}\{Z^{2(2T+s_0+2)}(Zb),...,(Zb)\}
\end{align}
Repeating the same steps as in~\cref{hxy_steps}, it can be verified that 
\begin{align}\label{eq:intimedia}
H_{x}^{s_{Q-1}} \subseteq \mbox{Span}\{Z^{2(s_{Q-1} + QT+Q)}(Zb),...,Z^{2j}(Zb),...,Z^2(Zb),(Zb)\}. 
\end{align}
Recall \cref{x_span} that $\mathcal{H}_x^K = \mathcal{H}_x^{s_{Q-1}}$ and $s_{Q-1}\leq K$. Then, we obtain from \cref{eq:intimedia} that $H_{x}^{K }$ satisfies 
\begin{align}\label{eq:H_xK}
H_{x}^{K }\subseteq \mbox{Span}\{Z^{2(K+QT+Q)}(Zb),....,Z^2(Zb),(Zb) \}.
\end{align}
%\vspace{0.2cm}

\noindent {\bf Step 4: Characterizing convergence and complexity.}
\vspace{0.2cm}

Based on the results in Steps 1 and 2, we are now ready to provide a lower bound on the convergence rate and complexity of bilevel optimization algorithms.  Let $M = K+QT+Q+ 2$ and $x_0 = {\bf 0}$, and  
%and choose the ball radius $B = \frac{3}{2} \sqrt{\frac{p}{1-p}}$.
let the dimension $d$ satisfy 
\begin{align}\label{d_conditions}
d> \max\Big\{2M,M+1+\log_{r}\Big(\frac{\tau}{4(7+\lambda)}\Big)\Big\}.
\end{align}
Recall \Cref{le:x_star} that $Zb$ has zeros at all coordinates with $t=3,...,d$. Then, based on the form of subspaces $\mathcal{H}_x^K$ in \cref{eq:H_xK} and using the zero-chain property in \Cref{zero_chain}, we have $x^K$ has zeros at the coordinates with $t=M+1,...,d$, and hence 
\begin{align}\label{x_kssca}
\|x^K-\hat x\| \geq \sqrt{\sum_{i=M+1}^d\|\hat x_i\|}  = r^{M}\sqrt{r^2+...+r^{2(d-M)}} \overset{(i)}\geq \frac{r^M}{\sqrt{2}}\|\hat x -x_0\|,
\end{align}
where $(i)$ follows from \cref{d_conditions}. Then, based on \Cref{le:x_star} and~\cref{d_conditions}, we have 
\begin{align}\label{eq:hatxxs}
\|\hat x - x^*\| \leq \frac{7+\lambda}{\tau} < \frac{r^M}{2\sqrt{2}} r \overset{(i)}\leq \frac{r^M}{2\sqrt{2}} \|\hat x-x_0\|, 
\end{align}
where $(i)$ follows from the fact that $\|\hat x-x_0\|=\|\hat x\| \geq r$. Combining \cref{x_kssca} and \cref{eq:hatxxs} further yields 
\begin{align}\label{eq:supp1}
\|x^K-x^*\| \geq \|x^K-\hat x\| - \|\hat x - x^*\| \geq \frac{r^M}{\sqrt{2}} \|\hat x -x_0\| - \frac{r^M}{2\sqrt{2}}\|\hat x-x_0\| =\frac{r^M}{2\sqrt{2}}\|\hat x-x_0\|.
\end{align}
In addition,  note that 
\begin{align*}
\|x^* -\hat x\| \leq \frac{7+\lambda}{\tau} r^d  \overset{\eqref{d_conditions}}\leq \frac{1}{4} r \leq \frac{1}{4}\|\hat x\| \leq \frac{1}{4}\|\hat x - x^*\| +\frac{1}{4} \|x^*\|,
\end{align*}
which, in conjunction with $\|x_0-\hat x\|\geq \|x^*-x_0\| - \|x^*-\hat x\|$, yields
\begin{align}\label{eq:supp2}
\|x_0-\hat x\|\geq \frac{2}{3} \|x^*-x_0\|.
\end{align}
Combining \cref{eq:supp1} and \cref{eq:supp2} yields 
\begin{align}\label{x_lowbound}
\|x^K-x^*\| \geq \frac{\|x^*-x_0\|}{3\sqrt{2}} r^M.
\end{align}
%Note that \cref{eq:supp2} also implies that $\|x^*\|\leq \frac{3}{2}\|\hat x\|<B$, and hence $x^*\in\mathcal{X}$.
 Then, since the objective function $\Phi(x)$ is $\mu_x$-strongly-convex, we have $\Phi(x^K)-\Phi(x^*)\geq \frac{\mu_x}{2}\|x^K-x^*\|^2$ and $\|x_0-x^*\|^2\geq \Omega(\mu_y^2)(\Phi(x_0)-\Phi(x^*))$, and hence \cref{x_lowbound} yields
\begin{align}
\Phi(x^K)-\Phi(x^*) \geq \Omega\Big( \frac{\mu_x\mu_y^2(\Phi(x_0)-\Phi(x^*))}{36} r^{2M}\Big). 
\end{align}
Recall that $r$ is the solution of the equation $1 - (4+\lambda)r + (6+2\lambda +\tau) r^2 -(4+\lambda) r^3 +r^4=0$. Based on Lemma 4.2 in~\cite{zhang2019lower}, we have 
\begin{align}\label{eq:prange}
1-\frac{1}{\frac{1}{2}+\sqrt{\frac{\lambda}{2\tau}+\frac{1}{4}}} < r< 1, 
\end{align}
which, in conjunction with the definitions of $\lambda$ and $\tau$ in \cref{z_equations} and the fact $\bar L_{xy}\geq 0$, yields the first result \cref{result:first} in \Cref{thm:low1}. Then, in order to achieve an $\epsilon$-accurate solution, i.e., $\Phi(x^K)-\Phi(x^*) \leq \epsilon$, it requires
\begin{align}\label{eq: mbound}
M&=K+QT+Q + 2 \geq \frac{\log  \frac{\mu_x\mu_y^2(\Phi(x_0)-\Phi(x^*))}{\epsilon} }{2\log \frac{1}{r}} \nonumber
\\& \overset{(i)} \geq \Omega\Big(\sqrt{\frac{\lambda}{2\tau}}\log  \frac{\mu_x\mu_y^2(\Phi(x_0)-\Phi(x^*))}{\epsilon} \Big) \geq \Omega\bigg(\sqrt{\frac{L_y\widetilde L_{xy}^2}{\mu_x\mu_y^2}}\log  \frac{\mu_x\mu_y^2(\Phi(x_0)-\Phi(x^*))}{\epsilon} \bigg),
\end{align}
where $(i)$ follows from \cref{eq:prange}. 
Recall that the complexity measure is given by $\mathcal{C}_{\text{\normalfont fun}}(\mathcal{A},\epsilon) \geq \Omega(n_J+n_H + n_G)$, where the numbers $n_J,n_H$ of Jacobian- and Hessian-vector products  are given by $n_J=Q$ and $n_H=QT$ and the number $n_G$ of gradient evaluations is given by $n_G=K$. Then, the total complexity $\mathcal{C}_{\text{\normalfont fun}}(\mathcal{A},\epsilon)\geq \Omega(Q+QT+K)$, which combined with \cref{eq: mbound} implies 
\begin{align*}
\mathcal{C}_{\text{\normalfont fun}}(\mathcal{A},\epsilon) \geq \Omega\bigg(\sqrt{\frac{L_y\widetilde L_{xy}^2}{\mu_x\mu_y^2}}\log  \frac{\mu_x\mu_y^2(\Phi(x_0)-\Phi(x^*))}{\epsilon}  \bigg).
\end{align*}
Then, the proof is complete. 
%\begin{align}
%\frac{\lambda}{2\tau} = \frac{8\beta\mu_x\mu_y+4\alpha\mu_y^2+ 2\mu_y\bar L_{xy}\widetilde L_{xy} +L_y\widetilde L_{xy}^2}{8\mu_x\mu_y^2} = \frac{\beta}{\mu_y} +\frac{\alpha}{2\mu_x} + \frac{\bar L_{xy}\widetilde L_{xy} }{4\mu_x\mu_y}+\frac{L_y\widetilde L_{xy}^2}{8\mu_x\mu_y^2}
%\end{align}

%\end{proof}

\section{Proof of \Cref{main:convex}}\label{apep:mainconvexs}
In this section, we provide the proof for \Cref{main:convex} under the convex-strongly-convex geometry.  
 The proof is divided into the following steps: 
1) constructing the worst-case instance that belongs to the convex-strongly-convex problem class $\mathcal{F}_{csc}$ defined in \Cref{de:pc}; 2) characterizing $x^*\in \argmin_{x\in\mathbb{R}^d}\Phi(x)$; 3) developing the lower bound on the gradient norm $\|\nabla\Phi(x)\|$ when the last several coordinates of $x$ are zeros; 4) characterizing the subspaces $\mathcal{H}_x^k$ and $\mathcal{H}_x^k$; and 5) characterizing the convergence and complexity.
%\end{small}

\vspace{0.2cm}
\noindent {\bf Step 1: Constructing the worst-case instance that satisfies \Cref{de:pc}.}
\vspace{0.2cm}

It can be verified that  the constructed $f,g$ in \cref{con_fg} satisfy \cref{def:first} \eqref{df:sec} and \eqref{def:three}  in Assumptions~\ref{fg:smooth} and~\ref{g:hessiansJaco}.  Then, similarly to the proof of \Cref{thm:low1}, we have $y^*(x) = (\beta Z^2 +\mu_y I)^{-1} (\frac{\widetilde L_{xy}}{2} Zx-b)$ and hence $\Phi(x) = f(x,y^*(x))$ takes the form of 
\begin{align*}
\Phi(x) =  \frac{L_x}{8} x^T Z^2 x+ \frac{L_y}{2} \Big (\frac{\widetilde L_{xy}}{2} Zx-b \Big)^T (\beta Z^2 +\mu_y I)^{-2}\Big(\frac{\widetilde L_{xy}}{2} Zx-b\Big),
\end{align*}
which can be verified to be convex. 
%In addition, note that $\nabla_y f(x,y^*(x))$ takes the form of 
%\begin{align*}
%\nabla_y f(x,y^*(x)) = L_y(\beta Z^2 +\mu_y I)^{-1} \Big(\frac{\widetilde L_{xy}}{2} Zx-b\Big),
%\end{align*}
%which, combined with the boundedness of $\mathcal{X}$, implies that  $\nabla_y f(x,y^*(x))$ is bounded for any $x\in\mathcal{X}$ and the bound {\small $U=\mbox{\normalfont poly}(B,L_x,L_{xy},L_y,\widetilde L_{xy},\widetilde L_{y},\mu_x,\mu_y)$}. Then, Assumption~\ref{boundness_x_g} is satisfied. 

\vspace{0.2cm}
\noindent {\bf Step 2: Characterizing $x^*$.}
\vspace{0.2cm}

Note that 
%Then, letting $\nabla \Phi(x)$ 
the gradient $\nabla \Phi(x)$ is given by 
\begin{align}\label{conv_gdform}
\nabla \Phi(x) =  \frac{L_x}{4} Z^2 x + \frac{L_y\widetilde L_{xy}}{2} Z (\beta Z^2 +\mu_y I)^{-2}\Big(\frac{\widetilde L_{xy}}{2} Zx-b\Big). 
\end{align}
Then, setting $\nabla \Phi(x^*) = 0$ and using \cref{iterchange}, we have 
\begin{align}\label{x_star_convex}
\Big (\frac{L_x\beta^2}{4} Z^6 +  \frac{L_x\beta^2\beta\mu_y}{2} Z^4 +\Big (\frac{L_y\widetilde L_{xy}^2}{4}+\frac{L_x\mu_y^2}{4}\Big) Z^2 \Big) x^* =  \frac{L_y\widetilde L_{xy}}{2} Zb. 
\end{align}
%Let $D$ be a universal constant
Let $\widetilde b =\frac{L_y\widetilde L_{xy}}{2} Zb$, and we choose $b$ such that $\widetilde b_t = 0$ for $t=4,...,d$ and 
\begin{align}\label{de:bwidetilde}
\widetilde b_1 = &\frac{B}{\sqrt{d}} \Big( \frac{5}{4} L_x\beta^2+ L_x \beta\mu_y+ \frac{\widetilde L^2_{xy}L_y}{4}+\frac{L_x}{4} \mu_y^2\Big), \nonumber
\\ \widetilde b_2 =& \frac{B}{\sqrt{d}} (-L_x \beta^2 -\frac{L_x\beta }{2}\mu_y),\; \widetilde b_3 = \frac{B}{\sqrt{d}} \frac{L_x\beta^2}{4},
\end{align}
where the selection of $b$ is achievable because $Z$ is invertible with $Z^{-1}$ given by 
\begin{align*}
Z^{-1}= \begin{bmatrix}
 &   &  & -1\\
% &   & & 1 & -1\\
 &  &  -1& -1  \\
  &\text{\reflectbox{$\ddots$}} &\text{\reflectbox{$\ddots
  $}} & \vdots \\
  -1& -1  &-1  &-1 \\ 
\end{bmatrix}.
\end{align*}
Based on the form of $Z^2$ in~\cref{matrices_coupling_convex} and the forms of $Z^4, Z^6$ given by 
{\footnotesize
\begin{align}\label{coup_matrices_4_6}
Z^4=
\begin{bmatrix}
 5& -4 & 1 &  & &\\
 -4& 6  &-4 &1  & &\\
 1& -4 & 6 & -4 & 1&\\
  &  \ddots&  \ddots&  \ddots &  \ddots & \ddots\\
   &  & 1& -4 &  6& -3 \\
  &  & & 1 & -3 & 2\\ 
\end{bmatrix},\;
Z^6=
\begin{bmatrix}
 14& -14 & 6 & -1 & & & &\\
 -14& 20  &-15 & 6   & -1 & &&\\
 6& -15 & 20 & -15 & 6& -1&&\\
 -1&6& -15 & 20 & -15 & 6& -1&\\
&   \ddots&  \ddots&  \ddots&  \ddots &  \ddots & \ddots& \ddots\\
&& -1 & 6 & -15& 20 &  -15& 5 \\
 &&  &  -1& 6& -15 &  19& -9 \\
  && &  &-1 & 5 & -9 & 5\\
\end{bmatrix},
\end{align}
}
\hspace{-0.2cm} it can be checked from~\cref{x_star_convex}  that $x^* = \frac{B}{\sqrt{d}} {\bf 1}$, where $\bf 1$ denotes the all-one vector and thus $\|x^*\|=B$. 
%Based on the definitions of $x^*$ and $x^*_\epsilon = \argmin_x \big(\Phi_\epsilon(x):= \Phi(x)+\frac{\epsilon}{4B}\|x\|^2\big)$, we have $\nabla \Phi_\epsilon(x^{\epsilon}) =0 $ and $\nabla \Phi(x^*) = 0$, and hence
%\begin{align*}
%\nabla \Phi_\epsilon(x^*)-\nabla \Phi_\epsilon(x_\epsilon^*)= \frac{\epsilon}{2B} x^* ,
%\end{align*}
%which, combined with the strong convexity of $\Phi_\epsilon(x)$, yields
%\begin{align*}
%\frac{\epsilon}{2B}\|x^*-x_\epsilon^*\| \leq \| \nabla \Phi_\epsilon(x^*)-\nabla \Phi_\epsilon(x_\epsilon^*)\| =\frac{\epsilon}{4}. 
%\end{align*}
%This further implies that $\|x^*_\epsilon\| \leq B$ so that $x_\epsilon^*\in\mathcal{X}$. 
%Therefore, the constructed instance in~\cref{con_fg} satisfies \Cref{de:pc} in the convex-strongly-convex setting.  
%alpha = \frac{L_x}{4}, 
%\eta = L_y
%\xi = \frac{\widetilde L_{xy}}{2}

\vspace{0.2cm}
\noindent {\bf Step 3: Characterizing lower bound on $\|\nabla \Phi(x)\|$.}
\vspace{0.2cm}

Next, we characterize a lower bound on $\|\nabla \Phi(x)\|$ when the last three coordinates of $x$ are zeros, i.e., $x_{d-2}=x_{d-1}=x_d=0$. Let $\Omega = [I_{d-3}, {\bf 0}]^T$ and define $\widetilde x\in\mathbb{R}^{d-3}$ such that $\widetilde x_i=x_i$ for $i=1,...,d-3$. Then for any matrix $H$, $H\Omega$ is equivalent to removing the last three columns of $H$. Then,  based on the form of $\nabla \Phi(x)$ in~\cref{conv_gdform}, we have 
\begin{align}\label{eq:gdnorm1}
\min_{x\in\mathbb{R}^d: x_{d-2}=x_{d-1}=x_d=0} \|\nabla\Phi(x)\|^2 = \min_{\widetilde x\in\mathbb{R}^{d-3}} \|H\Omega \widetilde x -  (\beta Z^2 +\mu_y I)^{-2}\widetilde b\|^2
\end{align}
where the matrix $H$ is given by 
%carmon2019lower
\begin{align}\label{wideHdef}
H =  (\beta Z^2 +\mu_y I)^{-2}\underbrace{\Big (\frac{L_x\beta^2}{4} Z^6 +  \frac{L_x\beta^2\beta\mu_y}{2} Z^4 +\Big (\frac{L_y\widetilde L_{xy}^2}{4}+\frac{L_x\mu_y^2}{4}\Big) Z^2 \Big)}_{\widetilde H}.
\end{align}
Then using an approach similar to (7) in~\cite{carmon2019lower}, we have 
\begin{align}\label{eq:gdnomr2s}
\min_{\widetilde x\in\mathbb{R}^{d-3}} \|H\Omega \widetilde x -  (\beta Z^2 +\mu_y I)^{-2}\widetilde b\|^2 = \big(\widetilde b^T (\beta Z^2 +\mu_y I)^{-2} z\big)^2,
\end{align}
where $z$ is the normalized (i.e., $\|z\|=1$) solution of equation $(H\Omega)^Tz=0$. Next we characterize the solution $z$. Since $H=(\beta Z^2 +\mu_y I)^{-2}\widetilde H$, we have  
\begin{align}
(H\Omega)^Tz = (\widetilde H \Omega)^T(\beta Z^2 +\mu_y I)^{-2} z = 0.
\end{align}
Based on the definition of $\widetilde H $ in \cref{wideHdef} and the forms of $Z^2,Z^4,Z^6$ in \cref{matrices_coupling_convex} and~\cref{coup_matrices_4_6}, we have that the solution $z$ takes the form of $ z= \lambda (\beta Z^2 +\mu_y I)^{2}h$, where $\lambda$ is a factor such that $\|z\|=1$ and $h$ is a vector satisfying $h_t = t$ for $t=1,...,d$. Based on the definition of $Z^2$ in \cref{matrices_coupling_convex}, we have
\begin{align*}
1= \|z\| =& \lambda \sqrt{\sum_{i=1}^{d-2}(i\mu_y^2)^2 + ((d-1)\mu_y^2-\beta^2)^2 + (d\mu_y^2 +\beta^2+2\beta\mu_y)^2} \nonumber
\\\leq & \lambda \sqrt{\sum_{i=1}^{d-2}(i\mu_y^2)^2 + 2(d-1)^2\mu_y^4+2\beta^4 + 2d^2\mu_y^4 +2(\beta^2+2\beta\mu_y)^2} \nonumber
%\\\leq &\lambda \sqrt{2\mu_y^4\sum_{i=1}^d i^2 +4\beta^4+8\beta^3\mu_y+8\beta^2\mu_y^2}  \nonumber
\\<& \lambda\sqrt{\frac{2}{3}\mu_y^4(d+1)^3 +4\beta^4+8\beta^3\mu_y+8\beta^2\mu_y^2},
\end{align*}
which further implies that 
\begin{align}\label{ggnormsacsa}
\lambda > \frac{1}{\sqrt{\frac{2}{3}\mu_y^4(d+1)^3 +4\beta^4+8\beta^3\mu_y+8\beta^2\mu_y^2}}.
\end{align}
Then, combining \cref{eq:gdnorm1}, \cref{eq:gdnomr2s} and \cref{ggnormsacsa} yields
\begin{align}\label{min_gdnorm}
\min_{x: x_{d-2}=x_{d-1}=x_d=0} \|\nabla\Phi(x)\|^2 =&\big(\widetilde b^T (\beta Z^2 +\mu_y I)^{-2} z\big)^2 = (\lambda \widetilde b^T h)^2 = \lambda^2 (\widetilde b_1 + 2\widetilde b_2+3\widetilde b_3)^2 \nonumber
\\\overset{(i)}=& \lambda^2 \frac {B^2}{4d}\Big(\frac{\widetilde L^2_{xy}L_y}{4}+\frac{L_x\mu_y^2}{4}\Big)^2 \nonumber
\\ \geq& \frac{B^2\Big(\frac{\widetilde L^2_{xy}L_y}{4}+\frac{L_x\mu_y^2}{4}\Big)^2}{\frac{8}{3}\mu_y^4d(d+1)^3 +16d\beta^4+32d\beta^3\mu_y+32d\beta^2\mu_y^2} \nonumber
\\\overset{(ii)}\geq &  \frac{B^2\Big(\frac{\widetilde L^2_{xy}L_y}{4}+\frac{L_x\mu_y^2}{4}\Big)^2}{8\mu_y^4d^4 +16d\beta^4+32d\beta^3\mu_y+32d\beta^2\mu_y^2}
\end{align}
where $(i)$ follows from the definition of $\widetilde b$ in~\cref{de:bwidetilde}, and $(ii)$ follows because  $d\geq 3$. 
%where $(i)$ follows from $d\geq 4$. 

\vspace{0.2cm}
\noindent {\bf Step 4: Characterizing subspaces $\mathcal{H}_x^k$ and $\mathcal{H}_x^k$ .}
\vspace{0.2cm}

Based on the constructions of $f,g$ in \cref{con_fg}, we have 
%\begin{align}
%f(x,y) &= \frac{L_x}{8} x^T Z^2 x+ \frac{L_y}{2}\|y\|^2, \nonumber
%\\g(x,y) &= \frac{1}{2} y^T (\beta Z^2 +\mu_y I)y -\frac{\widetilde L_{xy}}{2} x^TZy + b^Ty, \;\;\mathcal{X} = \mathcal{B} := \{x: \|x\|\leq B\}
%\end{align}
\begin{align*}
\nabla_x f(x,y) &= \frac{L_x}{4} Z^2x,\;  \nabla_y f(x,y) = L_y y, \;\nabla_x\nabla_y g(x,y) = -\frac{\widetilde L_{xy}}{2} Z 
\nonumber
\\ \nabla_y^2g(x,y) &= \beta Z^2 +\mu_y I, \; \nabla_y g(x,y) = (\beta Z^2 + \mu_y I) y - \frac{\widetilde L_{xy}}{2} Zx + b, 
\end{align*}
which, in conjunction with \cref{hxk} and \cref{x_span}, yields 
\begin{align*}
\mathcal{H}_y^0 &= \mbox{Span}\{0\}, ...., \mathcal{H}_y^{s_0} = \mbox{Span}\{Z^{2(s_0-1)}b,...,Z^2b,b\} \nonumber
%\\\mathcal{H}_y^{s_0+1} &= \mbox{Span}\{Z^{2s_0}b,...,b,Z^{2(T+s_0+1)}b,...,b\}, ...., \mathcal{H}_y^{s_1} = \mbox{Span}\{Z^{2(s_1-1)}b,...,b,Z^{2(T+s_0+1)}b,...,b\} \nonumber
\\\mathcal{H}_x^0 &= .... \mathcal{H}_x^{s_0-1} =\mbox{Span}\{0\}, \mathcal{H}_x^{s_0} =\mbox{Span}\{Z^{2(T+s_0-2)}(Zb),....,Z^2(Zb),(Zb)\}. 
%\\\mathcal{H}_x^{s_0} &= .... \mathcal{H}_x^{s_1-1}  \subseteq\mbox{Span}\{Z^{2(T+s_0)}(Zb),....,Z(Zb),(Zb)\}, \mathcal{H}_x^{s_1}  =  \mbox{Span}\{Z^{2(2T+s_0+2)}(Zb),...,(Zb)\}
\end{align*}
Repeating the above procedure and noting that $s_{Q-1}\leq K$ yield
\begin{align}\label{eq:subsscas} 
\mathcal{H}_x^K = \mathcal{H}_x^{s_{Q-1}} &=  \mbox{Span}\{Z^{2(s_{Q-1}+QT-Q-1)}(Zb),....,Z^2(Zb),(Zb)\}\nonumber
\\&\subseteq \mbox{Span}\{Z^{2(K+QT-Q)}(Zb),....,Z^2(Zb),(Zb)\}. 
\end{align}

\noindent {\bf Step 5: Characterizing convergence and complexity.}
\vspace{0.2cm}

Let $M = K+QT-Q+3$ and consider the following equation 
%\begin{align*}
%\frac{D^2\Big(\frac{\widetilde L^2_{xy}L_y}{4}+\frac{L_x\mu_y^2}{4}\Big)^2}{2\mu_y^4d^4 +4d\beta^4+8d\beta^3\mu_y+8d\beta^2\mu_y^2}  \geq \epsilon^2,
%\end{align*}
%which implies that $d$ satisfies the following equation
\begin{align}\label{d_eqtion}
r^4 +r \Big(\frac{2\beta^4}{\mu_y^4}+\frac{4\beta^3}{\mu_y^3}+\frac{4\beta^2}{\mu_y^2}\Big) = \frac{B^2\Big(\widetilde L^2_{xy}L_y+L_x\mu_y^2\Big)^2}{128\mu_y^4\epsilon^2},
\end{align}
which has a solution denoted as $r^*$.  
%\begin{align}\label{r_startssc}
%r^* = \sqrt{\frac{\sqrt{\frac{D^2\Big(\widetilde L^2_{xy}L_y+L_x\mu_y^2\Big)^2}{8\mu_y^4\epsilon^2}+\Big(\frac{2\beta^4}{\mu_y^4}+\frac{4\beta^3}{\mu_y^3}+\frac{4\beta^2}{\mu_y^2}\Big)^2 }-\Big(\frac{2\beta^4}{\mu_y^4}+\frac{4\beta^3}{\mu_y^3}+\frac{4\beta^2}{\mu_y^2}\Big)}{2} }.
%\end{align}
We choose $d = \lfloor r^* \rfloor$. Then, based on  \cref{min_gdnorm}, we have
\begin{align}\label{eq:quitescs}
\min_{x: x_{d-2}=x_{d-1}=x_d=0} \|\nabla\Phi(x)\|^2 \geq  \frac{B^2\Big(\frac{\widetilde L^2_{xy}L_y}{4}+\frac{L_x\mu_y^2}{4}\Big)^2}{8\mu_y^4(r^*)^4 +16r^*\beta^4+32r^*\beta^3\mu_y+32r^*\beta^2\mu_y^2} = \epsilon^2.
\end{align}
Then, to achieve $\|\nabla\Phi(x^K)\|< \epsilon$, it requires that $M> d-3$. Otherwise (i.e., if $M\leq d-3$), based on \cref{eq:subsscas} and the fact that $Zb$ has nonzeros only at the first three coordinates, we have $x^K$ has zeros at the last three coordinates, and hence \cref{eq:quitescs} yields 
$\|\nabla \Phi(x^K)\|\geq \epsilon$, which leads to a contradiction. Therefore, we have $M>  \lfloor r^* \rfloor-3$.

To characterize the total complexity, using the metric in~\Cref{complexity_measyre}, we have 
\begin{align*} 
\mathcal{C}_{\text{\normalfont grad}}(\mathcal{A},\epsilon)\geq  \Omega(Q+QT+K) \geq  \Omega(M) \geq  \Omega(r^*).
\end{align*}
 Then, the proof is complete. 
%\end{proof}
\section{Proof of \Cref{co:co1}}
In this case, the condition number $\kappa_y$ satisfies $\kappa_y=\frac{\widetilde L_y}{\mu_y}\leq \mathcal{O}(1)$.
%, which means that the inner-level problem is easy to solve. 
Then, it can be verified that $r^*$ satisfies $(r^*)^3> \Omega (\frac{2\beta^4}{\mu_y^4}+\frac{4\beta^3}{\mu_y^3}+\frac{4\beta^2}{\mu_y^2})$, and hence it follows from \cref{eq:rsoltion} that 
\begin{align*}
\mathcal{C}_{\text{\normalfont grad}}(\mathcal{A},\epsilon) \geq r^*\geq \Omega \Big(\frac{B^{\frac{1}{2}}(\widetilde L^2_{xy}L_y+L_x\mu_y^2)^{\frac{1}{2}}}{\mu_y\epsilon^{\frac{1}{2}}} \Big).
\end{align*}
\section{Proof of \Cref{co:co2}}
To prove  \Cref{co:co2}, we consider two cases $\mu_y\geq \Omega(\epsilon^{\frac{3}{2}})$ and $\mu_y\leq \mathcal{O}(\epsilon^{\frac{3}{2}})$ separately. 

{\bf Case 1: $\mu_y\geq \Omega(\epsilon^{\frac{3}{2}})$.} For this case, we have 
$\big(\frac{2\beta^4}{\mu_y^4}+\frac{4\beta^3}{\mu_y^3}+\frac{4\beta^2}{\mu_y^2}\big)\leq \mathcal{O}\big(\frac{1}{\mu_y^{3}\epsilon^{3/2}}\big)$. Then, it follows from \cref{eq:rsoltion} that $\mathcal{C}_{\text{\normalfont grad}}(\mathcal{A},\epsilon) \geq r^*\geq \Omega \big(\frac{1}{\mu_y\epsilon^{1/2}}\big)$. 

{\bf Case 2: $\mu_y\leq \mathcal{O}(\epsilon^{\frac{3}{2}})$.} For this case, first suppose $(r^*)^3\leq  \mathcal{O}\big(\frac{2\beta^4}{\mu_y^4}+\frac{4\beta^3}{\mu_y^3}+\frac{4\beta^2}{\mu_y^2}\big) $, and then it follows from \cref{eq:rsoltion}  that $r^* \geq \Omega(\frac{1}{\epsilon^{2}}) $. On the other hand, if $(r^*)^3\geq  \Omega\big(\frac{2\beta^4}{\mu_y^4}+\frac{4\beta^3}{\mu_y^3}+\frac{4\beta^2}{\mu_y^2}\big) $, then we obtain from \cref{eq:rsoltion}  that $r^*\geq \Omega(\frac{1}{\mu_y\epsilon^{1/2}})\geq \Omega(\frac{1}{\epsilon^{2}})$, which yields $\mathcal{C}_{\text{\normalfont grad}}(\mathcal{A},\epsilon)\geq r^*\geq \Omega(\frac{1}{\epsilon^{2}})$.
Then, combining these two cases finishes the proof. 

\section{Proof of \Cref{upper_srsr_withnoB}}\label{proof:upss} 
To simplify the notations, we define the following quantities. 
\begin{align}\label{pf:ntationscs}
\mathcal{M}_k =& \|y^*(x^*)\|+ \frac{\widetilde L_{xy}}{\mu_y}\|x_k-x^*\|, \;\;\mathcal{N}_k = \|\nabla_y f( x^*,y^*(x^*))\|+ \Big(L_{xy}+\frac{L_y\widetilde L_{xy}}{\mu_y}\Big)\|x_k-x^*\| \nonumber
\\\mathcal{M}_* =& \|y^*(x^*)\|+ \frac{3\widetilde L_{xy}}{\mu_y}\sqrt{\frac{2}{\mu_x}(\Phi(0) -\Phi(x^*))+ \|x^*\|^2+\frac{\epsilon}{\mu_x}}
 \nonumber%= \text{poly}\Big(  \|y^*(x^*)\|+ \frac{\|x^*\|+\epsilon}{\mu_x\mu_y} \Big) \nonumber
\\\mathcal{N}_* =&\|\nabla_y f( x^*,y^*(x^*))\|+ 3\Big(L_{xy}+\frac{L_y\widetilde L_{xy}}{\mu_y}\Big)\sqrt{\frac{2}{\mu_x}(\Phi(0) -\Phi(x^*))+ \|x^*\|^2+\frac{\epsilon}{\mu_x}}, %= \text{poly}\Big(\|\nabla_y f( x^*,y^*(x^*))\|+\frac{\|x^*\|+\epsilon}{\mu_x\mu_y}  \Big)
\end{align}
where $\mathcal{M}_k$ and $\mathcal{N}_k$ change with the optimality gap $\|x_k-x^*\|$ at the $k^{th}$ iteration, and $\mathcal{M}_*$ and $\mathcal{N}_*$ are two positive constants depending on the information of the objective function at the optimal point $x^*$.
 %where $U=\max\big\{\|x^*\|,\|y^*(x^*)\|,\|\nabla_y f( x^*,y^*(x^*))\| \big\}$. 
We first establish the following lemma to 
upper-bound the hypergradient estimation error $\|\nabla\Phi( x_k)-G_k\|$. 
\begin{lemma}\label{le:hgestr}
Let $G_k$ be the hypergradient estimator used in \Cref{alg:bioNoBG} at iteration $k$. Then, we have 
\begin{align}\label{eq:hgesterr}
\|G_k-\nabla \Phi( x_k)\| \leq &\sqrt{\frac{\widetilde L_y +\mu_y}{\mu_y}} \Big(L_y +\frac{2\widetilde L_{xy}L_y}{\mu_y} +\Big(\frac{\rho_{xy}}{\mu_y}+\frac{\widetilde L_{xy}\rho_{yy}}{\mu_y^2}\Big)\mathcal{N}_k\Big) \mathcal{M}_k \exp\Big(-\frac{N}{2\sqrt{\kappa_y}}\Big) \nonumber
\\&+\frac{\widetilde L_{xy}}{\mu_y}\Big(\frac{\sqrt{\kappa_y}-1}{\sqrt{\kappa_y}+1}\Big)^M\mathcal{N}_k,
\end{align}
where the quantities $\mathcal{M}_k$ and $\mathcal{N}_k$ are defined in \cref{pf:ntationscs}. 
\end{lemma}
\Cref{le:hgestr} shows that the estimation error $\|\nabla\Phi( x_k)-G_k\|$ is bounded given that the optimality gap $\|x_k-x^*\|$ is bounded. We will show in the proof of \Cref{upper_srsr_withnoB} that $\|x_k-x^*\|$ is bounded as the algorithm runs due to the strongly-convex geometry of the objective function $\Phi(x)$. In addition, it can be seen that this error decays exponentially with respect to the number $N$ of inner-level steps and the number $M$ of steps of the heavy-ball method for solving the linear system in  \Cref{alg:bioNoBG}. 
Then, to prove the convergence of \Cref{alg:bioNoBG}, we set $N=M=c\sqrt{\kappa_y}\log (\kappa_y)$ in the proof of \Cref{upper_srsr_withnoB}, where $c$ is a constant independent of $\kappa_y$.

\begin{proof}
Recall line $7$ of \Cref{alg:bioNoBG} that 
\begin{align}\label{laomuzhu}
G_k:= \nabla_x f(x_k,y_k^N) -\nabla_x \nabla_y g( x_k,y_k^N)v_k^M,
\end{align}
where $v_k^M$ is the $M^{th}$ step output of the heavy-ball method for solving $$\min_v Q(v):=\frac{1}{2}v^T\nabla_y^2 g(x_k,y_k^N) v - v^T
\nabla_y f( x_k,y^N_k).$$
Recall the smoothness parameter $\widetilde L_y$ of $g(x,\cdot)$ defined in Assumption~\ref{fg:smooth}. 
Then, based on the convergence result of the heavy-ball method in~\cite{badithela2019analysis} with stepsizes $\lambda=\frac{4}{(\sqrt{\widetilde L_y}+\sqrt{\mu_y})^2}$ and $\theta=\max\big\{\big(1-\sqrt{\lambda\mu_y}\big)^2,\big(1-\sqrt{\lambda\widetilde L_y}\big)^2\big\}$ and noting that $v_k^0=v_k^1=0$, we have 
\begin{align}\label{gg:worimass}
\|v_k^M - \nabla_y^2 &g(x_k,y_k^N)^{-1}\nabla_y f( x_k,y^N_k) \| \nonumber
\\\leq  &\Big(\frac{\sqrt{\kappa_y}-1}{\sqrt{\kappa_y}+1}\Big)^M \Big\| \big(\nabla_y^2 g(x_k,y_k^N)\big)^{-1}\nabla_y f(x_k,y^N_k)\Big\| \nonumber
%\\\leq &\frac{1}{\mu_y}\Big(\frac{\sqrt{\kappa_y}-1}{\sqrt{\kappa_y}+1}\Big)^M \|\nabla_y f( x_k,y_k^N)\| \nonumber
\\\leq &\frac{L_y}{\mu_y}\Big(\frac{\sqrt{\kappa_y}-1}{\sqrt{\kappa_y}+1}\Big)^M \|y^*(x_k)-y_k^N\|  + \frac{\|\nabla_y f( x_k,y^*(x_k))\|}{\mu_y}\Big(\frac{\sqrt{\kappa_y}-1}{\sqrt{\kappa_y}+1}\Big)^M \nonumber
\\\overset{(i)}\leq &\frac{L_y}{\mu_y} \|y^*(x_k)-y_k^N\|  + \frac{\|\nabla_y f( x_k,y^*(x_k))\|}{\mu_y}\Big(\frac{\sqrt{\kappa_y}-1}{\sqrt{\kappa_y}+1}\Big)^M
\end{align}
where $y^*(x_k)=\argmin_{y\in\mathbb{R}^q} g( x_k,y)$ and $(i)$ follows because $\frac{\sqrt{\kappa_y}-1}{\sqrt{\kappa_y}+1}\leq 1$. 
 Then, based on the forms of $G_k$ and $\nabla\Phi(x)$ in \cref{laomuzhu} and \cref{hyperG}, and using Assumptions~\ref{fg:smooth} and~\ref{g:hessiansJaco}, we have 
\begin{align}\label{jingyikeai}
\|G_k&-\nabla \Phi(x_k)\| \nonumber
\\\overset{(i)}\leq & \| \nabla_x f( x_k,y_k^N) -\nabla_x f( x_k,y^*(x_k))\| + \widetilde L_{xy}\|v_k^M- \nabla_y^2 g(x_k,y^*(x_k))^{-1}\nabla_y f(x_k,y^*(x_k)) \|  \nonumber
\\&+\frac{\|\nabla_y f( x_k,y^*(x_k)) \|}{\mu_y}  \|\nabla_x \nabla_y g( x_k,y_k^N)-\nabla_x \nabla_y g( x_k,y^*(x_k))\| \nonumber
\\\leq & L_y \|y^*(x_k)-y_k^N\| + \widetilde L_{xy} \|v_k^M- \nabla_y^2 g(x_k,y_k^N)^{-1}\nabla_y f( x_k,y^N_k) \| \nonumber
\\&+ \widetilde L_{xy}\big\|\nabla_y^2 g(x_k,y_k^N)^{-1}\nabla_y f( x_k,y^N_k)-\nabla_y^2 g(x_k,y^*(x_k))^{-1}\nabla_y f(x_k,y^*(x_k)) \big\|\nonumber
\\&+\frac{\rho_{xy}}{\mu_y} \|y_k^N-y^*(x_k)\| \|\nabla_y f( x_k,y^*(x_k)) \|\nonumber
\\\leq & \Big(L_y +\frac{\widetilde L_{xy}L_y}{\mu_y} +\frac{\rho_{xy}}{\mu_y}\|\nabla_y f( x_k,y^*(x_k)) \|\Big)\|y_k^N-y^*(x_k)\|  \nonumber
\\&+ \frac{\widetilde L_{xy}\rho_{yy}\|y_k^N-y^*(x_k)\|}{\mu_y^2}\|\nabla_y f( x_k,y^*(x_k)) \| + \widetilde L_{xy} \|v_k^M- \nabla_y^2 g(x_k,y_k^N)^{-1}\nabla_y f( x_k,y^N_k) \|  \nonumber
\\\overset{(ii)}\leq&\Big(L_y +\frac{2\widetilde L_{xy}L_y}{\mu_y} +\Big(\frac{\rho_{xy}}{\mu_y}+\frac{\widetilde L_{xy}\rho_{yy}}{\mu_y^2}\Big)\|\nabla_y f( x_k,y^*(x_k)) \|\Big)\|y_k^N-y^*(x_k)\|  \nonumber
\\&+\frac{\widetilde L_{xy}}{\mu_y}\left(\frac{\sqrt{\kappa_y}-1}{\sqrt{\kappa_y}+1}\right)^M\|\nabla_y f( x_k,y^*(x_k))\|,
\end{align}
where $(i)$ follows from Assumption~\ref{fg:smooth} that $\|\nabla_x\nabla_y g(\cdot,\cdot)\|\leq \widetilde L_{xy}$ and $\|(\nabla_y^2 g(\cdot,\cdot))^{-1}\|\leq \frac{1}{\mu_y}$ and $(ii)$ follows from \cref{gg:worimass}. Note that $y_k^N$ is obtained as the $N$-step output of AGD for minimizing the inner-level loss function $g(x_k,\cdot)$ and recall $y^*(x_k)=\argmin_{y\in\mathbb{R}^q} g( x_k,y)$. Then, based on the analysis in \cite{nesterov2003introductory} for AGD, we have 
\begin{align}\label{eq:ideazhiqian}
\|y_k^N-y^*(x_k)\|\leq &\sqrt{ \frac{\widetilde L_y +\mu_y}{\mu_y} }\|y_k^0-y^*(x_k)\| \exp\Big(-\frac{N}{2\sqrt{\kappa_y}}\Big) \nonumber
\\\leq & \sqrt{\frac{\widetilde L_y +\mu_y}{\mu_y}} \Big(\|y^*(x^*)\| + \frac{\widetilde L_{xy}}{\mu_y}\|x_k-x^*\|\Big) \exp\Big(-\frac{N}{2\sqrt{\kappa_y}}\Big), 
%\\\leq &\underbrace{\frac{2(\widetilde L_y +\mu_y)}{\mu_y} \exp\Big(-\frac{N}{\sqrt{\kappa_y}}\Big) }_{\tau_N}(\|y_{k-1}^N-y_{k-1}^*\|^2 + \kappa_y\|\widetilde x_k - \widetilde x_{k-1}\|^2),
\end{align}
where $x^*=\argmin_{x\in\mathbb{R}^p}\Phi(x)$. 
Moreover, based on Lemma 2.2 in~\cite{ghadimi2018approximation}, we have $\|y^*(x_1)-y^*(x_2)\|\leq \frac{\widetilde L_{xy}}{\mu_y} \|x_1-x_2\|$ for any $x_1,x_2\in\mathbb{R}^p$, and hence 
%which, combined with \cref{gg:opascas}, yields 
\begin{align}\label{maoxinadaoxx}
\|\nabla_y &f( x_k,y^*(x_k))\| \leq \|\nabla_y f( x^*,y^*(x^*))\| + \Big( L_{xy}+\frac{L_y\widetilde L_{xy}}{\mu_y}\Big)\|x_k-x^*\|.
%\\\|v_k^M - \nabla_y^2 &g(x_k,y_k^N)^{-1}\nabla_y f( x_k,y^N_k) \| \nonumber
%\\\leq &\frac{L_y}{\mu_y}\Big(\frac{\sqrt{\kappa_y}-1}{\sqrt{\kappa_y}+1}\Big)^M \|y^*(x_k)-y_k^N\|  + \frac{\|\nabla_y f( x^*,y^*(x^*))\|}{\mu_y}\Big(\frac{\sqrt{\kappa_y}-1}{\sqrt{\kappa_y}+1}\Big)^M  \nonumber
%\\&+  \frac{\widetilde L_{xy}+\widetilde L_{xy}\kappa_y}{\mu_y}\Big(\frac{\sqrt{\kappa_y}-1}{\sqrt{\kappa_y}+1}\Big)^M\|x_k-x^*\|
\end{align}
Substituting \cref{eq:ideazhiqian}  and \cref{maoxinadaoxx} into \cref{jingyikeai}, and using the definition of  $\mathcal{M}_k$ and $\mathcal{N}_k$ in \cref{pf:ntationscs}, we have 
\begin{align*}
\|G_k-\nabla \Phi( x_k)\| \leq &\sqrt{\frac{\widetilde L_y +\mu_y}{\mu_y}} \Big(L_y +\frac{2\widetilde L_{xy}L_y}{\mu_y} +\Big(\frac{\rho_{xy}}{\mu_y}+\frac{\widetilde L_{xy}\rho_{yy}}{\mu_y^2}\Big)\mathcal{N}_k\Big) \mathcal{M}_k \exp\Big(-\frac{N}{2\sqrt{\kappa_y}}\Big) \nonumber
\\&+\frac{\widetilde L_{xy}}{\mu_y}\Big(\frac{\sqrt{\kappa_y}-1}{\sqrt{\kappa_y}+1}\Big)^M\mathcal{N}_k,
\end{align*}
which completes the proof. 
\end{proof}
We then establish the following lemma to characterize the smoothness parameter of the objective function $\Phi(x)$ around the iterate $x_k$. 
Recall \cref{hyperG} that $\nabla \Phi(x)$ is given by 
\begin{align}\label{hgforms}
\nabla \Phi(x) =  \nabla_x f(x,y^*(x)) -\nabla_x \nabla_y g(x,y^*(x)) [\nabla_y^2 g(x,y^*(x)) ]^{-1}\nabla_y f(x,y^*(x)),
\end{align}
where $y^*(x)=\argmin_{y}g(x,\cdot)$ denotes  the minimizer of the inner-level function $g(x,\cdot)$. 
\begin{lemma}\label{smoothness_Phis}
Consider the hypergradient $\nabla \Phi(x)$ given by \cref{hgforms}. For any $x\in\mathbb{R}^p$, we have
\begin{align}\label{maoxizaosscsa}
\|\nabla &\Phi(x)  - \nabla \Phi(x_k)\|  \nonumber
\\&\leq \Big(\underbrace{ L_x+\frac{2L_{xy}\widetilde L_{xy}}{\mu_y} + \frac{L_y\widetilde L^2_{xy}}{\mu_y^2}+\Big(\frac{\widetilde L_{xy} \rho_{yy}}{\mu_y^2} +  \frac{\rho_{xy}}{\mu_y}\Big) \Big( 1+\frac{\widetilde L_{xy}}{\mu_y}  \Big)\mathcal{N}_k}_{L_{\Phi_k}}\Big)\|x-x_k\|, \end{align}
where $\mathcal{N}_k$ is defined in \cref{pf:ntationscs}. Furthermore, \cref{maoxizaosscsa} implies that, for any $x\in\mathbb{R}^p$,  
\begin{align}
\Phi(x)\leq \Phi(x_k)+\langle \nabla \Phi(x_k),x-x_k\rangle +  \frac{L_{\Phi_k}}{2}\|x-x_k\|^2.
\end{align}
\end{lemma}
\Cref{smoothness_Phis} shows that $\nabla\Phi(x)$ is Lipschitz continuous around the iterate $x_k$, i.e., $\Phi(x)$ is smooth, where the smoothness parameter $L_{\Phi_k}$ contains a term proportional to $\|x_k-x^*\|$. 
We will show in the proof of \Cref{upper_srsr_withnoB} that the optimality distance $\|x_k-x^*\|$ is bounded as the algorithm runs, and hence the smoothness parameter $L_{\Phi_k}$ is bounded by $\mathcal{O}(\frac{1}{\mu_y^3})$ during the entire process. %As a result, we can guarantee the convergence of \Cref{alg:bioNoBG}.    

\begin{proof}
Based on the form of $\nabla \Phi(x)$ in \cref{hgforms}, we have 
{\small
\begin{align*}
\|\nabla &\Phi(x)  - \nabla \Phi(x_k)\|  
\\\leq& \| \nabla_x f(x,y^*(x)) - \nabla_x f(x_k,y^*(x_k)) \|+\frac{\widetilde L_{xy}}{\mu_y} \|\nabla_y f(x,y^*(x))-\nabla_y f(x_k,y^*(x_k))\| 
\\&+ \underbrace{\|\nabla_x \nabla_y g(x,y^*(x)) \nabla_y^2 g(x,y^*(x))^{-1}-\nabla_x \nabla_y g(x_k,y^*(x_k)) \nabla_y^2 g(x_k,y^*(x_k))^{-1}\|}_{P}\|\nabla_y f(x_k,y^*(x_k))\|,
%\\&+\frac{\widetilde}{\mu_y} L_{xy}\|\nabla_x \nabla_y g(x,y^*(x)) [\nabla_y^2 g(x,y^*(x)) ]^{-1}\|
\end{align*}}
\hspace{-0.14cm}which, in conjunction with the inequality
\begin{align*}
P \leq &\frac{\widetilde L_{xy} \rho_{yy}}{\mu_y^2} (\|x-x_k\| + \|y^*(x) -y^*(x_k)\|) + \frac{\rho_{xy}}{\mu_y}(\|x-x_k\| + \|y^*(x) -y^*(x_k)\|) 
\\\overset{(i)}\leq  & \Big(\frac{\widetilde L_{xy} \rho_{yy}}{\mu_y^2} +  \frac{\rho_{xy}}{\mu_y}\Big) \Big( 1+\frac{\widetilde L_{xy}}{\mu_y}  \Big) \|x-x_k\|,
\end{align*}
and using Assumption~\ref{fg:smooth}, yields 
\begin{align}\label{eq:particsasq}
\|\nabla \Phi(x)  - \nabla \Phi(x_k)\| \leq& \Big( L_x+\frac{2L_{xy}\widetilde L_{xy}}{\mu_y} + \frac{L_y\widetilde L^2_{xy}}{\mu_y^2}\Big)\|x-x_k\| \nonumber
\\&+ \Big(\frac{\widetilde L_{xy} \rho_{yy}}{\mu_y^2} +  \frac{\rho_{xy}}{\mu_y}\Big) \Big( 1+\frac{\widetilde L_{xy}}{\mu_y}  \Big)\|\nabla_y f(x_k,y^*(x_k))\| \|x-x_k\|,
\end{align}
where $(i)$ follows from the $\frac{\widetilde L_{xy}}{\mu_y}$-smoothness of $y^*(\cdot)$. 
Substituting \cref{maoxinadaoxx} into \cref{eq:particsasq} and using the definition of $\mathcal{N}_k$ in \cref{pf:ntationscs}, we have
\begin{align}\label{ijj:ntidesaca}
\|\nabla \Phi(x) & - \nabla \Phi(x_k)\| \nonumber
\\\leq& \Big(\underbrace{ L_x+\frac{2L_{xy}\widetilde L_{xy}}{\mu_y} + \frac{L_y\widetilde L^2_{xy}}{\mu_y^2}+\Big(\frac{\widetilde L_{xy} \rho_{yy}}{\mu_y^2} +  \frac{\rho_{xy}}{\mu_y}\Big) \Big( 1+\frac{\widetilde L_{xy}}{\mu_y}  \Big)\mathcal{N}_k}_{L_{\Phi_k}}\Big)\|x-x_k\|. 
\end{align}
Based on \cref{ijj:ntidesaca}, we further obtain
\begin{align*}
|\Phi(x)-\Phi(x_k)-&\langle \nabla \Phi(x_k),x-x_k\rangle| \nonumber
\\=& \Big|\int_0^1 \langle \nabla\Phi(x_k+t(x-x_k)),x-x_k \rangle d t-\langle \nabla \Phi(x_k),x-x_k\rangle\Big| \nonumber
\\\leq& \Big| \int_0^1 \langle \nabla\Phi(x_k+t(x-x_k))-\nabla \Phi(x_k),x-x_k \rangle d t \Big| \nonumber
\\\leq & \Big| \int_0^1 \| \nabla\Phi(x_k+t(x-x_k))-\nabla \Phi(x_k)\| \|x-x_k\| d t \Big| \nonumber
\\\leq &\Big| \int_0^1 L_{\Phi_k}\|(x-x_k)\|^2 td t \Big| = \frac{L_{\Phi_k}}{2}\|x-x_k\|^2.
\end{align*}
Then, the proof is now complete. 
\end{proof}
Based on \Cref{le:hgestr} and \Cref{smoothness_Phis}, we are ready to prove \Cref{upper_srsr_withnoB}. 
\begin{proof}[{\bf Proof of \Cref{upper_srsr_withnoB}}]
 \Cref{alg:bioNoBG} conducts the following updates 
\begin{align}\label{updateRule}
z_{k+1}=&x_k -\frac{1}{L_\Phi} G_k, \nonumber
\\x_{k+1}=&\Big(1+\frac{\sqrt{\kappa_x}-1}{\sqrt{\kappa_x}+1}\Big)z_{k+1} - \frac{\sqrt{\kappa_x}-1}{\sqrt{\kappa_x}+1} z_k,
\end{align}
where the smoothness parameter $L_{\Phi}$ takes the form of 
\begin{align}\label{wodishenamxd}
 L_{\Phi} =& L_x+\frac{2L_{xy}\widetilde L_{xy}}{\mu_y} + \frac{L_y\widetilde L^2_{xy}}{\mu_y^2}  +\Big(\frac{\widetilde L_{xy} \rho_{yy}}{\mu_y^2} +  \frac{\rho_{xy}}{\mu_y}\Big) \Big( 1+\frac{\widetilde L_{xy}}{\mu_y}  \Big)\|\nabla_y f( x^*,y^*(x^*))\|\nonumber
 \\&+3\Big(\frac{\widetilde L_{xy} \rho_{yy}}{\mu_y^2} +  \frac{\rho_{xy}}{\mu_y}\Big) \Big( 1+\frac{\widetilde L_{xy}}{\mu_y}  \Big)\Big(L_{xy}+\frac{L_y\widetilde L_{xy}}{\mu_y}\Big)\sqrt{\frac{2}{\mu_x}(\Phi(0) -\Phi(x^*))+ \|x^*\|^2+\frac{\epsilon}{\mu_x}} \nonumber
 \\= &\Theta\Big(\frac{1}{\mu_y^2}+\Big(\frac{ \rho_{yy}}{\mu_y^3} +  \frac{\rho_{xy}}{\mu_y^2}\Big)\Big(\|\nabla_y f( x^*,y^*(x^*))\|+\frac{\|x^*\|}{\mu_y}+\frac{\sqrt{\Phi(0)-\Phi(x^*)}}{\sqrt{\mu_x}\mu_y}\Big)  \Big),
\end{align}
and  
$\kappa_x = \frac{L_\Phi}{\mu_x}$ is the condition number of the objective function $\Phi(x)$. 

The remaining proof adapts the results in Section 2.2.5 of \cite{nesterov2018lectures}, but with two key differences: we need to (a) prove the boundedness of the iterates as the algorithm runs, and (b) carefully handle the hypergradient estimation error in the convergence analysis for accelerated gradient methods. In specific, we first construct the estimate sequences as follows. 
\begin{align}\label{ssses_seq}
S_0(x) =& \Phi(x_0) +\frac{\mu_x}{2} \|x-x_0\|^2 \nonumber
\\S_{k+1}(x) = & \Big(1 -\frac{1}{\sqrt{\kappa_x}} \Big)S_{k}(x)  +\frac{1}{\sqrt{\kappa_x}} \Big( \Phi(x_k) +\langle G_k,x-x_k\rangle + \frac{\mu_x}{2}\|x - x_k\|^2 +  \frac{\epsilon}{4}\Big).
\end{align}
Note that $\nabla^2S_0(x) = \mu_x I $  and $\nabla^2 S_{k+1}(x) = \big(1 -\frac{1}{\sqrt{\kappa_x}} \big)\nabla^2 S_{k}(x)+\frac{ \mu_x}{\sqrt{\kappa_x}}I$. Then, by induction, it can be verified that $\nabla^2 S_{k}(x)=\mu_x I$ for all $k=0,...,K$. This implies that $S_k(x)$ can be written as $S_k(x) = S_k^* + \frac{\mu_x}{2}\|x-v_k\|^2$, where $v_k = \argmin_{x\in\mathbb{R}^p}S_k(x)$. Next, we show by induction that
\begin{align}
&1.\quad \|z_k-x^*\|\leq \sqrt{\frac{2}{\mu_x}(\Phi(0) -\Phi(x^*))+ \|x^*\|^2+\frac{\epsilon}{\mu_x}} \text{ for all } k=0,...,K. \label{pocasca}
\\&2.\quad S_k^*\geq  \Phi(z_k) \text{ for all } k=0,...,K. \label{dasimaomaoss}
\end{align} 
Combining the first item in~\cref{pocasca} with the updates in~\cref{updateRule} also implies the boundedness of the sequence $x_k,k=0,...,K$ by noting that 
\begin{align}\label{boundednessofx_k}
\|x_{k}-x^*\| \leq &\Big(1+\frac{\sqrt{\kappa_x}-1}{\sqrt{\kappa_x}+1}\Big)\|z_{k} -x^*\|+ \frac{\sqrt{\kappa_x}-1}{\sqrt{\kappa_x}+1} \|z_{k-1}-x^*\| \nonumber
\\ \leq& 3\sqrt{\frac{2}{\mu_x}(\Phi(0) -\Phi(x^*))+ \|x^*\|^2+\frac{\epsilon}{\mu_x}}. 
\end{align}
Next, we prove the above two items given in \cref{pocasca} and \cref{dasimaomaoss} by induction. First, it can be verified that they hold for $k=0$ by noting that $\|z_0-x^*\|=\|x^*\|$ and $S_0^* = \Phi(x_0)$. Then, we suppose that they hold for all $k=0,...,k^\prime$ and prove the $k^\prime+1$ case. 

Based on \Cref{smoothness_Phis}, we have, for all $k=0,...,k^\prime$,  
\begin{align}\label{dotahaonanss}
\Phi(z_{k+1})\leq& \Phi(x_k)+\langle \nabla \Phi(x_k),z_{k+1}-x_k\rangle +  \frac{L_{\Phi_k}}{2}\|z_{k+1}-x_k\|^2 \nonumber
\\\overset{(i)}=& \Phi(x_k)-\frac{1}{L_\Phi}\langle \nabla \Phi(x_k),G_k\rangle +  \frac{L_{\Phi_k}}{2L_\Phi^2}\|G_k\|^2,  
\end{align}
 where $(i)$ follows from the updates in~\cref{updateRule}. Note that for $k=0,...,k^\prime$, it is seen from \cref{boundednessofx_k} that the optimality gap  $\|x_k-x^*\|\leq 3\sqrt{\frac{2}{\mu_x}(\Phi(0) -\Phi(x^*))+ \|x^*\|^2+\frac{\epsilon}{\mu_x}}$, which, combined with the definition of $L_{\Phi_k}$ in \cref{maoxizaosscsa}, 
% \begin{align}
% L_{\Phi_k} \leq& L_x+\frac{2L_{xy}\widetilde L_{xy}}{\mu_y} + \frac{L_y\widetilde L^2_{xy}}{\mu_y^2}  +\Big(\frac{\widetilde L_{xy} \rho_{yy}}{\mu_y^2} +  \frac{\rho_{xy}}{\mu_y}\Big) \Big( 1+\frac{\widetilde L_{xy}}{\mu_y}  \Big)\|\nabla_y f( x^*,y^*(x^*))\|\nonumber
% \\&+\Big(\frac{\widetilde L_{xy} \rho_{yy}}{\mu_y^2} +  \frac{\rho_{xy}}{\mu_y}\Big) \Big( 1+\frac{\widetilde L_{xy}}{\mu_y}  \Big)\Big(L_{xy}+\frac{L_y\widetilde L_{xy}}{\mu_y}\Big)\Big(\frac{6+3\mu_x}{\mu_x}\|x^*\| + \frac{12\epsilon}{\mu_x}\Big)=L_\Phi,
% \end{align}
% which, 
 yields $L_{\Phi_k} \leq L_\Phi$ for all $k=0,...,k^\prime$, where $L_\Phi$ is given by \cref{wodishenamxd}. Then, we obtain from \cref{dotahaonanss} that for all $k=0,...,k^\prime$, 
 \begin{align}\label{mxdhuaqingduo}
 \Phi(z_{k+1})\leq& \Phi(x_k)-\frac{1}{L_\Phi}\langle \nabla \Phi(x_k),G_k\rangle +  \frac{1}{2L_\Phi}\|G_k\|^2 \nonumber
 \\=& \Phi(x_k)-\frac{1}{L_\Phi}\| \nabla \Phi(x_k)\|^2 - \frac{1}{L_\Phi} \langle \nabla \Phi(x_k),G_k- \nabla \Phi(x_k)\rangle +  \frac{1}{2L_\Phi}\|G_k\|^2 \nonumber
 \\=& \Phi(x_k)-\frac{1}{L_\Phi}\| \nabla \Phi(x_k)\|^2 + \frac{1}{2L_\Phi} \|\nabla \Phi(x_k)\|^2 +  \frac{1}{2L_\Phi} \|G_k-\nabla \Phi(x_k)\|^2 \nonumber
 \\=& \Phi(x_k)-\frac{1}{2L_\Phi}\| \nabla \Phi(x_k)\|^2 +  \frac{1}{2L_\Phi} \|G_k-\nabla \Phi(x_k)\|^2,
 \end{align}
 which, in conjunction with the strong convexity of $\Phi(\cdot)$,  yields
 \begin{align}\label{wodialayacs}
 \Phi(z_{k+1})\leq &\Big( 1- \frac{1}{\sqrt{\kappa_x}}\Big) \Phi(z_k) +\Big( 1- \frac{1}{\sqrt{\kappa_x}}\Big) \langle \nabla\Phi(x_k),x_k-z_k\rangle + \frac{1}{\sqrt{\kappa_x}}\Phi(x_k) \nonumber
 \\&-\frac{1}{2L_\Phi}\| \nabla \Phi(x_k)\|^2 +  \frac{1}{2L_\Phi} \|G_k-\nabla \Phi(x_k)\|^2 \nonumber
 \\\overset{(i)}\leq &\Big( 1- \frac{1}{\sqrt{\kappa_x}}\Big) S_k^* +\Big( 1- \frac{1}{\sqrt{\kappa_x}}\Big) \langle \nabla\Phi(x_k),x_k-z_k\rangle + \frac{1}{\sqrt{\kappa_x}}\Phi(x_k) \nonumber
 \\&-\frac{1}{2L_\Phi}\| \nabla \Phi(x_k)\|^2 +  \frac{1}{2L_\Phi} \|G_k-\nabla \Phi(x_k)\|^2,
 \end{align}
 where $(i)$ follows because $S_k^*\geq  \Phi(z_k)$ for $k=0,...,k^\prime$. Next, based on the definition of $S_k(x)$ in \cref{ssses_seq} and taking derivative w.r.t.~$x$ on both sides of  \cref{ssses_seq}, we have 
 \begin{align}\label{gg_smidasssca}
 \nabla S_{k+1}(x) &\overset{(i)}= \Big( 1- \frac{1}{\sqrt{\kappa_x}} \Big)\nabla S_k(x) + \frac{1}{\sqrt{\kappa_x}} G_k +  \frac{\mu_x}{\sqrt{\kappa_x}}(x-x_k) \nonumber
 \\&=\mu_x\Big( 1- \frac{1}{\sqrt{\kappa_x}} \Big)(x-v_k)+\frac{1}{\sqrt{\kappa_x}} G_k +  \frac{\mu_x}{\sqrt{\kappa_x}}(x-x_k), 
 \end{align}
 where $(i)$ follows because $S_k(x) = S_k^* + \frac{\mu_x}{2}\|x-v_k\|^2$. Noting that $\nabla S_{k+1}(v_{k+1})= 0$, we obtain from  \cref{gg_smidasssca} that 
 \begin{align*}
 \mu_x\Big( 1- \frac{1}{\sqrt{\kappa_x}} \Big)(v_{k+1}-v_k)+\frac{1}{\sqrt{\kappa_x}} G_k +  \frac{\mu_x}{\sqrt{\kappa_x}}(v_{k+1}-x_k) = 0,
 \end{align*}
 which yields 
 \begin{align}\label{mamamiyasscas}
 v_{k+1} = \Big(1-\frac{1}{\sqrt{\kappa_x}} \Big)v_k + \frac{1}{\sqrt{\kappa_x}} x_k - \frac{1}{\mu_x\sqrt{\kappa_x}} G_k.
 \end{align}
 Based on \cref{ssses_seq} and using $S_k(x) = S_k^* + \frac{\mu_x}{2}\|x-v_k\|^2$,  we have 
 \begin{align*}
S_{k+1}^*  + \frac{\mu_x}{2}\|x_k-v_{k+1}\|^2 = \Big(1-\frac{1}{\sqrt{\kappa_x}} \Big)\Big(S_{k}^*  + \frac{\mu_x}{2}\|x_k-v_{k}\|^2 \Big) + \frac{1}{\sqrt{\kappa_x}}\Phi(x_k) + \frac{\epsilon}{4\sqrt{\kappa_x}},
 \end{align*}
 which, in conjunction with \cref{mamamiyasscas}, yields
 \begin{align}\label{xingyunwos}
 S_{k+1}^*  =&  \Big(1-\frac{1}{\sqrt{\kappa_x}} \Big) S_{k}^*  +  \Big(1-\frac{1}{\sqrt{\kappa_x}} \Big)  \frac{\mu_x}{2}\|x_k-v_k\|^2 + \frac{1}{\sqrt{\kappa_x}}\Phi(x_k) + \frac{\epsilon}{4\sqrt{\kappa_x}} \nonumber
 \\&- \Big(1-\frac{1}{\sqrt{\kappa_x}} \Big)^2\frac{\mu_x}{2}\|x_k-v_k\|^2 - \frac{1}{2\mu_x\kappa_x}\|G_k\|^2 +  \Big(1-\frac{1}{\sqrt{\kappa_x}} \Big)\frac{1}{\sqrt{\kappa_x}}\langle v_k-x_k,G_k \rangle \nonumber
 \\ = & \Big(1-\frac{1}{\sqrt{\kappa_x}} \Big) S_{k}^*  + \Big(1-\frac{1}{\sqrt{\kappa_x}} \Big) \frac{1}{\sqrt{\kappa_x}} \frac{\mu_x}{2}\|x_k-v_k\|^2 +\frac{1}{\sqrt{\kappa_x}}\Phi(x_k) + \frac{\epsilon}{4\sqrt{\kappa_x}}\nonumber
 \\&- \frac{1}{2\mu_x\kappa_x}\|G_k\|^2 +  \Big(1-\frac{1}{\sqrt{\kappa_x}} \Big)\frac{1}{\sqrt{\kappa_x}}\langle v_k-x_k,G_k \rangle.
 \end{align}  
 Based on the definition of $\kappa_x$, we simplify \cref{xingyunwos} to 
 \begin{align}\label{yifenyimiaos}
  S_{k+1}^* \geq & \Big(1-\frac{1}{\sqrt{\kappa_x}} \Big) S_{k}^* +\frac{1}{\sqrt{\kappa_x}}\Phi(x_k) + \frac{\epsilon}{4\sqrt{\kappa_x}}- \frac{1}{2L_\Phi}\|G_k\|^2 \nonumber
  \\&+  \Big(1-\frac{1}{\sqrt{\kappa_x}} \Big)\frac{1}{\sqrt{\kappa_x}}\langle v_k-x_k,G_k \rangle.
 \end{align}
 Next, we prove $v_k-x_k = \sqrt{\kappa_x}(x_k-z_k)$ by induction. First note that this equality holds for $k=0$ based on the fact that $v_0-x_0=\sqrt{\kappa_x}(x_0-z_0)=0$. Then, suppose that it holds for iteration $k$, and for iteration $k+1$, we obtain from \cref{mamamiyasscas} that 
 \begin{align}
  v_{k+1} -x_{k+1}= &\Big(1-\frac{1}{\sqrt{\kappa_x}} \Big)v_k + \frac{1}{\sqrt{\kappa_x}} x_k-x_{k+1} - \frac{1}{\mu_x\sqrt{\kappa_x}} G_k \nonumber
  \\ \overset{(i)}=& \Big(1-\frac{1}{\sqrt{\kappa_x}} \Big)\Big(1+\sqrt{\kappa_x} \Big) x_k -\Big(1-\frac{1}{\sqrt{\kappa_x}} \Big)\sqrt{\kappa_x}z_k +\frac{1}{\sqrt{\kappa_x}} x_k -x_{k+1} -\frac{1}{\mu_x\sqrt{\kappa_x}} G_k \nonumber
  \\=& \sqrt{\kappa_x} \Big(x_k-\frac{1}{L_\Phi}G_k\Big) -(\sqrt{\kappa_x}-1) z_k -x_{k+1} \nonumber
%  \\=&\sqrt{\kappa_x} \Big( z_{k+1}-\frac{\sqrt{\kappa_x} -1}{\sqrt{\kappa_x} }z_k\Big) -x_{k+1} \nonumber
  \\\overset{(ii)}=&\sqrt{\kappa_x} (x_{k+1}-z_{k+1}),
 \end{align}
 where $(i)$ follows because $v_k-x_k = \sqrt{\kappa_x}(x_k-z_k)$ and $(ii)$ follows from the updating step in \cref{updateRule}. Then, by induction, we have that $v_k-x_k = \sqrt{\kappa_x}(x_k-z_k)$ holds for all iterations. 
 
 \noindent Combining this equality with \cref{yifenyimiaos}, we have
 {\small\begin{align}\label{wocayoudiandaca}
 S_{k+1}^* \geq & \Big(1-\frac{1}{\sqrt{\kappa_x}} \Big) S_{k}^* +\frac{1}{\sqrt{\kappa_x}}\Phi(x_k) + \frac{\epsilon}{4\sqrt{\kappa_x}}- \frac{1}{2L_\Phi}\|G_k\|^2 +  \Big(1-\frac{1}{\sqrt{\kappa_x}} \Big)\langle x_k-z_k,G_k \rangle \nonumber
 \\=& \Big(1-\frac{1}{\sqrt{\kappa_x}} \Big) S_{k}^* +\frac{1}{\sqrt{\kappa_x}}\Phi(x_k) + \frac{\epsilon}{4\sqrt{\kappa_x}}- \frac{1}{2L_\Phi}\|\nabla\Phi(x_k)\|^2 +\Big(1-\frac{1}{\sqrt{\kappa_x}} \Big)\langle x_k-z_k,\nabla \Phi(x_k) \rangle  \nonumber
 \\& +  \Big(1-\frac{1}{\sqrt{\kappa_x}} \Big)\langle x_k-z_k,G_k-\nabla \Phi(x_k)\rangle - \frac{1}{2L_\Phi}\|G_k-\nabla\Phi(x_k)\|^2 - \frac{1}{L_\Phi}\langle G_k-\nabla \Phi(x_k), \nabla\Phi(x_k)\rangle \nonumber
\\\overset{(i)}\geq& \Big(1-\frac{1}{\sqrt{\kappa_x}} \Big) S_{k}^* +\frac{1}{\sqrt{\kappa_x}}\Phi(x_k) - \frac{1}{2L_\Phi}\|\nabla\Phi(x_k)\|^2 +\Big(1-\frac{1}{\sqrt{\kappa_x}} \Big)\langle x_k-z_k,\nabla \Phi(x_k) \rangle  + \frac{\epsilon}{4\sqrt{\kappa_x}}\nonumber
 \\& -  \Big(1-\frac{1}{\sqrt{\kappa_x}} \Big)\| x_k-z_k\|\|G_k-\nabla \Phi(x_k)\| - \frac{1}{2L_\Phi}\|G_k-\nabla\Phi(x_k)\|^2 \nonumber
 \\&-\| G_k-\nabla \Phi(x_k)\|\| x_k-x^*\|
 \end{align}}
 \hspace{-0.14cm}where $(i)$ follows from \Cref{smoothness_Phis} with  $L_{\Phi_k} \leq L_\Phi$ for $k=0,...,k^\prime$. Based on $\|z_k-x^*\|\leq\sqrt{\frac{2}{\mu_x}(\Phi(0) -\Phi(x^*))+ \|x^*\|^2+\frac{\epsilon}{\mu_x}}$ and $\|x_k-x^*\|<3\sqrt{\frac{2}{\mu_x}(\Phi(0) -\Phi(x^*))+ \|x^*\|^2+\frac{\epsilon}{\mu_x}}$ for $k=0,...,k^\prime$, and using  $\|x_k-z_k\|\leq \|z_k-x^*\|+\|x_k-x^*\|$, we obtain from \cref{wocayoudiandaca} that 
 \begin{align}\label{wocaolalascs}
 S_{k+1}^* \geq & \Big(1-\frac{1}{\sqrt{\kappa_x}} \Big) S_{k}^* +\frac{1}{\sqrt{\kappa_x}}\Phi(x_k) - \frac{1}{2L_\Phi}\|\nabla\Phi(x_k)\|^2 +\Big(1-\frac{1}{\sqrt{\kappa_x}} \Big)\langle x_k-z_k,\nabla \Phi(x_k) \rangle  \nonumber
 \\& + \frac{\epsilon}{4\sqrt{\kappa_x}}-  \Big(7-\frac{4}{\sqrt{\kappa_x}} \Big)\sqrt{\frac{2}{\mu_x}(\Phi(0) -\Phi(x^*))+ \|x^*\|^2+\frac{\epsilon}{\mu_x}}\|G_k-\nabla \Phi(x_k)\|  \nonumber
 \\&- \frac{1}{2L_\Phi}\|G_k-\nabla\Phi(x_k)\|^2.
 \end{align}
Next, we upper-bound the hypergradient estimation error $\|G_k-\nabla\Phi(x_k)\|$ in \cref{wocaolalascs}.  Based on \Cref{le:hgestr}, we have 
\begin{align*}
 \|G_k-\nabla \Phi( x_k)\| \leq &\sqrt{\frac{\widetilde L_y +\mu_y}{\mu_y}} \Big(L_y +\frac{2\widetilde L_{xy}L_y}{\mu_y} +\Big(\frac{\rho_{xy}}{\mu_y}+\frac{\widetilde L_{xy}\rho_{yy}}{\mu_y^2}\Big)\mathcal{N}_k\Big) \mathcal{M}_k \exp\Big(-\frac{N}{2\sqrt{\kappa_y}}\Big) 
 \\&+\frac{\widetilde L_{xy}}{\mu_y}\Big(\frac{\sqrt{\kappa_y}-1}{\sqrt{\kappa_y}+1}\Big)^M\mathcal{N}_k, \nonumber
%\\\leq&
\end{align*}
which, combined with $\|x_{k}-x^*\|\leq 3\sqrt{\frac{2}{\mu_x}(\Phi(0) -\Phi(x^*))+ \|x^*\|^2+\frac{\epsilon}{\mu_x}}$ for $k=0,...,k^\prime$ and the definitions of $\mathcal{M}_k,\mathcal{N}_k$ in \cref{pf:ntationscs}, yields 
\begin{align*}
 \|G_k-\nabla \Phi( x_k)\| \leq &\sqrt{\frac{\widetilde L_y +\mu_y}{\mu_y}} \Big(L_y +\frac{2\widetilde L_{xy}L_y}{\mu_y} +\Big(\frac{\rho_{xy}}{\mu_y}+\frac{\widetilde L_{xy}\rho_{yy}}{\mu_y^2}\Big)\mathcal{N}_*\Big) \mathcal{M}_* \exp\Big(-\frac{N}{2\sqrt{\kappa_y}}\Big) \nonumber
 \\&+\frac{\widetilde L_{xy}}{\mu_y}\Big(\frac{\sqrt{\kappa_y}-1}{\sqrt{\kappa_y}+1}\Big)^M\mathcal{N}_*,
\end{align*}
where the constants $\mathcal{M}_*$ and $\mathcal{N}_*$ are defined in \cref{pf:ntationscs}. We choose 
\begin{align}\label{wangdehuas}
N&=\Theta\Big(\sqrt{\kappa_y}\log \Big(\frac{\mathcal{M}_*(\mathcal{N}_*+\mu_y)}{\mu_x^{0.25}\mu_y^{2.5}\sqrt{\epsilon L_\Phi}}+\frac{\mathcal{M}_*(\mathcal{N}_*+\mu_y)\sqrt{L_\Phi}(\Phi(0) -\Phi(x^*)+\mu_x^{0.5}\|x^*\|+\epsilon)}{\mu_x\mu_y^{2.5}\epsilon}\Big)\Big), \nonumber
\\M&=\Theta\Big(\sqrt{\kappa_y}\log \Big(\frac{\mathcal{N}_*}{\mu_x^{0.25}\mu_y\sqrt{\epsilon L_\Phi}}+\frac{\mathcal{N}_*\sqrt{L_\Phi}(\Phi(0) -\Phi(x^*)+\mu_x^{0.5}\|x^*\|+\epsilon)}{\mu_x\mu_y\epsilon}\Big)\Big).
\end{align} 
In other words, $M$ and $N$ scale linearly with $\sqrt{\kappa_y}$ and depend only logarithmically on other constants such as $\mu_x,\mu_y,\|x^*\|,\|y^*(x^*)\|$, $\Phi(0) -\Phi(x^*)$ and $\epsilon$. Then, we have 
\begin{align*}
&\|G_k-\nabla \Phi( x_k)\| \leq \frac{\sqrt{\epsilon L_\Phi}}{2\sqrt{2}\kappa_x^{1/4}}
\\&\big(7-\frac{4}{\sqrt{\kappa_x}} \big)\sqrt{\frac{2}{\mu_x}(\Phi(0) -\Phi(x^*))+ \|x^*\|^2+\frac{\epsilon}{\mu_x}}\|G_k-\nabla \Phi( x_k)\| \leq \frac{\epsilon}{8\sqrt{\kappa_x}}. 
\end{align*}
  Substituting these two inequalities into \cref{wocaolalascs} yields, for any $k=0,...,k^\prime$,
\begin{align}\label{jigyiscasesc}
 S_{k+1}^* \geq & \Big(1-\frac{1}{\sqrt{\kappa_x}} \Big) S_{k}^* +\frac{1}{\sqrt{\kappa_x}}\Phi(x_k) - \frac{1}{2L_\Phi}\|\nabla\Phi(x_k)\|^2 \nonumber
 \\&+\Big(1-\frac{1}{\sqrt{\kappa_x}} \Big)\langle x_k-z_k,\nabla \Phi(x_k) \rangle  + \frac{\epsilon}{16\sqrt{\kappa_x}} \nonumber
 \\\overset{(i)}\geq &\Phi(z_{k+1}),
\end{align}
where $(i)$ follows from $\|G_k-\nabla \Phi( x_k)\| \leq \frac{\sqrt{\epsilon L_\Phi}}{2\sqrt{2}\kappa_x^{1/4}}$ in \cref{wodialayacs}, 
%we have, for any $k=0,...,k^\prime$
%\begin{align*}
%\Phi(z_{k+1})\leq &\Big( 1- \frac{1}{\sqrt{\kappa_x}}\Big) S_k^* +\Big( 1- \frac{1}{\sqrt{\kappa_x}}\Big) \langle \nabla\Phi(x_k),x_k-z_k\rangle + \frac{1}{\sqrt{\kappa_x}}\Phi(x_k) \nonumber
% \\&-\frac{1}{2L_\Phi}\| \nabla \Phi(x_k)\|^2 + \frac{\epsilon}{16\sqrt{\kappa_x}},
%\end{align*}
which, by induction, finishes the proof of the second item \cref{dasimaomaoss}. To prove the first item \cref{pocasca}, letting $x=x^*$ in \cref{ssses_seq} yields, for $x=0,...,k^\prime$, 
\begin{align}\label{dalaoqiuings}
S_{k+1}(x^*) = & \Big(1 -\frac{1}{\sqrt{\kappa_x}} \Big)S_{k}(x^*)  +\frac{1}{\sqrt{\kappa_x}} \Big( \Phi(x_k) +\langle \nabla\Phi(x_k),x^*-x_k\rangle + \frac{\mu_x}{2}\|x^* - x_k\|^2 + \frac{ \epsilon}{4}\Big) \nonumber
\\&+\frac{1}{\sqrt{\kappa_x}}\langle G_k-\nabla\Phi(x_k),x^*-x_k\rangle \nonumber
\\\leq&  \Big(1 -\frac{1}{\sqrt{\kappa_x}} \Big)S_{k}(x^*) +\frac{1}{\sqrt{\kappa_x}} \Phi(x^*) +\frac{\epsilon}{4\sqrt{\kappa_x}}  + \frac{1}{\sqrt{\kappa_x}} \|x_k-x^*\|\|G_k-\nabla\Phi(x_k)\|\nonumber
\\\overset{(i)}\leq & \Big(1 -\frac{1}{\sqrt{\kappa_x}} \Big)S_{k}(x^*) +\frac{1}{\sqrt{\kappa_x}} \Phi(x^*) +\frac{\epsilon}{2\sqrt{\kappa_x}},  
\end{align} 
where $(i)$ follows because $\|x_k-x^*\|\|G_k-\nabla\Phi(x_k)\|\leq  \frac{\epsilon}{8\sqrt{\kappa_x}}/(7-\frac{4}{\sqrt{\kappa_x}})<\frac{\epsilon}{24\sqrt{\kappa_x}}<\frac{\epsilon}{4}$. Subtracting both sides of \cref{dalaoqiuings} by $\Phi(x^*)$ yields, for all $k=0,...,k^\prime$,
\begin{align}\label{aihenqingchous}
S_{k+1}(x^*) - \Phi(x^*) \leq \Big(1 -\frac{1}{\sqrt{\kappa_x}} \Big)(S_{k}(x^*)-\Phi(x^*)) +\frac{\epsilon}{2\sqrt{\kappa_x}}.
\end{align}
Telescoping \cref{aihenqingchous} over $k$ from $0$ to $k^\prime$ and using $S_0(x^*) = \Phi(0)+\frac{\mu_x}{2} \|x^*\|^2$, we have 
\begin{align*}
S_{k^\prime+1}(x^*) - \Phi(x^*) &\leq  \Big(1 -\frac{1}{\sqrt{\kappa_x}} \Big)^{k^\prime+1}( \Phi(0) -\Phi(x^*)+\frac{\mu_x}{2} \|x^*\|^2)+\frac{\epsilon}{2} \nonumber
\\& \leq \Phi(0) -\Phi(x^*)+\frac{\mu_x}{2} \|x^*\|^2+\frac{\epsilon}{2},
\end{align*} 
which, in conjunction with $S_{k^\prime+1}(x^*)\geq S_{k^\prime+1}^*\geq \Phi(z_{k^\prime+1})$ and $\Phi(z_{k^\prime+1})-\Phi(x^*)\geq \frac{\mu_x}{2}\|z_{k^\prime+1}-x^*\|^2$, yields
\begin{align*}
\|z^{k^\prime+1}-x^*\|\leq \sqrt{\frac{2}{\mu_x}\Phi(0) -\Phi(x^*)+\|x^*\|^2+\frac{\epsilon}{\mu_x}}.
\end{align*}
Then, by induction, we finish the proof of the first item \cref{pocasca}.  Therefore, based on  \cref{pocasca} and \cref{dasimaomaoss} and using an approach similar to \cref{aihenqingchous}, we have
\begin{align}\label{heiyeibaizhoussc}
\Phi(z_K)- \Phi(x^*)\leq S_{K}(x^*) - \Phi(x^*) \leq \Big(1 -\frac{1}{\sqrt{\kappa_x}} \Big)^{K}(\Phi(0) -\Phi(x^*)+\frac{\mu_x}{2} \|x^*\|^2) +\frac{\epsilon}{2}.
\end{align}
In order to achieve $\Phi(z_K)- \Phi(x^*)\leq S_{K}(x^*) - \Phi(x^*) \leq \epsilon$, it requires at most 
\begin{align}\label{anqilababascs}
K &\leq \mathcal{O}\Big( \sqrt{\frac{L_\Phi}{\mu_x}}\log\Big(\frac{\Phi(0) -\Phi(x^*)+\frac{\mu_x}{2} \|x^*\|^2}{\epsilon}\Big)\Big) \nonumber
\\\leq& \mathcal{\widetilde O}\Big(\frac{1}{\mu_x^{0.5}\mu_y}+\Big(\frac{ \sqrt{\rho_{yy}}}{\mu_x^{0.5}\mu_y^{1.5}} +  \frac{\sqrt{\rho_{xy}}}{\mu_x^{0.5}\mu_y}\Big)\sqrt{\|\nabla_y f( x^*,y^*(x^*))\|+\frac{\|x^*\|}{\mu_y}+\frac{\sqrt{\Phi(0)-\Phi(x^*)}}{\sqrt{\mu_x}\mu_y}}\Big). 
\end{align}
Following from the choice of $M=N=\Theta(\sqrt{\kappa_y})$, the complexity of \Cref{alg:bioNoBG} is given by 
\begin{align*}
&\mathcal{C}_{\text{\normalfont fun}}(\mathcal{A},\epsilon) \leq \mathcal{O}(n_J+n_H + n_G)\leq \mathcal{O}(K+KM+KN)   \nonumber
\\&\leq \mathcal{\widetilde O}\Big(\frac{\widetilde L_y^{0.5}}{\mu_x^{0.5}\mu_y^{1.5}}+\big(\frac{ (\rho_{yy}\widetilde L_y)^{0.5}}{\mu_x^{0.5}\mu_y^{2}} +  \frac{(\rho_{xy}\widetilde L_y)^{0.5}}{\mu_x^{0.5}\mu_y^{1.5}}\big)\sqrt{\|\nabla_y f( x^*,y^*(x^*))\|+\frac{\|x^*\|}{\mu_y}+\frac{\sqrt{\Phi(0)-\Phi(x^*)}}{\sqrt{\mu_x}\mu_y}}\Big)
\end{align*}
which finishes the proof. 
\end{proof}

\section{Proof of \Cref{coro:quadaticSr}}\label{proof:coUsrwB}
The proof follows a procedure similar to that for \Cref{upper_srsr_withnoB} except that the smoothness parameter of $\Phi(\cdot)$ at iterate $x_k$  and the hypergradient estimation error $\|G_k-\nabla\Phi(x_k)\|$ are different. In specific, for the quadratic inner problem, we have that $\nabla_y^2 g (x,y) \equiv H, \nabla_x\nabla_y g(x,y) \equiv J, \forall  x\in\mathbb{R}^p, y\in \mathbb{R}^q$. Then, based on the form of $\nabla \Phi(x)$ in \cref{hgforms}, we have 
\begin{align*}
\|\nabla\Phi(x_1&) - \nabla\Phi(x_2)\| \nonumber
\\\leq& \|\nabla_x f(x_1,y^*(x_1)) -\nabla_x f(x_2,y^*(x_2))  \| \nonumber
\\&+ \|JH^{-1}\nabla_y f(x_1,y^*(x_1)) -JH^{-1}\nabla_y f(x_2,y^*(x_2))\| \nonumber
\\  \leq & L_x\|x_1-x_2\| + L_{xy}\|y^*(x_1)-y^*(x_2)\| + \frac{\widetilde L_{xy}}{\mu_y}(L_{xy}\|x_1-x_2\|+L_y\|y^*(x_1)-y^*(x_2)\|)
\end{align*}
which, in conjunction with $\|y^*(x_1)-y^*(x_2)\|\leq \frac{\widetilde L_{xy}}{\mu_y} \|x_1-x_2\|$, yields 
\begin{align}\label{co:qiguqiaopwc}
\|\nabla\Phi(x_1) - \nabla\Phi(x_2)\| \leq \Big(\underbrace{ L_x + \frac{2\widetilde L_{xy}L_{xy}}{\mu_y} +\frac{L_y\widetilde L_{xy}^2}{\mu_y^2} }_{L_\Phi} \Big) \|x_1-x_2\|.
\end{align}
Note that \cref{co:qiguqiaopwc} shows that the objective function $\Phi(\cdot)$ is globally smooth, i.e., the smoothness parameter is bounded at all $x\in\mathbb{R}^p$. This is different from the proof in \Cref{upper_srsr_withnoB}, where the smoothness parameter is unbounded over $x\in\mathbb{R}^p$, but can be bounded at all iterates $x_k,k=0,...,K$ along the optimization path of the algorithm. Therefore, the proof for such a quadratic special case is simpler. 

We next upper-bound the hypergradient estimation error $\|G_k-\nabla \Phi(x_k)\|$. Using an approach similar to \cref{jingyikeai}, we have 
\begin{align}\label{co:huxiuwancong}
\|G_k-\nabla& \Phi(x_k)\| \nonumber
\\\leq & L_y \|y^*(x_k)-y_k^N\| + \widetilde L_{xy} \|v_k^M- H^{-1}\nabla_y f( x_k,y^N_k) \| \nonumber
\\&+ \widetilde L_{xy}\big\|H^{-1}\nabla_y f( x_k,y^N_k)-H^{-1}\nabla_y f(x_k,y^*(x_k)) \big\|\nonumber
\\\leq & \Big(L_y +\frac{\widetilde L_{xy}L_y}{\mu_y} \Big)\|y_k^N-y^*(x_k)\| + \widetilde L_{xy} \|v_k^M- H^{-1}\nabla_y f( x_k,y^N_k) \|  \nonumber
\\\leq& \Big(L_y +\frac{\widetilde L_{xy}L_y}{\mu_y} \Big)\|y_k^N-y^*(x_k)\|+\frac{\widetilde L_{xy}}{\mu_y}\left(\frac{\sqrt{\kappa_y}-1}{\sqrt{\kappa_y}+1}\right)^M\|\nabla_y f( x_k,y^*(x_k))\| \nonumber
\\\leq &\sqrt{\frac{\widetilde L_y +\mu_y}{\mu_y}} \Big(L_y +\frac{\widetilde L_{xy}L_y}{\mu_y}\Big) \mathcal{M}_* \exp\Big(-\frac{N}{2\sqrt{\kappa_y}}\Big) +\frac{\widetilde L_{xy}}{\mu_y}\Big(\frac{\sqrt{\kappa_y}-1}{\sqrt{\kappa_y}+1}\Big)^M\mathcal{N}_*,
\end{align}
where $\mathcal{M}_*$ and $\mathcal{N}_*$ are given by \cref{pf:ntationscs}. 
Based on \cref{co:qiguqiaopwc} and \cref{co:huxiuwancong}, we choose
\begin{itemize}
\item $N=\Theta(\sqrt{\kappa_y}\log (\frac{\mathcal{M}_*}{\mu_x^{0.25}\mu_y^{1.5}\sqrt{\epsilon L_\Phi}}+\frac{\mathcal{M}_*\sqrt{L_\Phi}(\Phi(0) -\Phi(x^*)+\mu_x^{0.5}\|x^*\|+\epsilon)}{\mu_x\mu_y^{1.5}\epsilon}))$
\item $M=\Theta(\sqrt{\kappa_y}\log (\frac{\mathcal{N}_*}{\mu_x^{0.25}\mu_y\sqrt{\epsilon L_\Phi}}+\frac{\mathcal{N}_*\sqrt{L_\Phi}(\Phi(0) -\Phi(x^*)+\mu_x^{0.5}\|x^*\|+\epsilon)}{\mu_x\mu_y\epsilon})).$
\end{itemize} 
Then, using an approach similar to \cref{heiyeibaizhoussc} with $\rho_{xy}=\rho_{yy}=0$, we have 
\begin{align}
\Phi(z_K)- \Phi(x^*) \leq \Big(1 -\sqrt{\frac{\mu_x}{L_\Phi}} \Big)^{K}(\Phi(0) -\Phi(x^*)+\frac{\mu_x}{2} \|x^*\|^2) +\frac{\epsilon}{2},
\end{align}
where $L_\Phi$ is given in \cref{co:qiguqiaopwc}. Then, in order to achieve $\Phi(z_K)- \Phi(x^*) \leq \epsilon$, it requires at most 
\begin{align*}
\mathcal{C}_{\text{\normalfont fun}}(\mathcal{A},\epsilon) \leq &\mathcal{O}(n_J+n_H + n_G)\leq \mathcal{O}(K+KM+KN)   \nonumber
\\\leq& \mathcal{O}\Big(\sqrt{\frac{\widetilde L_y}{\mu_x\mu_y^3}}\log\, {\small \text{poly}(\mu_x,\mu_y,\|x^*\|,\Phi(0)-\Phi(x^*),\|\nabla_y f(x^*,y^*(x^*))\|)}\Big),
\end{align*}
which finishes the proof. 

%pf:ntationscs
%Then, using an approach similar to \cref{mxdhuaqingduo}, we have 
% \begin{align}\label{yuzhonghemengjiu}
% \Phi(z_{k+1})\leq \Phi(x_k)-\frac{1}{2L_\Phi}\| \nabla \Phi(x_k)\|^2 +  \frac{1}{2L_\Phi} \|G_k-\nabla \Phi(x_k)\|^2,
% \end{align}

\section{Proof of \Cref{th:upper_csc1sc}}
%$\Phi_\epsilon(\cdot)=f_\epsilon(x,y^*(x))$ by adding a quadratic regularizer to the upper-level function $f(x,y)$, i.e., 
%%\begin{align}
%$f_\epsilon(x,y) = f(x,y) +\frac{\epsilon}{2B} \|x\|^2$.

Recall that $\widetilde \Phi(\cdot)=\widetilde f(x,y^*(x))$ with  $\widetilde f(x,y) = f(x,y) +\frac{\epsilon}{2R} \|x\|^2$. Then, we have $\widetilde\Phi(x) = \Phi(x) +\frac{\epsilon}{2R} \|x\|^2$ is strongly-convex with parameter $\mu_x=\frac{\epsilon}{R}$. Note that the smoothness parameters of $\widetilde f(x,y)$ are the same as those of $f(x,y)$ except that $L_x$ in \cref{def:first} becomes $L_x+ \frac{\epsilon}{R}$ for $\widetilde f(x,y)$. %\end{align}
%where $D$ is the upper bound of $\|x^*\|$. 	
%By \Cref{de:pc}, we have $x_\epsilon^*\in\mathcal{X}$, where $x_\epsilon^*=\argmin_x \Phi_\epsilon(x)$. 
%Then, apply the results in \Cref{upper_srsr_withnoB} to $\Phi_\epsilon(x)$ with $\mu_x=\frac{\epsilon}{B}$ and 
Let $x^*\in\argmin_{x\in\mathbb{R}^p}\Phi(x)$ be one minimizer of the original objective function $\Phi(\cdot)$ and let $\widetilde x^*=\argmin_{x\in\mathbb{R}^p}\widetilde\Phi(x)$ be the minimizer of the regularized objective function $\widetilde \Phi(\cdot)$.  We next characterize some useful inequalities between $x^*$ and $\widetilde x^*$. Based on the definition of $x^*$ and $\widetilde x^*$, we have $\nabla \widetilde \Phi(\widetilde x^*) = 0$ and $\nabla\widetilde\Phi( x^*) = \nabla\Phi(x^*) + \frac{\epsilon}{R} x^* = \frac{\epsilon}{R} x^*  $, which, combined with the strong convexity of $\widetilde \Phi(\cdot)$,  implies that $\frac{\epsilon}{R}\|x^*-\widetilde x^*\|\leq\|\nabla \widetilde \Phi(\widetilde x^*) -\nabla\widetilde\Phi( x^*) \|=  \frac{\epsilon}{R} \|x^*\|$ and hence $\|\widetilde x^*\|\leq2\|x^*\|$. Similarly, the following  (in)equalities hold:
\begin{align}\label{ggsmidacposcs1}
\|y^*(\widetilde x^*)\|&\leq \|y^*(x^*)\|+\frac{3\widetilde L_{xy}}{\mu_y}\|x^*\|, \nonumber
\\\|\nabla_y \widetilde f(\widetilde x^*,y^*(\widetilde x^*))\|&\leq \|\nabla_y f(\widetilde x^*,y^*(\widetilde x^*))\| + \frac{\epsilon}{R}\|\widetilde x^*\| \nonumber
\\&\leq\|\nabla_y f(x^*,y^*(x^*))\| + \Big( 3L_{xy}+\frac{3L_y\widetilde L_{xy}}{\mu_y}+\frac{2\epsilon}{R}\Big)\|x^*\|  \nonumber
\\\widetilde \Phi(0) -\widetilde \Phi(\widetilde x^*) &= \Phi(0)-\Phi(\widetilde x^*) - \frac{\epsilon}{2R}\|\widetilde x^*\|^2 \overset{(i)}\leq \Phi(0) -\Phi(x^*),
\end{align}
where $(i)$ follows from the definition of $x^*\in\argmin_x\Phi(x)$.  

Let $L_{\widetilde \Phi}$ be one smoothness parameter of the function $\widetilde \Phi(\cdot)$, which takes the same form as $L_\Phi$ in \cref{wodishenamxd} except that $L_x,f,x^*$ and $\Phi$ become $L_x+\frac{\epsilon}{R},\widetilde f,\widetilde x^*$ and $\widetilde \Phi$ in \cref{wodishenamxd}, respectively. Similarly to \cref{wangdehuas}, we choose  
\begin{align}\label{michiganletgo}
N=&\Theta(\sqrt{\kappa_y}\log (\text{poly}(\epsilon,\mu_x,\mu_y,\|\widetilde x^*\|,\|y^*(\widetilde x^*)\|,\|\nabla_y \widetilde f(\widetilde x^*,y^*(\widetilde x^*))\|,\widetilde \Phi(0) -\widetilde \Phi(\widetilde x^*)))), \nonumber
\\M=&\Theta(\sqrt{\kappa_y}\log (\text{poly}(\epsilon,\mu_x,\mu_y,\|\widetilde x^*\|,\|y^*(\widetilde x^*)\|,\|\nabla_y \widetilde f(\widetilde x^*,y^*(\widetilde x^*))\|,\widetilde \Phi(0) -\widetilde \Phi(\widetilde x^*) ))).
\end{align} 

We first prove the case when the convergence is measured in term of the suboptimality gap. Note that in this case we choose $R=B^2$. Using an approach similar to \cref{heiyeibaizhoussc} in the proof of \Cref{upper_srsr_withnoB} with $\epsilon$ and $\mu_x$ being replaced by $\epsilon/2$ and  $\frac{\epsilon}{B^2}$, respectively, we have %$\Phi_\epsilon(z_K)- \Phi_\epsilon(x^*)\leq \epsilon$. 
\begin{align*}
\widetilde\Phi(z_K)- \widetilde\Phi(\widetilde x^*) \leq \Big(1 -\sqrt{\frac{\epsilon}{B^2L_{\widetilde \Phi}}} \Big)^{K}(\widetilde\Phi(0) -\widetilde\Phi(\widetilde x^*)+\frac{\epsilon}{2B^2} \|\widetilde x^*\|^2) +\frac{\epsilon}{4},
\end{align*}
which, in conjunction with $\widetilde\Phi(z_K)\geq\Phi(z_K)$ and $\widetilde\Phi(\widetilde x^*)\leq\widetilde\Phi(x^*)=\Phi(x^*)+\frac{\epsilon}{2B^2}\|x^*\|^2$, yields
\begin{align}\label{ca:youdiannaoketongs}
\Phi(z_K) - \Phi(x^*)\leq \Big(1 -\sqrt{\frac{\epsilon}{B^2L_{\widetilde \Phi}}} \Big)^{K}(\widetilde\Phi(0) -\widetilde\Phi(\widetilde x^*)+\frac{\epsilon}{2B^2} \|\widetilde x^*\|^2) +\frac{\epsilon}{4}+\frac{\epsilon}{2B^2}\|x^*\|^2.
\end{align}
Recall $\|x^*\|=B$. Similarly to \cref{anqilababascs}, we choose
{\footnotesize
\begin{align}\label{Kwocayoudianda}
K =& \Theta\Big( \sqrt{\frac{B^2L_{\widetilde \Phi}}{\epsilon}}\log\Big(\frac{\widetilde\Phi(0) -\widetilde\Phi(\widetilde x^*)+\frac{\epsilon}{2B^2} \|\widetilde x^*\|^2}{\epsilon}\Big)\Big) \nonumber
\\=&\widetilde \Theta \Big(\sqrt{\frac{B^2}{\epsilon\mu_y^2}}+\Big(\sqrt{\frac{B^2\rho_{yy}}{\epsilon\mu_y^{3}}} +  \sqrt{\frac{B^2\rho_{xy}}{\epsilon\mu_y^2}}\Big)\sqrt{\|\nabla_y \widetilde f(\widetilde x^*,y^*(\widetilde x^*))\|+\frac{\|\widetilde x^*\|}{\mu_y}+\frac{\sqrt{B^2(\widetilde\Phi(0)-\widetilde\Phi(\widetilde x^*))}}{\sqrt{\epsilon}\mu_y}}\Big).
\end{align}}
\hspace{-0.17cm}Then,  we obtain from \cref{ca:youdiannaoketongs} that $\Phi(z_K) - \Phi(x^*)\leq \epsilon$, and the complexity $\mathcal{C}_{\text{\normalfont fun}}(\mathcal{A},\epsilon)$ after substituting \cref{ggsmidacposcs1} into \cref{michiganletgo} and \cref{Kwocayoudianda} is given by 
\begin{align}
\mathcal{C}_{\text{\normalfont fun}}&(\mathcal{A},\epsilon) \leq \mathcal{O}(n_J+n_H + n_G)\leq \mathcal{O}(K+KM+KN)   \nonumber
\\\leq& \mathcal{O}\Big( \Big( \sqrt{\frac{B^2\widetilde L_y}{\epsilon\mu_y^3}}+\Big(\sqrt{\frac{B^2\rho_{yy}\widetilde L_y}{\epsilon\mu_y^{4}}} +  \sqrt{\frac{B^2\rho_{xy}\widetilde L_y}{\epsilon\mu_y^3}}\Big)\sqrt{\Delta^*_{\text{\normalfont\tiny CSC}}}\Big)\log\, {\small \text{poly}(\epsilon,\mu_x,\mu_y,\Delta^*_{\text{\normalfont\tiny CSC}})}\Big).
\end{align}
%where $\Delta^*_{\text{\normalfont\tiny CSC}} = \|\nabla_y f( x^*,y^*(x^*))\|+\frac{\|x^*\|}{\mu_y}+\frac{\sqrt{B^2(\Phi(0)-\Phi(x^*))}}{\sqrt{\epsilon}\mu_y}$. 
 Next, we characterize the convergence rate and complexity under the gradient norm metric.  Note that in this case we choose $R=B$. 
Using eq. (9.14) in \cite{boyd2004convex}, we have $\|\nabla\widetilde \Phi (z_k)\|^2\leq 2L_{\widetilde \Phi}(\widetilde\Phi(z_K)- \widetilde\Phi(\widetilde x^*)) $, which, combined with  
%\begin{align*}
$\|\nabla\widetilde \Phi (z_k)\|^2 \geq \frac{1}{2}\|\nabla\Phi (z_k)\|^2 - \frac{\epsilon^2}{B^2}\|z_k\|^2\geq  \frac{1}{2}\|\nabla\Phi (z_k)\|^2 -  \frac{\epsilon^2}{B^2}(2\|z_k-\widetilde x^*\|^2 + 2\|\widetilde x^*\|^2)$
%\end{align*}
yields
\begin{align}\label{haofantoutongsc}
\|\nabla \Phi (z_k)\|^2\leq& 4L_{\widetilde \Phi}(\widetilde\Phi(z_K)- \widetilde\Phi(\widetilde x^*)) +\frac{4\epsilon^2}{B^2}\|z_k-\widetilde x^*\|^2 +\frac{4\epsilon^2}{B^2} \|\widetilde x^*\|^2 \nonumber
\\\overset{(i)}\leq&4L_{\widetilde \Phi}(\widetilde\Phi(z_K)- \widetilde\Phi(\widetilde x^*)) + \frac{8\epsilon}{B}(\widetilde\Phi(z_K)- \widetilde\Phi(\widetilde x^*)) + \frac{16\epsilon^2}{B^2} \| x^*\|^2  \nonumber
\\=& \Big(4L_{\widetilde \Phi}+ \frac{8\epsilon}{B}\Big)(\widetilde\Phi(z_K)- \widetilde\Phi(\widetilde x^*)) + \frac{16\epsilon^2}{B^2} \| x^*\|^2,  
\end{align}
where $(i)$ follows from the strong convexity of $\widetilde \Phi(\cdot)$ and $\|\widetilde x^*\|\leq 2\|x^*\|$, and $L_{\widetilde \Phi}$ takes the same form as $L_\Phi$ in \cref{wodishenamxd} except that $L_x,f,x^*$ and $\Phi$ become $L_x+\frac{\epsilon}{B},\widetilde f,\widetilde x^*$ and $\widetilde \Phi$ in \cref{wodishenamxd}, respectively.  Then, using an approach similar to \cref{heiyeibaizhoussc} in the proof of \Cref{upper_srsr_withnoB} with $\epsilon$ and $\mu_x$ being replaced by $\epsilon^2/(4L_{\widetilde \Phi}+ \frac{8\epsilon}{B})$ and $\frac{\epsilon}{B}$, respectively, we have
\begin{align*}
\widetilde\Phi(z_K)- \widetilde\Phi(\widetilde x^*) \leq \Big(1 -\sqrt{\frac{\epsilon}{BL_{\widetilde \Phi}}} \Big)^{K}(\widetilde\Phi(0) -\widetilde\Phi(\widetilde x^*)+\frac{\epsilon}{2B} \|\widetilde x^*\|^2) +\frac{\epsilon^2}{2(4L_{\widetilde \Phi}+ \frac{8\epsilon}{B})},
\end{align*}
which, in conjunction with \cref{wangdehuas} and \cref{haofantoutongsc}, yields 
\begin{align*}
\|\nabla \Phi (z_k)\|^2\leq \Big(1 -\sqrt{\frac{\epsilon}{BL_{\widetilde \Phi}}} \Big)^{K}\Big(\widetilde\Phi(0) -\widetilde\Phi(\widetilde x^*)+\frac{\epsilon}{2B} \|\widetilde x^*\|^2\Big) \Big(4L_{\widetilde \Phi}+ \frac{8\epsilon}{B}\Big)  + \frac{\epsilon^2}{2} + \frac{16\epsilon^2}{B^2} \| x^*\|^2.
\end{align*}
Note that $\|x^*\|=B$. Then, to achieve $\|\nabla \Phi (z_k)\|\leq 5\epsilon$, it suffices to  choose $M,N$ as in \cref{michiganletgo} by replacing $\epsilon$ with $\epsilon^2/(4L_{\widetilde \Phi}+ \frac{8\epsilon}{B})$, and choose
\begin{align*}
K =& \Theta\Big( \sqrt{\frac{BL_{\widetilde \Phi}}{\epsilon}}\log\Big(\frac{(\widetilde\Phi(0) -\widetilde\Phi(\widetilde x^*)+\frac{\epsilon}{2B} \|\widetilde x^*\|^2)(4L_{\widetilde \Phi}+ \frac{8\epsilon}{B})}{\epsilon}\Big)\Big).
%\\=&\widetilde \Theta \Big(\sqrt{\frac{B}{\epsilon^\prime\mu_y^2}}+\Big(\sqrt{\frac{B\rho_{yy}}{\epsilon^\prime\mu_y^{3}}} +  \sqrt{\frac{B\rho_{xy}}{\epsilon^\prime\mu_y^2}}\Big)\sqrt{\|\nabla_y \widetilde f(\widetilde x^*,y^*(\widetilde x^*))\|+\frac{\|\widetilde x^*\|}{\mu_y}+\frac{\sqrt{B(\widetilde\Phi(0)-\widetilde\Phi(\widetilde x^*))}}{\sqrt{\epsilon^\prime}\mu_y}}\Big).
\end{align*}
This in conjunction with \cref{ggsmidacposcs1} yields
\begin{align*}
\mathcal{C}_{\text{\normalfont grad}}(\mathcal{A},\epsilon) \leq &\mathcal{O}(n_J+n_H + n_G)\leq \mathcal{O}(K+KM+KN)   \nonumber
\\\leq& \mathcal{O}\Big( \Big( \sqrt{\frac{B\widetilde L_y}{\epsilon\mu_y^3}}+\Big(\sqrt{\frac{B\rho_{yy}\widetilde L_y}{\epsilon\mu_y^{4}}} +  \sqrt{\frac{B\rho_{xy}\widetilde L_y}{\epsilon\mu_y^3}}\Big)\sqrt{\Delta^*_{\text{\normalfont\tiny CSC}}}\Big)\log\, {\small \text{poly}(\epsilon,\mu_x,\mu_y,\Delta^*_{\text{\normalfont\tiny CSC}})}\Big),
\end{align*}
which finishes the proof. 
\section{Proof of \Cref{coro:quadaticConv}}
Note that for the quadratic inner problem, the Jacobians $\nabla_x\nabla_y g(x,y)$ and Hessians $\nabla_y^2 g(x,y)$ are {\bf constant} matrices, which imply that the parameters $\rho_{xx}=\rho_{xy}=0$ in Assumption \ref{g:hessiansJaco}. Then, letting $\rho_{xx}=\rho_{xy}=0$ in the results of \Cref{th:upper_csc1sc} finishes the proof.

\section{Proof of \Cref{upper_srsr}}\label{proof:upss_wb}
Based on the update in line 9 of \Cref{alg:bio}, we  have,  for any $x\in\mathbb{R}^p$
\begin{align}\label{e1:ggmida}
\langle \beta_k G_k,x_{k+1}-x\rangle = \tau_k\underbrace{\langle x-x_{k+1}, x_{k+1} -\widetilde x_k\rangle}_{P} + (1-\tau_k)\underbrace{\langle x-x_{k+1},x_{k+1}-x_k\rangle}_{Q}.
\end{align}
Note that $P$ in the above \cref{e1:ggmida} satisfies 
\begin{align*}
P &= \langle \widetilde x_k - x_{k+1}, x -\widetilde x_k \rangle + \|x-\widetilde x_k\|^2 - \|x-x_{k+1}\|^2
\\&=-P + \|x-\widetilde x_k\|^2 -\|\widetilde x_k-x_{k+1}\|^2 -\|x-x_{k+1}\|^2,
\end{align*}
which yields $P =\frac{1}{2}(\|x-\widetilde x_k\|^2 -\|\widetilde x_k-x_{k+1}\|^2 -\|x-x_{k+1}\|^2)$. Taking an approach similar to the derivation of $P$, we can obtain $Q=\frac{1}{2}(\|x- x_k\|^2 -\| x-x_{k+1}\|^2 -\|x_{k}-x_{k+1}\|^2)$. Then,  substituting  the forms of $P,Q$ to \cref{e1:ggmida} and using the choices of $\tau_k$ and $\beta_k$, we have 
\begin{align}\label{eq:gdangle}
\big\langle G_k, \frac{\sqrt{\alpha\mu_x}}{2}(x_{k+1}-x)\big\rangle &=  \frac{\sqrt{\alpha\mu_x}\mu_x}{8} (\|x-\widetilde x_k\|^2 -\|\widetilde x_k-x_{k+1}\|^2 -\|x-x_{k+1}\|^2)  \nonumber
\\+&\frac{2\mu_x-\sqrt{\alpha\mu_x}\mu_x}{8}(\|x- x_k\|^2 -\| x-x_{k+1}\|^2 -\|x_{k}-x_{k+1}\|^2).
\end{align}
Based on the update $z_{k+1}= \widetilde x_k -\alpha_k G_k $ and the choice of $\alpha_k=\alpha$, we have, for any $x^\prime\in\mathbb{R}^p$,
\begin{align}\label{eq:changeone}
\langle z_{k+1}-x^\prime, G_k \rangle  =& \frac{1}{\alpha}\langle x^\prime-z_{k+1}, z_{k+1}-\widetilde x_k \rangle \nonumber
\\ = & \frac{1}{2\alpha} (\|x^\prime-\widetilde x_k\| - \|x^\prime-z_{k+1}\|^2 - \|z_{k+1}-\widetilde x_k\|^2).
\end{align}
Let  $x^{\prime} = (1-\frac{\sqrt{\alpha\mu_x}}{2})z_k + \frac{\sqrt{\alpha \mu_x}}{2}$ and recall $\widetilde x_k = \eta_kx_k + (1-\eta_k)z_k$. Then, we have
%in \cref{eq:changeone} yields
\begin{align}\label{eq:xprimes}
\|x^\prime - \widetilde x_k\|^2 = & \Big\|\frac{\sqrt{\alpha\mu_x}}{2}(x_{k+1}-z_k)+ \frac{\sqrt{\alpha\mu_x}}{\sqrt{\alpha\mu_x}+2}(z_k-x_k)\Big\|^2  \nonumber
\\ =& \Big\| \frac{\sqrt{\alpha\mu_x}}{2}(x_{k+1}-x_k)+ \frac{\alpha\mu_x}{2(\sqrt{\alpha\mu_x}+2)}(z_k-x_k)  \Big\|^2 \nonumber
\\ \overset{(i)}= & \frac{\alpha\mu_x}{4} \Big\|(1-\frac{\sqrt{\alpha\mu_x}}{2})(x_{k+1}-x_k) + \frac{\sqrt{\alpha\mu_x}}{2} (x_{k+1}-\widetilde x_k)\Big\|^2 \nonumber
\\ \leq &  \frac{\alpha\mu_x}{4} \big(1-\frac{\sqrt{\alpha\mu_x}}{2}\big) \|x_{k+1}-x_k\|^2 +   \frac{\alpha\mu_x\sqrt{\alpha\mu_x}}{8}\|x_{k+1}-\widetilde x_k\|^2,
\end{align}
where $(i)$ follows because $\widetilde x_k - x_k=\frac{2}{2+\sqrt{\alpha\mu_x} }(z_k-x_k)$.  Then, substituting \cref{eq:xprimes} into \cref{eq:changeone}, adding \cref{eq:gdangle} and \cref{eq:changeone},  and cancelling out several negative terms, we have
\begin{align}\label{eq:gkquick}
\big \langle G_k, &\frac{\sqrt{\alpha\mu_x}}{2}(z_{k+1}-x) + (1-\frac{\sqrt{\alpha\mu_x}}{2}) (z_{k+1}-z_k) \big\rangle \nonumber
\\\leq& \frac{\sqrt{\alpha\mu_x}\mu_x}{8}\|x-\widetilde x_k\|^2-\frac{1}{2\alpha} \|z_{k+1}-\widetilde x_k\|^2 - \frac{\mu_x\sqrt{\alpha\mu_x}}{16} \|x_{k+1}-\widetilde x_k\|^2\nonumber
\\&-\frac{\mu_x}{4} \|x-x_{k+1}\|^2 -\frac{2\mu_x -\sqrt{\alpha\mu_x}\mu_x}{16} \|x_k-x_{k+1}\|^2.
\end{align}
Next, we characterize the smoothness property of $\Phi(x)$. Using the form of $\nabla \Phi(x)$ in \cref{hyperG}, and based on Assumptions~\ref{fg:smooth},~\ref{g:hessiansJaco} and Assumption~\ref{bounded_f_assump} that $\|\nabla_yf(\cdot,\cdot)\|\leq U$, we have, for any $x_1,x_2\in\mathbb{R}^p$, 
\begin{align*}
\|\nabla\Phi(&x_1) - \nabla\Phi(x_2)\| \nonumber
\\\leq& \|\nabla_x f(x_1,y^*(x_1)) -\nabla_x f(x_2,y^*(x_2))  \| \nonumber
\\&+ \|\nabla_x \nabla_y g(x_1,y^*(x_1)) \nabla_y^2 g(x_1,y^*(x_1))^{-1}\nabla_y f(x_1,y^*(x_1)) \nonumber
\\&\hspace{0.5cm}-\nabla_x \nabla_y g(x_2,y^*(x_2)) \nabla_y^2 g(x_2,y^*(x_2))^{-1}\nabla_y f(x_2,y^*(x_2))\| \nonumber
\\  \leq & L_x\|x_1-x_2\| + L_{xy}\|y^*(x_1)-y^*(x_2)\| + \frac{\widetilde L_{xy}}{\mu_y} (L_{xy}\|x_1-x_2\|
+L_y\|y^*(x_1)-y^*(x_2)\|)  \nonumber
\\&+ \Big(\frac{U\rho_{xy}}{\mu_y}+\frac{\widetilde L_{xy}U\rho_{yy}}{\mu_y^2}\Big)(\|x_1-x_2\|+\|y^*(x_1)-y^*(x_2)\|),
\end{align*}
which, combined with Lemma 2.2 in~\cite{ghadimi2018approximation} that $\|y^*(x_1)-y^*(x_2)\|\leq \frac{\widetilde L_{xy}}{\mu_y} \|x_1-x_2\|$, yields 
\begin{align}\label{smooth_phi}
\|&\nabla\Phi(x_1) - \nabla\Phi(x_2)\| \nonumber
\\&\leq \Big (\underbrace{L_x+ \frac{2L_{xy}\widetilde L_{xy}}{\mu_y} +\Big(\frac{U\rho_{xy}}{\mu_y}+\frac{U\widetilde L_{xy}\rho_{yy}}{\mu_y^2}\Big)\Big(1+\frac{\widetilde L_{xy}}{\mu_y}\Big)+\frac{\widetilde L^2_{xy}L_y}{\mu^2_y}}_{L_\Phi}\Big) \|x_1-x_2\|.
\end{align}
Then, based on the above $L_\Phi$-smoothness of $\Phi(\cdot)$, we have 
\begin{align}\label{eq:phismooth}
\Phi(z_{k+1}) \leq & \Phi(\widetilde x_k) + \langle \nabla \Phi(\widetilde x_k), z_{k+1}-\widetilde x_k\rangle  + \frac{L_\Phi}{2} \|z_{k+1}-\widetilde x_k\|^2 \nonumber
\\=& \big(1-\frac{\sqrt{\alpha\mu_x}}{2}\big)(\Phi(\widetilde x_k) + \langle \nabla \Phi(\widetilde x_k), z_{k+1}-\widetilde x_k\rangle) \nonumber
\\&+ \frac{\sqrt{\alpha\mu_x}}{2}(\Phi(\widetilde x_k) + \langle \nabla \Phi(\widetilde x_k), z_{k+1}-\widetilde x_k\rangle)+\frac{L_\Phi}{2} \|z_{k+1}-\widetilde x_k\|^2.
\end{align}
Adding \cref{eq:gkquick} and~\cref{eq:phismooth} yields
\begin{align*}
\Phi(z_{k+1})\leq &  \big(1-\frac{\sqrt{\alpha\mu_x}}{2}\big)(\Phi(\widetilde x_k) + \langle \nabla \Phi(\widetilde x_k), z_{k}-\widetilde x_k\rangle) + \frac{\sqrt{\alpha\mu_x}}{2}(\Phi(\widetilde x_k) + \langle \nabla \Phi(\widetilde x_k), x-\widetilde x_k\rangle)  \nonumber
\\&+\big \langle \nabla\Phi(\widetilde x_k)-G_k, \frac{\sqrt{\alpha\mu_x}}{2}(z_{k+1}-x) + (1-\frac{\sqrt{\alpha\mu_x}}{2}) (z_{k+1}-z_k) \big\rangle  \nonumber
\\&+ \frac{\sqrt{\alpha\mu_x}\mu_x}{8}\|x-\widetilde x_k\|^2-\frac{1}{2\alpha}(1-\alpha L_\Phi) \|z_{k+1}-\widetilde x_k\|^2 - \frac{\mu_x\sqrt{\alpha\mu_x}}{16} \|x_{k+1}-\widetilde x_k\|^2\nonumber
\\&-\frac{\mu_x}{4} \|x-x_{k+1}\|^2 -\frac{2\mu_x -\sqrt{\alpha\mu_x}\mu_x}{16} \|x_k-x_{k+1}\|^2,
\end{align*}
which, in conjunction with the strong-convexity of $\Phi(\cdot)$, $\sqrt{\alpha \mu_x}\leq 1$ and $\alpha\leq \frac{1}{2L_\Phi}$, yields 
\begin{align}\label{eq:orginalsca}
\Phi(z_{k+1})\leq& \big(1-\frac{\sqrt{\alpha\mu_x}}{2}\big) \big(\Phi(z_k) -\frac{\mu_x}{2}\|z_k-\widetilde x_k\|^2\big) + \frac{\sqrt{\alpha\mu_x}}{2} \big(\Phi(x) -\frac{\mu_x}{2}\|x-\widetilde x_k\|^2 \big)   \nonumber
\\&+\big \langle \nabla\Phi(\widetilde x_k)-G_k, \frac{\sqrt{\alpha\mu_x}}{2}(z_{k+1}-x) + (1-\frac{\sqrt{\alpha\mu_x}}{2}) (z_{k+1}-z_k) \big\rangle \nonumber
\\&+\frac{\sqrt{\alpha\mu_x}\mu_x}{8}\|x-\widetilde x_k\|^2-\frac{1}{4\alpha} \|z_{k+1}-\widetilde x_k\|^2 - \frac{\mu_x\sqrt{\alpha\mu_x}}{16} \|x_{k+1}-\widetilde x_k\|^2.
\end{align}
Note that  we have the equality that 
\begin{align}\label{eq:xzk1}
\frac{\sqrt{\alpha\mu_x}}{2}(z_{k+1}-x) + &(1-\frac{\sqrt{\alpha\mu_x}}{2}) (z_{k+1}-z_k) \nonumber
\\& =  (z_{k+1}-\widetilde x_k) + \frac{\sqrt{\alpha\mu_x}}{2} (\widetilde x_k - x) + (1-\frac{\sqrt{\alpha\mu_x}}{2})(\widetilde x_k-z_k).
\end{align}
Then, using \cref{eq:xzk1} and the Cauchy-Schwarz inequality, we have
\begin{align}\label{eq:gkdistance}
\big \langle \nabla\Phi&(\widetilde x_k)-G_k, \frac{\sqrt{\alpha\mu_x}}{2}(z_{k+1}-x) + (1-\frac{\sqrt{\alpha\mu_x}}{2}) (z_{k+1}-z_k) \big\rangle \nonumber
\\\leq & \Big(2\alpha+\frac{1}{2\mu_x} +\frac{\sqrt{\alpha\mu_x}}{4\mu_x}\Big)\|\nabla\Phi(\widetilde x_k)-G_k\|^2 + \frac{1}{8\alpha}\|z_{k+1}-\widetilde x_k\|^2
+   \frac{\sqrt{\alpha\mu_x}\mu_x}{8} \|\widetilde x_k - x\| \nonumber
\\&  + (1-\frac{\sqrt{\alpha\mu_x}}{2})\frac{\mu_x}{2} \|z_k-\widetilde x_k\|^2.
\end{align}
Substituting \cref{eq:gkdistance} into \cref{eq:orginalsca} and cancelling out negative terms, we have 
\begin{align}\label{cocoinasca}
\Phi(z_{k+1})\leq & \big(1-\frac{\sqrt{\alpha\mu_x}}{2}\big) \Phi(z_k) + \frac{\sqrt{\alpha\mu_x}}{2} \Phi(x) -\frac{1}{8\alpha} \|z_{k+1}-\widetilde x_k\|^2 - \frac{\mu_x\sqrt{\alpha\mu_x}}{16} \|x_{k+1}-\widetilde x_k\|^2 \nonumber
\\&+ \Big(2\alpha+\frac{1}{2\mu_x} +\frac{\sqrt{\alpha\mu_x}}{4\mu_x}\Big)\|\nabla\Phi(\widetilde x_k)-G_k\|^2.
\end{align}
We next upper-bound the hypergradient estimation error $\|\nabla\Phi(\widetilde x_k)-G_k\|^2$. Recall that 
\begin{align}
G_k:= \nabla_x f(\widetilde x_k,y_k^N) -\nabla_x \nabla_y g(\widetilde x_k,y_k^N)v_k^M,
\end{align}
where $v_k^M$ is the $M^{th}$ step output of the heavy-ball method for solving $$\min_v Q(v):=\frac{1}{2}v^T\nabla_y^2 g(\widetilde x_k,y_k^N) v - v^T
\nabla_y f(\widetilde x_k,y^N_k)$$
Then, based on the convergence result of the heavy-ball method in~\cite{badithela2019analysis} with the stepsizes $\lambda=\frac{4}{(\sqrt{\widetilde L_y}+\sqrt{\mu_y})^2}$ and $\theta=\max\big\{\big(1-\sqrt{\lambda\mu_y}\big)^2,\big(1-\sqrt{\lambda\widetilde L_y}\big)^2\big\}$, we have 
\begin{align}\label{gg:opascas}
\|v_k^M - \nabla_y^2 g(\widetilde x_k,y_k^N)^{-1}\nabla_y f(\widetilde x_k,y^N_k) \| \leq  &\Big(\frac{\sqrt{\kappa_y}-1}{\sqrt{\kappa_y}+1}\Big)^M \Big\| \nabla_y^2 g(\widetilde x_k,y_k^N)^{-1}\nabla_y f(\widetilde x_k,y^N_k)\Big\| \nonumber
\\\overset{(i)}\leq& \frac{U}{\mu_y}\Big(\frac{\sqrt{\kappa_y}-1}{\sqrt{\kappa_y}+1}\Big)^M, %\nonumber
%\\\overset{(i)}\leq &\frac{L_y}{\mu_y}\Big(\frac{\sqrt{\kappa_y}-1}{\sqrt{\kappa_y}+1}\Big)^M \|y_k^*-y_k^N\|  + \frac{U}{\mu_y}\Big(\frac{\sqrt{\kappa_y}-1}{\sqrt{\kappa_y}+1}\Big)^M
\end{align}
where $(i)$ follows from Assumption~\ref{bounded_f_assump} that  $\|\nabla_y f(\cdot,\cdot)\|\leq U$. Let $y_k^*=\argmin_{y} g(\widetilde x_k,y)$. Then, based on the form of $\nabla\Phi(x)$ in \cref{hyperG}, we have 
\begin{align}\label{jingyisang}
\|G_k-&\nabla \Phi(\widetilde x_k)\| \nonumber
\\\leq & \| \nabla_x f( \widetilde x_k,y_k^N) -\nabla_x f(\widetilde x_k,y^*_k)\| + \widetilde L_{xy}\|v_k^M- \nabla_y^2 g(\widetilde x_k,y^*_k)^{-1}\nabla_y f(\widetilde x_k,y^*_k) \|  \nonumber
\\&+\frac{\|\nabla_y f(\widetilde  x_k,y^*_k) \|}{\mu_y}  \|\nabla_x \nabla_y g(\widetilde  x_k,y_k^N)-\nabla_x \nabla_y g(\widetilde x_k,y^*_k)\| \nonumber
\\\leq & L_y \|y^*_k-y_k^N\| + \widetilde L_{xy} \|v_k^M- \nabla_y^2 g(\widetilde x_k,y_k^N)^{-1}\nabla_y f(\widetilde x_k,y^N_k) \| \nonumber
\\&+ \widetilde L_{xy}\big\|\nabla_y^2 g(\widetilde x_k,y_k^N)^{-1}\nabla_y f(\widetilde x_k,y^N_k)-\nabla_y^2 g(\widetilde x_k,y^*_k)^{-1}\nabla_y f(\widetilde x_k,y^*_k) \big\|+\frac{U\rho_{xy}}{\mu_y} \|y_k^N-y^*_k\| \nonumber
%\\\leq & \Big(L_y +\frac{\widetilde L_{xy}L_y}{\mu_y} +\frac{\rho_{xy}}{\mu_y}\|\nabla_y f(\widetilde  x_k,y^*_k) \|\Big)\|y_k^N-y^*_k\|  \nonumber
%\\&+ \frac{\rho_{yy}\|y_k^N-y^*_k\|}{\mu_y^2}\|\nabla_y f(\widetilde x_k,y^*_k) \| + \widetilde L_{xy} \|v_k^M- \nabla_y^2 g(\widetilde x_k,y_k^N)^{-1}\nabla_y f(\widetilde x_k,y^N_k) \|  \nonumber
\\\overset{(i)}\leq&\Big(L_y +\frac{\widetilde L_{xy}L_y}{\mu_y} +\Big(\frac{\rho_{xy}}{\mu_y}+\frac{\widetilde L_{xy}\rho_{yy}}{\mu_y^2}\Big)U\Big)\|y_k^N-y^*_k\| + \frac{U\widetilde L_{xy}}{\mu_y}\Big(\frac{\sqrt{\kappa_y}-1}{\sqrt{\kappa_y}+1}\Big)^M,
\end{align}
where $(i)$ follows from \cref{gg:opascas}. Note that $y_k^N$ is obtained as the $N^{th}$ step output of AGD. Then, based on the analysis in \cite{nesterov2003introductory} for AGD, we have 
\begin{align}\label{eq:taun}
\|y_k^N-y_k^*\|^2\leq & \frac{\widetilde L_y +\mu_y}{\mu_y} \|y_k^0-y^*_k\|^2 \exp\Big(-\frac{N}{\sqrt{\kappa_y}}\Big) =  \frac{\widetilde L_y +\mu_y}{\mu_y} \|y_{k-1}^N-y^*_k\|^2 \exp\Big(-\frac{N}{\sqrt{\kappa_y}}\Big) \nonumber
\\\leq &\frac{2(\widetilde L_y +\mu_y)}{\mu_y} \exp\Big(-\frac{N}{\sqrt{\kappa_y}}\Big) (\|y_{k-1}^N-y_{k-1}^*\|^2 + \|y_{k-1}^*-y_k^*\|^2) \nonumber
\\\leq &\underbrace{\frac{2(\widetilde L_y +\mu_y)}{\mu_y} \exp\Big(-\frac{N}{\sqrt{\kappa_y}}\Big) }_{\tau_N}(\|y_{k-1}^N-y_{k-1}^*\|^2 + \kappa_y\|\widetilde x_k - \widetilde x_{k-1}\|^2),
\end{align}
which, in conjunction with $\widetilde x_k - \widetilde x_{k-1}= \eta_k (x_k-\widetilde x_{k-1}) + (1-\eta_k)(z_k - \widetilde x_{k-1})$, yields
\begin{align}\label{eq:telejjs}
\|y_k^N-y_k^*\|^2 \leq & \tau_N \|y_{k-1}^N-y_{k-1}^*\|^2+ \kappa_y\eta_k\tau_N\|x_k-\widetilde x_{k-1}\|^2 \nonumber
\\&+\kappa_y(1-\eta_k)\tau_N \|z_k-\widetilde x_{k-1}\|^2.
\end{align} 
Telescoping \cref{eq:telejjs} over $k$ yields
\begin{align*}
\|y_k^N-y_k^*\|^2\leq \tau_N^k\| y_0^N-y_0^*  \|^2 + \sum_{i=0}^{k-1}\tau_N^{k-i}\kappa_y\eta_k\|x_{i+1}-\widetilde x_{i}\|^2  +\sum_{i=0}^{k-1}\tau_N^{k-i}\kappa_y(1-\eta_k)\|z_{i+1}-\widetilde x_{i}\|^2, 
\end{align*}
which, in conjunction with \cref{cocoinasca} and \cref{jingyisang} and letting $x=x^*$, yields
\begin{align}\label{tele:bigeq}
\Phi(z_{k+1})-\Phi(x^*)\leq & \big(1-\frac{\sqrt{\alpha\mu_x}}{2}\big) (\Phi(z_k) -\Phi(x^*)-\frac{1}{8\alpha} \|z_{k+1}-\widetilde x_k\|^2 - \frac{\mu_x\sqrt{\alpha\mu_x}}{16} \|x_{k+1}-\widetilde x_k\|^2 \nonumber
\\&+ \lambda\sum_{i=0}^{k-1}\tau_N^{k-i}\kappa_y\eta_k\|x_{i+1}-\widetilde x_{i}\|^2  +\lambda\sum_{i=0}^{k-1}\tau_N^{k-i}\kappa_y(1-\eta_k)\|z_{i+1}-\widetilde x_{i}\|^2 \nonumber
\\&+\Delta+ \lambda \tau_N^k \|y_0^*-y_0^N\|^2,
\end{align}
where $\Delta$ and $\lambda$ are given by 
\begin{small}
\begin{align}\label{def:lambda}
\Delta =& \Big(4\alpha+\frac{1}{\mu_x} +\frac{\sqrt{\alpha\mu_x}}{2\mu_x}\Big)\frac{U^2\widetilde L^2_{xy}}{\mu^2_y}\Big(\frac{\sqrt{\kappa_y}-1}{\sqrt{\kappa_y}+1}\Big)^{2M} \nonumber
\\\lambda =& \Big(4\alpha+\frac{1}{\mu_x} +\frac{\sqrt{\alpha\mu_x}}{2\mu_x}\Big) \Big(L_y +\frac{\widetilde L_{xy}L_y}{\mu_y} +\Big(\frac{\rho_{xy}}{\mu_y}+\frac{\widetilde L_{xy}\rho_{yy}}{\mu_y^2}\Big)U\Big)^2.
\end{align}
\end{small}
\hspace{-0.12cm} Telescoping \cref{tele:bigeq} over $k$ from $0$ to $K-1$ and noting that $0<\eta_k\leq 1$, we have 
\begin{small}
\begin{align}
\Phi(z_{K})-\Phi(x^*) \leq & \big(1-\frac{\sqrt{\alpha\mu_x}}{2}\big)^K(\Phi(z_{0})-\Phi(x^*)) - \frac{1}{8\alpha}\sum_{k=0}^{K-1} \big(1-\frac{\sqrt{\alpha\mu_x}}{2}\big)^{K-1-k}\|z_{k+1}-\widetilde x_k\|^2 \nonumber
\\&- \frac{\mu_x\sqrt{\alpha\mu_x}}{16}\sum_{k=0}^{K-1} \big(1-\frac{\sqrt{\alpha\mu_x}}{2}\big)^{K-1-k}\|x_{k+1}-\widetilde x_k\|^2 + \frac{2\Delta}{\sqrt{\alpha\mu_x}} \nonumber
\\&+\sum_{k=0}^{K-1} \big(1-\frac{\sqrt{\alpha\mu_x}}{2}\big)^{K-1-k}\lambda \tau_N^k \|y_0^*-y_0^N\|^2 \nonumber
\\&+ \lambda\sum_{k=0}^{K-1} \big(1-\frac{\sqrt{\alpha\mu_x}}{2}\big)^{K-1-k}\sum_{i=0}^{k-1}\tau_N^{k-i}\kappa_y\|x_{i+1}-\widetilde x_{i}\|^2  \nonumber
\\&+\lambda\sum_{k=0}^{K-1} \big(1-\frac{\sqrt{\alpha\mu_x}}{2}\big)^{K-1-k}\sum_{i=0}^{k-1}\tau_N^{k-i}\kappa_y\|z_{i+1}-\widetilde x_{i}\|^2, \nonumber
\end{align}
\end{small}
\hspace{-0.12cm}which, in conjunction with the fact that $k\leq K-1$, yields
%Note that 
\begin{align}\label{eq:zkcgsaca}
\Phi(z_{K})-&\Phi(x^*) \leq  \big(1-\frac{\sqrt{\alpha\mu_x}}{2}\big)^K(\Phi(z_{0})-\Phi(x^*)) - \frac{1}{8\alpha}\sum_{k=0}^{K-1} \big(1-\frac{\sqrt{\alpha\mu_x}}{2}\big)^{K-1-k}\|z_{k+1}-\widetilde x_k\|^2 \nonumber
\\&- \frac{\mu_x\sqrt{\alpha\mu_x}}{16}\sum_{k=0}^{K-1} \big(1-\frac{\sqrt{\alpha\mu_x}}{2}\big)^{K-1-k}\|x_{k+1}-\widetilde x_k\|^2 + \frac{2\Delta}{\sqrt{\alpha\mu_x}} \nonumber
\\&+\sum_{k=0}^{K-1} \big(1-\frac{\sqrt{\alpha\mu_x}}{2}\big)^{K-1-k}\lambda \tau_N^k \|y_0^*-y_0^N\|^2 \nonumber
\\&+ \frac{2\tau_N\lambda\kappa_y}{\sqrt{\alpha\mu_x}}\sum_{i=0}^{K-2}\tau_N^{K-2-i}\|x_{i+1}-\widetilde x_{i}\|^2  + \frac{2\tau_N\lambda\kappa_y}{\sqrt{\alpha\mu_x}}\sum_{i=0}^{K-2}\tau_N^{K-2-i}\|z_{i+1}-\widetilde x_{i}\|^2.
\end{align}
Recall the definition of $\tau_N$ in \cref{eq:taun}. Then, choose $N$ such that 
\begin{align}\label{eq:nrequire}
\tau_N=\frac{2(\widetilde L_y +\mu_y)}{\mu_y} \exp\Big(-\frac{N}{\sqrt{\kappa_y}}\Big) \leq \min\Big\{ \frac{\sqrt{\mu_x}}{16\lambda\kappa_y\sqrt{\alpha}}, \frac{\alpha\mu_x^2}{32\lambda\kappa_y},\big(1-\frac{\sqrt{\alpha\mu_x}}{2}\big)^2\Big\}, 
\end{align}
which, in conjunction with \cref{eq:zkcgsaca}, yields  
% \mbox{\normalfont poly}(\mu_x,\mu_y)
\begin{align*}
\Phi(z_{K})-&\Phi(x^*) \leq  \big(1-\frac{\sqrt{\alpha\mu_x}}{2}\big)^K\Big(\Phi(z_{0})-\Phi(x^*)+\frac{2\lambda  \|y_0^*-y_0^N\|^2}{\sqrt{\alpha\mu_x}}\Big)+ \frac{2\Delta}{\sqrt{\alpha\mu_x}}. 
\end{align*}
Then, based on the definitions of $\lambda$ and $\Delta$ in~\cref{def:lambda} and $L_\Phi$ in~\cref{smooth_phi}, to achieve $\Phi(z^K)-\Phi(x^*)\leq \epsilon$, we have 
\begin{align}\label{eq:com_km}
K&\leq \mathcal{O}\Big( \sqrt{\frac{1}{\mu_x\mu_y^3}}\log\frac{\mbox{\small poly}(\mu_x,\mu_y,U,\Phi(x_{0})-\Phi(x^*))}{\epsilon} \Big) \nonumber
\\M &\leq \mathcal{O} \Big(\sqrt{\frac{\widetilde L_y}{\mu_y}}\log \frac{\mbox{\small poly}(\mu_x,\mu_y,U)}{\epsilon}\Big).
\end{align}
In addition, it follows from \cref{eq:nrequire} that 
\begin{align}\label{eq:com_n}
N\leq \mathcal{O}\Big(\sqrt{\frac{\widetilde L_y}{\mu_y}}\log (\mbox{\small poly}(\mu_x,\mu_y,U))\Big).
\end{align} 
Based on \cref{eq:com_km} and \cref{eq:com_n}, the total complexity is given by 
\begin{align*}
\mathcal{C}_{\text{\normalfont fun}}(\mathcal{A},\epsilon) &\leq \mathcal{O}(n_J+n_H + n_G)\leq \mathcal{O}(K+KM+KN) \nonumber
\\&\leq \mathcal{O}\Big(\sqrt{\frac{\widetilde L_y}{\mu_x\mu_y^4}}\log\frac{\mbox{\small poly}(\mu_x,\mu_y,U,\Phi(x_{0})-\Phi(x^*))}{\epsilon}  \log \frac{\mbox{\small poly}(\mu_x,\mu_y,U)}{\epsilon} \Big), 
\end{align*}
which finishes the proof. 

\section{Proof of \Cref{convex_upper_BG}}
Let $\widetilde x^*$ be the minimizer of $\widetilde \Phi(\cdot)$. Then, applying the results in \Cref{upper_srsr} to $\widetilde\Phi(x)$ with the strongly-convex parameter $\mu_x = \frac{\epsilon}{B^2}$ and choosing $N=\Theta \Big(\sqrt{\frac{\widetilde L_y}{\mu_y}}\log (\mbox{\small poly}(B,\epsilon,\mu_y,U))\Big)$, we have 
\begin{align*}
\widetilde\Phi(z_{K})-&\widetilde\Phi(\widetilde x^*) \leq  \big(1-\frac{\sqrt{\epsilon}}{2\sqrt{2L_{\widetilde\Phi}}B}\big)^K\Big(\widetilde \Phi(z_{0})-\widetilde \Phi(\widetilde x^*)+\frac{2\sqrt{2L_{\widetilde\Phi}}B\widetilde\lambda  \|y_0^*-y_0^N\|^2}{\sqrt{\alpha\epsilon}}\Big)+ \frac{2\widetilde\Delta\sqrt{2L_{\widetilde\Phi}}B}{\sqrt{\epsilon}},
\end{align*}
where $\widetilde \Delta$ and $\widetilde \lambda$ take the same forms as $\Delta$ and $\lambda$ in \cref{def:lambda} with $\mu_x$ being replaced by $\frac{\epsilon}{B^2}$. By choosing $M = \Theta \Big(\sqrt{\frac{\widetilde L_y}{\mu_y}}\log \frac{\mbox{\small poly}(B,\epsilon,\mu_y,U)}{\epsilon}\Big)$ in $\widetilde \Delta$,  we have $\frac{2\widetilde\Delta\sqrt{2L_{\widetilde\Phi}}B}{\sqrt{\epsilon}}\leq \frac{\epsilon}{4}$, and hence
\begin{align*}
\widetilde\Phi(z_{K})-&\widetilde\Phi(\widetilde x^*) \leq  \big(1-\frac{\sqrt{\epsilon}}{2\sqrt{2L_{\widetilde\Phi}}B}\big)^K\Big(\widetilde \Phi(z_{0})-\widetilde \Phi(\widetilde x^*)+\frac{2\sqrt{2L_{\widetilde\Phi}}B\widetilde\lambda  \|y_0^*-y_0^N\|^2}{\sqrt{\alpha\epsilon}}\Big)+ \frac{\epsilon}{4},
\end{align*}
which, in conjunction with $\widetilde\Phi(z_K)\geq\Phi(z_K)$, $\widetilde\Phi(\widetilde x^*)\leq\widetilde\Phi(x^*)=\Phi(x^*)+\frac{\epsilon}{2B^2}\|x^*\|^2$ and $z_0=0$, yields 
\begin{align}\label{maiyigelunhuiba}
\Phi(z_{K})-\Phi( x^*) \leq & \big(1-\frac{\sqrt{\epsilon}}{2\sqrt{2L_{\widetilde\Phi}}B}\big)^K\Big(\Phi(0)-\widetilde \Phi(\widetilde x^*)+\frac{2\sqrt{2L_{\widetilde\Phi}}B\widetilde\lambda  \|y_0^*-y_0^N\|^2}{\sqrt{\alpha\epsilon}}\Big)\nonumber
\\&+ \frac{\epsilon}{4} +\frac{\epsilon}{2B^2}\|x^*\|^2.
\end{align}
Based on \cref{ggsmidacposcs1}, we have $\Phi(0)-\widetilde \Phi(\widetilde x^*) \leq \Phi(0)-\Phi(x^*)$, which, combined with $\|x^*\|= B$ and $K=\Theta \Big(B \sqrt{\frac{1}{\epsilon\mu_y^3}}\log\frac{\mbox{\small poly}(\epsilon,\mu_y,B,U,\Phi(x_{0})-\Phi(x^*))}{\epsilon} \Big)$, yields  $\Phi(z_{K})-\Phi( x^*) \leq \epsilon$. Then, the total complexity satisfies 
\begin{align}
\mathcal{C}_{\text{\normalfont fun}}(\mathcal{A},\epsilon) &\leq \mathcal{O}(n_J+n_H + n_G)\leq \mathcal{O}(K+KM+KN) \nonumber
\\&\leq \mathcal{O}\Big(B\sqrt{\frac{\widetilde L_y}{\epsilon\mu_y^4}}\log\frac{\mbox{\small poly}(\epsilon,\mu_y,B,U,\Phi(x_{0})-\Phi(x^*))}{\epsilon}  \log \frac{\mbox{\small poly}(B,\epsilon,\mu_y,U)}{\epsilon} \Big), 
\end{align}
which finishes the proof. 

\vskip 0.2in

\bibliography{bilevel}

\end{document}